\documentclass{article}

\usepackage[T1]{fontenc}
\usepackage[utf8]{inputenc}
\usepackage{microtype}
\usepackage{graphicx}
\usepackage{subcaption}
\usepackage{booktabs}
\usepackage{url}
\usepackage{comment}
\usepackage{cuted}
\usepackage{caption} 

\usepackage{hyperref}

\usepackage[preprint]{icml2026}

\usepackage{amsfonts}
\usepackage{amssymb}
\usepackage{amsmath}
\usepackage{amsthm}
\usepackage{mathtools}
\usepackage{bbm}
\usepackage{mathrsfs}
\usepackage{nicefrac}
\usepackage{enumerate}
\usepackage{enumitem}
\usepackage{xcolor}
\usepackage{algorithm}
\usepackage{algorithmic}

\usepackage[capitalize,noabbrev]{cleveref}

\usepackage[disable,textsize=tiny]{todonotes}

\newcommand{\cA}{\mathcal{A}}

\newcommand{\cD}{\mathcal{D}}

\newcommand{\cU}{\mathcal{U}}

\newcommand{\cX}{\mathcal{X}}

\newcommand{\PP}{\mathbb{P}}
\newcommand{\EE}{\mathbb{E}}
\newcommand{\R}{\mathbb{R}}

\DeclareMathOperator{\Var}{Var}

\DeclareMathOperator{\diag}{diag}
\DeclareMathOperator{\supp}{supp}
\DeclareMathOperator{\Law}{Law}
\DeclareMathOperator{\KL}{KL}
\newcommand{\Regret}{\operatorname{Regret}}

\newcommand{\softmax}{\operatorname{softmax}}

\theoremstyle{plain}
\newtheorem{theorem}{Theorem}[section]
\newtheorem{proposition}[theorem]{Proposition}
\newtheorem{lemma}[theorem]{Lemma}
\newtheorem{corollary}[theorem]{Corollary}
\theoremstyle{definition}
\newtheorem{definition}[theorem]{Definition}

\theoremstyle{remark}
\newtheorem{remark}[theorem]{Remark}

\icmltitlerunning{Personalized Alignment Revisited}

\begin{document}

\twocolumn[
  \icmltitle{Personalized Alignment Revisited: \\
  The Necessity and Sufficiency of User Diversity}

  \begin{icmlauthorlist}
    \icmlauthor{Enoch Hyunwook Kang}{uw}
  \end{icmlauthorlist}

  \icmlaffiliation{uw}{Foster School of Business, University of Washington, Seattle, Washington, USA}

  \icmlcorrespondingauthor{Enoch Hyunwook Kang}{ehwkang@uw.edu}

  \icmlkeywords{Machine Learning, ICML, Personalized Alignment, User Diversity, Preference Learning}

  \vskip 0.3in
]

\printAffiliationsAndNotice{}

\begin{abstract}
  Personalized alignment aims to adapt large language models to heterogeneous user preferences, yet the precise theoretical conditions for its statistical efficiency have not been formally established. This paper characterizes the conditions under which personalized alignment achieves $O(1)$ online regret and $\log(1/\varepsilon)$ offline sample complexity. We show that these optimal rates depend on a specific \textit{user-diversity} condition: the population of user-specific heads must span the latent reward directions that can alter the optimal response. We prove that this condition is both \textit{necessary and sufficient}. When it holds, simple greedy algorithms achieve benchmark efficiency; when it fails, every learner in a natural admissible class incurs at least logarithmic regret. Our results identify user diversity as the fundamental driver of personalized identifiability.
\end{abstract}

\section{Introduction}

Large language models are increasingly deployed in settings where different
users may reasonably prefer different responses to the same prompt. Standard
alignment pipelines, however, typically treat such user disagreement as noise
rather than as signal. This motivates \emph{personalized alignment}: adapting
model behavior to heterogeneous user preferences while sharing statistical
strength across users. A recent line of work makes this idea concrete through
low-dimensional personalized reward models, in which user rewards are
decomposed into a shared representation and user-specific linear heads
\citep{shenfeld2025language,bose2025lore,poddar2024personalizing,chen2024pal}.

The empirical case for this agenda, however, is mixed. While personalized
methods have been reported to outperform non-personalized baselines
\citep{shenfeld2025language,bose2025lore}, other studies find that non-personalized baselines often match or exceed personalized methods in both reward-model prediction and downstream alignment quality
\citep{rezk2025reward}. These conflicting findings suggest that we need a theoretical account of exactly when user-specific preference data translates
into better personalized decisions, and when it does not.

In this paper, we address this theoretical gap. For this, we first ask what it should theoretically mean, formally, for a
personalized alignment procedure to succeed. For non-personalized online alignment, \citet{kang2026demystifying}
show that the effective regret rate for online alignment is actually $O(1)$, i.e., bounded regret. For offline alignment from reference-logged preference
data, the corresponding sample complexity is $\log(1/\varepsilon)$ instead of $1/\varepsilon$ in \citet{wu2025greedy}. These rates match the strong empirical performance of 
alignment pipelines
\citep{ouyang2022training,xiong2024iterative,dong2024rlhf}.
This
sets a natural benchmark: a working personalized alignment method may also achieve personalized $O(1)$ online regret and $\log(1/\varepsilon)$ sample complexity. 

This paper characterizes exactly when personalized alignment achieves
$O(1)$ regret and $\log(1/\varepsilon)$ sample complexity. The critical condition is not about algorithmic design, but about a (decision-relevant) user-diversity: the population of user heads must span the latent
reward directions whose perturbations can change the
optimal response. When this condition holds, simple greedy algorithms achieve $O(1)$ regret and $\log(1/\epsilon)$ sample complexity; when it fails, every learner in a natural
admissible class suffers logarithmic regret. Conceptually, our condition is closely related to the task-diversity and
representation-coverage conditions that pervade the theory of transfer and
multi-task representation learning, where source tasks must be diverse enough
to identify a shared representation useful for downstream tasks
\citep{du2021fewshot,tripuraneni2020theory,tripuraneni2021provable,xu2021representation}.

The remainder of the paper is organized as follows. Section~\ref{sec:setup} introduces the personalized preference model, shared-representation reward structure, and notation. Section~\ref{sec:decision_framework} sets up the decision framework used by both the online and offline analyses, including the parameter class, temperature-zero evaluation criterion, regularity conditions, and reference-slate MNL loss. Section~\ref{sec:online_alignment} studies online alignment by presenting the greedy learner and proving the bounded-regret characterization. Section~\ref{sec:offline_exact_erm_pi0} develops the offline analogue and establishes logarithmic accuracy complexity. Section~\ref{sec:experiments} reports simulation evidence. Related work is discussed in Appendix~\ref{sec:related_works}.

\section{Setup}
\label{sec:setup}

\subsection{Interaction model (contexts, users, policies)}
We study an online (iterative) personalized alignment setting. Let $\cX$ denote a context space
(queries) and let $\cA$ denote an action space (the space of responses). $|\cA|$ is often
considered infinite. Contexts are drawn i.i.d.\ from an unknown distribution $d_0$ over $\cX$.

\paragraph{Users.}
We consider a population of users $\cU=\{1,\dots,U\}$ and a (possibly unknown) distribution $\rho$ over $\cU$. Each interaction round $t$ is associated with a user index $i_t \sim \rho$ who provides preference feedback. All main theorems below assume a finite user set. 

\paragraph{Personalized policies and reference model.}
A personalized policy for user $i$ at time $t$ is a mapping
\[
\pi_{it}:\ \cX \to \Delta(\cA),\qquad a \sim \pi_{it}(\cdot \mid x), \qquad i\in \cU.
\]
We are also given a fixed reference policy $\pi_0:\cX\to\Delta(\cA)$ (e.g., the SFT checkpoint). Following the standard alignment literature, we restrict attention to policies that are absolutely continuous with respect to the reference, so that the KL regularizer is well-defined:
\begin{align}
    \Pi \;:=\; \bigl\{& \pi:\cX\times\cU\to\Delta(\cA): \notag
    \\
    &\ \pi(\cdot\mid x,i)\ll \pi_0(\cdot\mid x)\ \ \forall (x,i)\in\cX\times\cU\bigr\}. \notag
\end{align}

\subsection{Preference model and feedback}
Fix a slate size $K\ge 2$. Given $(x,i)$ and a slate
\begin{align}
    \mathbf{a}&=(a_1,\dots,a_K)\in \cA^K, \notag
\\ 
a_k &\overset{\text{i.i.d.}}{\sim} \pi_{it}(\cdot\mid x)\quad \text{for }k=1,\dots,K, \notag
\end{align}

each user $i$ is associated with an (unknown) multinomial choice model
\[
P_i^\star:\ \cX\times\cA^K \to \Delta(\{1,\dots,K\}),
\]
where $P_i^\star(x,\mathbf{a})_k = \mathbb{P}(y=k\mid x,i,\mathbf{a})$ and $\sum_{k=1}^K P_i^\star(x,\mathbf{a})_k = 1$.
Feedback is generated as
\[
y \sim \mathrm{Cat}\!\big(P_i^\star(x,\mathbf{a})\big).
\]
Here \(\mathrm{Cat}(p)\) denotes the categorical distribution on \(\{1,\dots,K\}\) with
probability vector \(p\in\Delta(\{1,\dots,K\})\).

Suppose that we are given a known bounded base embedding $\psi:\cX\times\cA\to\R^d$ with $\|\psi(x,a)\|_2\le 1$, e.g., BERT embeddings \citep{warner2025smarter,teiletche2025modernvbert}. 
Let $\mathcal H$ be a class of $J$-dimensional representations (e.g., neural networks) of the embedding, where $h\in \mathcal H$ is 
\[
h:\R^d\to\R^J,
\qquad
\phi_h(x,a)\coloneqq h(\psi(x,a))\in\R^J.
\]

\paragraph{Representation structure.}
Following the standard representation model used in personalization literature \citep{bose2025lore, shenfeld2025language}, user rewards are linked through a shared representation and user-specific linear heads. Each user $i$ has an unknown coefficient vector (a ``task head'') $\lambda_i^\star\in\R^J$, and we write
$\Lambda^\star=[\lambda_1^\star,\dots,\lambda_U^\star]\in\R^{J\times U}$. There exists an unknown $h^\star\in\mathcal H$ such that the true shared representation satisfies
\[
\phi^\star(x,a)=\phi_{h^\star}(x,a)
\]
for all $(x,a)\in\cX\times\cA$, and the true individual reward $R_i^\star(x,a)$ is of the form
\begin{equation}
R_i^\star(x,a) \;=\; \langle \lambda_i^\star,\phi^\star(x,a)\rangle.
\label{eq:lowrank_reward}
\end{equation}

We assume boundedness $R_i^\star(x,a)\in[-B_\star,B_\star]$ for all $(x,a,i)$ for some fixed
$B_\star>0$. A sufficient condition is $\|\lambda_i^\star\|_2\le B_\star$ and
$\|\phi^\star(x,a)\|_2\le 1$ for all users and all $(x,a)$.

Under \eqref{eq:lowrank_reward}, the choices follow a multinomial logit model: for a slate
$\mathbf{a}=(a_1,\dots,a_K)$,
\begin{align}
\mathbb{P}(y=k\mid x,i,\mathbf{a})
\;&=\;
P_i^\star(x,\mathbf{a})_k
\;=\;
\frac{\exp\!\big(R_i^\star(x,a_k)\big)}{\sum_{\ell=1}^K \exp\!\big(R_i^\star(x,a_\ell)\big)}
\; \notag
\\
&=\;
\frac{\exp\!\big(\langle \lambda_i^\star,\phi^\star(x,a_k)\rangle\big)}
{\sum_{\ell=1}^K \exp\!\big(\langle \lambda_i^\star,\phi^\star(x,a_\ell)\rangle\big)}.
\label{eq:lowrank_mnl}
\end{align}
When $K=2$, this reduces to the Bradley--Terry model.

The compactness and continuity requirements on the representation family, head domain, and parameter class are collected in Condition~\ref{as:LP1} in Section~\ref{sec:class_assumptions}. The definitions below specify the parameterized score class to which those conditions apply.

\paragraph{Representation and user-specific linear heads.}

We parameterize the representation family as
\[
\mathcal H=\{h_\theta:\theta\in\Theta\}.
\]
Here \(\Theta\) denotes the parameter domain indexing the shared representation family.
To jointly learn the shared representation and user heads, we work over a joint personalized model class
\[
\mathcal M \subseteq \Theta\times \Xi,
\]
where \(\Xi\subset \R^{J\times U}\) is a head-matrix domain; for
\(\Lambda\in\Xi\), we write
\[
\Lambda=[\lambda_1,\dots,\lambda_U]
\]
with columns \(\lambda_i\in\R^J\).

\paragraph{Candidate parameterized scores.}
Given any candidate shared representation $h\in\mathcal H$ and any candidate
user-head matrix $\Lambda=[\lambda_1,\dots,\lambda_U]\in\Xi$, define the raw linear score
\begin{equation}
R^{\mathrm{raw}}_{h,\Lambda}(x,a,i)
\;\coloneqq\;
\langle \lambda_i,\phi_h(x,a)\rangle.
\label{eq:param_score_class_def}
\end{equation}
For likelihood, policy, and temperature-zero calculations below, we use the corresponding
\(\pi_0\)-centered score. Define
\[
\bar\phi_h(x,a)
\;\coloneqq\;
\phi_h(x,a)-\EE_{A\sim\pi_0(\cdot\mid x)}[\phi_h(x,A)]
\]
and set
\[
R_{h,\Lambda}(x,a,i)
\coloneqq
\langle \lambda_i,\bar\phi_h(x,a)\rangle.
\]
By the \(\pi_0\)-centering invariance used in the non-personalized temperature-zero analysis of \citet{kang2026demystifying} and restated in Lemma~\ref{lem:pi0_centering_wlog_psi} of Appendix \ref{app:supporting_lemmas}, replacing the raw score by its centered version does not change the induced MNL likelihoods, KL-tilted policies, or supportwise argmax sets.
For \((\theta,\Lambda)\in \mathcal M\), we write
\[
R_{\theta,\Lambda}(x,a,i)
\coloneqq
\langle \lambda_i,\bar\phi_{h_\theta}(x,a)\rangle.
\]

The compactness, continuity, and envelope consequences of this parameterization are stated with the model-class regularity conditions in Section~\ref{sec:class_assumptions}.

\section{Decision framework}
\label{sec:decision_framework}

This section fixes the decision-theoretic objects used by both the online and offline analyses: the regret criteria, support-wise score geometry, the parameter class, the structural regularity conditions, and the reference-slate MNL loss used in the proofs.

\subsection{Temperature-zero regret}
\label{sec:temperature_zero_regret}

In this paper, for both online and offline alignment, we use the traditional \textit{temperature-zero regret} \cite{kang2026demystifying} which evaluates the deterministic top-ranked action
induced by a score\footnote{This choice is motivated by the fact that the temperature-zero regret criterion yields bounded
cumulative regret, whereas the KL-regularized criterion can continue to incur
logarithmic regret \citep{wu2025greedy,kang2026demystifying}. The residual
KL-regularized regret, therefore, only reflects finite-temperature randomization rather than continued exploration.}. For any centered score \(R\), require the supportwise
argmax to be nonempty on a full \(d_0\times\rho\)-measure set and fix a
measurable selector
\[
a_R(x,i)\in \arg\max_{a\in \cA_0(x)} R(x,a,i)
\]
there. Selectors may be extended arbitrarily off this full-measure set. For
truth \(p\in\mathcal P\), write
\[
a_p(x,i)\coloneqq a_{R_p}(x,i).
\]

Define the expected one-step temperature-zero regret under
truth \(p\) by
\begin{align}
    \mathcal G_p(R) 
&\coloneqq \EE_{(X,I)\sim d_0\times\rho}\notag
\\
&
\Big[
R_p\big(X,a_p(X,I),I\big)-R_p\big(X,a_R(X,I),I\big)
\Big].
\label{eq:truth_centered_expected_regret_def}
\end{align}
This one-step functional is the common temperature-zero regret primitive used
in both the online and offline analyses. For any random score output
\(\widehat R\), its expected temperature-zero regret under truth \(p\) is
\(\EE_p[\mathcal G_p(\widehat R)]\). Online performance accumulates this
quantity over deployed estimates, while offline performance evaluates it for
the final score output.

In the online setting, a learning rule \(\mathsf A\) produces score estimates
\[
\widehat R_0^{\mathsf A},\widehat R_1^{\mathsf A},\widehat R_2^{\mathsf A},\dots,
\]
where \(\widehat R_t^{\mathsf A}\) is fitted from the first \(t\) rounds of data and is deployed on round \(t+1\). Its expected cumulative temperature-zero regret under truth \(p\in\mathcal P\) is
\begin{equation}
\Regret_{0,p}^{\mathsf A}(T)
\coloneqq
\sum_{t=0}^{T-1}\EE_p\big[\mathcal G_p(\widehat R_t^{\mathsf A})\big].
\label{eq:expected_regret_class}
\end{equation}

\subsection{Parameter class and supportwise score geometry}
\label{sec:parameter_class_eval}

We first formulate the theory around a closed class of nearby possible truths. The goal is to identify when personalized preference data separates nearby alternatives that still matter for the final recommendation.

Throughout the theoretical sections below, assume that \(\cA\) is a separable metric space and interpret
\(\supp(\pi_0(\cdot\mid x))\) as the topological support of \(\pi_0(\cdot\mid x)\). Write
\[
\cA_0(x)\coloneqq \supp(\pi_0(\cdot\mid x)).
\]
For a score \(R:\cX\times\cA\times\cU\to\R\), use the supportwise norm
\[
\|R\|_{\infty,\supp(\pi_0)}
\coloneqq
\sup_{i\in\cU}\sup_{x\in\cX}\sup_{a\in \cA_0(x)}|R(x,a,i)|,
\]
and use the analogous vector-valued supremum norm for representations. Throughout the theoretical sections below, we work with \(\pi_0\)-centered scores. This is without loss of generality: subtracting, for each \((x,i)\), an \(a\)-independent constant from a score leaves the MNL choice probabilities, the KL-tilted policy, and the supportwise argmax recommendation unchanged; see Lemma~\ref{lem:pi0_centering_wlog_psi}. The supportwise continuity arguments below adapt the topological-support upgrade used in \citet{kang2026demystifying} to user-indexed scores.

Fix a parameter class
\[
\mathcal P\subseteq \mathcal M.
\]
Each
\[
p=(\theta_p,\Lambda_p)\in\mathcal P
\]
induces the centered score
\[
R_p(x,a,i)\coloneqq \langle \lambda_{i,p},\bar\phi_{h_{\theta_p}}(x,a)\rangle.
\]
Let
\[
\mathcal F_{\mathcal P}\coloneqq \{R_p:p\in\mathcal P\}
\]
denote the induced score class.

\subsection{Structural regularity conditions on the parameter class}
\label{sec:class_assumptions}

We impose the following conditions throughout the online and offline analyses in Sections~\ref{sec:online_alignment} and~\ref{sec:offline_exact_erm_pi0}.
\begin{enumerate}[label=(C\arabic*),ref=C\arabic*,leftmargin=2em]
\item \label{as:LP1} \textbf{Compact continuous personalized parameter and score class.}
There exists a compact set \(\Omega_\psi\subset \R^d\) such that
\[
\psi(x,a)\in\Omega_\psi
\]
for every \((x,a)\) with \(a\in \cA_0(x)\). The representation parameter set \(\Theta\) is compact, and the map
\[
(\theta,z)\longmapsto h_\theta(z)
\]
is continuous on \(\Theta\times\Omega_\psi\). The head-matrix set \(\Xi\subset\R^{J\times U}\) is compact, and
\[
\mathcal M\subseteq \Theta\times\Xi
\]
is closed. The parameter class \(\mathcal P\subseteq\mathcal M\) is nonempty and closed. Moreover, for every \(p\in\mathcal P\), the action section
\[
a\longmapsto R_p(x,a,i)
\]
is continuous on \(\cA_0(x)\) for \(d_0\times\rho\)-a.e.\ \((x,i)\).

\item \label{as:LP2} \textbf{Uniform supportwise reward gap on \(\mathcal P\).}
There exists \(\Delta_{\min}^{\mathcal P}>0\) such that for every truth \(p\in\mathcal P\),
\(a_p(X,I)\) is unique and
\begin{align}
    R_p\!\big(X,a_p(X,I),I\big)
&-
\sup_{a\in \cA_0(X),\ a\neq a_p(X,I)}
R_p(X,a,I) \notag
\\
&\ge
\Delta_{\min}^{\mathcal P}
\qquad
 d_0\times\rho\text{-a.s.} \notag
\end{align}

\end{enumerate}
Condition~\ref{as:LP1} is a compactness-based supportwise regularity condition on the personalized score class. It plays the same role as the finite-class or bounded-covering-number assumptions commonly imposed in recent online alignment analyses \citep{xiong2024iterative,ye2024online,wu2025greedy,kang2026demystifying}\todo{add references}. The continuity requirement in Condition~\ref{as:LP1} is mild for standard parametric personalized reward classes, including bounded-parameter shared-representation models with user-specific linear heads and neural representation families with continuous activations; see Appendix~\ref{app:linear_neural_examples}. Condition~\ref{as:LP2} is a standard margin or selector-stability condition\footnote{Mathematically, it says that each user's top-ranked action is stable under small score perturbations. In bandit problems, positive gaps typically underlie logarithmic instance-dependent regret, while bounded or sub-logarithmic regret generally requires additional self-exploration structure, such as optimal-arm spanning, HLS, covariate diversity, smoothed contexts, or anti-concentration \citep{hao2020adaptive,papini2021leveraging,tirinzoni2023complexity,bastani2021mostly,kannan2018smoothed,raghavan2023greedy,kim2024local}.} in preference-based reward modeling. Practically, low-margin comparisons are often ambiguous: \citet{wang2024secrets} report low annotator agreement in real preference data, and preference-strength or margin-aware approaches have been proposed to handle heterogeneous comparison quality \citep{qin2024towards,kim2024margin}. Thus, collapsing imperceptible quality differences is a natural modeling convention in alignment settings where weak or ambiguous preference pairs may be unreliable or uninformative.

By Condition~\ref{as:LP1}, the score class \(\mathcal F_{\mathcal P}\) is compact under \(\|\cdot\|_{\infty,\supp(\pi_0)}\). In particular, the envelope
\[
B_{\mathcal P}
\coloneqq
\sup_{p\in\mathcal P}\|R_p\|_{\infty,\supp(\pi_0)}
\]
is finite; see Lemma~\ref{lem:compact_continuity_consequences_P}. For the regret bounds below, set
\[
\Delta_{\max}^{\mathcal P}\coloneqq 2B_{\mathcal P}.
\]

\subsection{Reference-slate MNL loss}
\label{sec:truth_centered_population_risk}

We record the common loss functional once here, leaving the online and offline sections to specify how samples are obtained. For any centered score \(R\) and slate \(\mathbf a=(a_1,\dots,a_K)\), define
\[
P_R(y=k\mid x,i,\mathbf a)
\coloneqq
\frac{\exp(R(x,a_k,i))}{\sum_{\ell=1}^K \exp(R(x,a_\ell,i))},
\;k=1,\dots,K,
\]
and
\[
\ell(\mathbf v,y)
\coloneqq
\log\!\Big(\sum_{k=1}^K e^{v_k}\Big)-v_y.
\]
For truth \(p\in\mathcal P\), define the reference-slate population loss
\begin{equation}
\mathcal L_p(R)
\coloneqq
\EE\!\left[
\ell\!\left(
\big(R(X,A_1,I),\dots,R(X,A_K,I)\big),Y
\right)
\right], \notag
\end{equation}
where \((X,I)\sim d_0\times\rho\), \(\mathbf A=(A_1,\dots,A_K)\sim\pi_0(\cdot\mid X)^{\otimes K}\), and \(Y\sim P_{R_p}(\cdot\mid X,I,\mathbf A)\). This is a definition of the reference-slate population objective; the offline model below draws i.i.d. observations from exactly this law.

\section{Online alignment}
\label{sec:online_alignment}

This section specializes the decision framework from Section~\ref{sec:decision_framework} to the online personalized alignment problem. We describe the greedy learner, define the decision-relevant user-diversity modulus, and prove the bounded-regret characterization.

\subsection{Greedy learning}
\label{sec:greedy_learning}

We describe the online learner directly in the joint model class
\(\mathcal M\) from Section~\ref{sec:setup}. After round \(t\), the learner uses the
induced centered score
\[
\widehat R_t(x,a,i)\coloneqq R_{\hat\theta_t,\hat\Lambda_t}(x,a,i),
\qquad
(\hat\theta_t,\hat\Lambda_t)\in \mathcal M.
\]

Fix a tilt parameter \(\eta>0\).
Given $\widehat R_t$, the KL-regularized greedy policy for user $i$ is
\begin{align}
    \pi_{i,t+1}(a\mid x)
=&
\frac{\pi_0(a\mid x)\exp\!\big(\eta\,\widehat R_t(x,a,i)\big)}
{Z_{i,t}(x)}, \notag
\\
Z_{i,t}(x)\coloneqq&
\int_{\cA}\pi_0(a'\mid x)\exp\!\big(\eta\,\widehat R_t(x,a',i)\big)\,da'.
\label{eq:greedy_policy_update_personalized_mnl}
\end{align}

Initialize \(\widehat R_0\equiv 0\), $\pi_{i,1}\gets \pi_0$ for all $i\in\cU$, and \(\cD_0=\emptyset\). For each round
\(t=1,\dots,T\):
\begin{enumerate}[label=\arabic*. ,leftmargin=2em]
\item Observe the user-query pair $(x_t,i_t)$ with $x_t\sim d_0$ and $i_t\sim\rho$.
\item Sample a slate $\mathbf{a}_t=(a_{t,1},\dots,a_{t,K})$ with
      $a_{t,k}\overset{\mathrm{i.i.d.}}{\sim}\pi_{i_t,t}(\cdot\mid x_t)$ for $k=1,\dots,K$.
\item Observe the user choice $y_t\sim \mathrm{Cat}\!\big(P_{i_t}^\star(x_t,\mathbf{a}_t)\big)$.
\item Update $\cD_t\gets \cD_{t-1}\cup\{(x_t,i_t,a_{t,1},\dots,a_{t,K},y_t)\}$.
\item Fit the shared representation and user heads by empirical MNL risk minimization\footnote{When Condition~\ref{as:LP1} is imposed,
Lemma~\ref{lem:compact_joint_gf_head} implies that \(\mathcal M\) is compact and
that \((\theta,\Lambda)\mapsto R_{\theta,\Lambda}\) is continuous under
\(\|\cdot\|_{\infty,\supp(\pi_0)}\). On the full-probability event that every sampled action lies in
\(\supp(\pi_0(\cdot\mid x_s))\), each summand in \eqref{eq:parametric_mnl_mle} is continuous in the
induced score, so the empirical objective is continuous on the compact parameter set and an exact
minimizer exists. On the complementary null event, one may use any fixed measurable default
parameter without affecting the probability statements below.}:
\begin{equation}
\big(\hat\theta_t,\hat\Lambda_t\big)
\in
\arg\min_{(\theta,\Lambda)\in \mathcal M}
\widehat{\mathcal L}_t(\theta,\Lambda),
\label{eq:parametric_mnl_mle}
\end{equation}
where
\begin{align}
    \widehat{\mathcal L}_t(\theta,\Lambda)
\coloneqq
\frac1t\sum_{s=1}^t
\bigl[
\log\!\Big(\sum_{k=1}^K \exp\big(R_{\theta,\Lambda}(x_s,a_{s,k},i_s)\big)\Big) \notag
\\-
R_{\theta,\Lambda}(x_s,a_{s,y_s},i_s)
\bigr]. \notag
\end{align}

\item Set \(\widehat R_t(x,a,i)\coloneqq R_{\hat\theta_t,\hat\Lambda_t}(x,a,i)\) and update
policies via \eqref{eq:greedy_policy_update_personalized_mnl}.
\end{enumerate}

\subsection{User diversity condition}
\label{sec:truth_centered_moduli}

Using the regret functional \(\mathcal G_p\) and population risk \(\mathcal L_p\) from Section~\ref{sec:decision_framework}, fix a regret threshold \(\varepsilon_0>0\). For truth \(p\in\mathcal P\) and radius \(r>0\), define the truth-centered regret shell
\begin{align}
    &\mathcal Q_{\mathrm{reg}}(p,r;\varepsilon_0) \notag
\\
&\coloneqq
\Big\{
q\in\mathcal P:
\max_{i\in\cU}\|\lambda_{i,q}-\lambda_{i,p}\|_2\le r,
\ \mathcal G_p(R_q)\ge \varepsilon_0
\Big\}. \notag
\end{align}
Thus \(\mathcal Q_{\mathrm{reg}}(p,r;\varepsilon_0)\) is the regret-relevant shell of head-close alternatives: it contains exactly those nearby alternatives that still incur at least \(\varepsilon_0\) expected one-step regret under truth \(p\).

For \(p\in\mathcal P\), define the user head second-moment matrix
\[
G_{\lambda,p}
\coloneqq
\EE_{I\sim\rho}\big[\lambda_{I,p}\lambda_{I,p}^\top\big],
\]
and, for \(q\in\mathcal P\), define the truth-centered representation difference
\[
g_{q\mid p}(x,a)
\coloneqq
\bar\phi_{h_{\theta_q}}(x,a)-\bar\phi_{h_{\theta_p}}(x,a).
\]
With \((X,I)\sim d_0\times\rho\) and \(A\sim\pi_0(\cdot\mid X)\), define
\begin{equation}
\mathfrak d_{\mathrm{reg},p}(r;\varepsilon_0)
\coloneqq
\inf_{q\in\mathcal Q_{\mathrm{reg}}(p,r;\varepsilon_0)}
\EE\!\left[
g_{q\mid p}(X,A)^\top G_{\lambda,p}\,g_{q\mid p}(X,A)
\right], \notag
\end{equation}
with \(\inf\emptyset=+\infty\), and its class-uniform version
\begin{equation}
\underline{\mathfrak d}_{\mathrm{reg}}(r;\varepsilon_0)
\coloneqq
\inf_{p\in\mathcal P}\mathfrak d_{\mathrm{reg},p}(r;\varepsilon_0) \notag
\end{equation}
The quantity \(\underline{\mathfrak d}_{\mathrm{reg}}(r;\varepsilon_0)\) is the decision-relevant user-diversity modulus. It only probes representation directions that are realizable by a head-close alternative and still large enough to change the temperature-zero recommendation at scale \(\varepsilon_0\). This motivates the following threshold definition.

\begin{lemma}[Automatic isolation of positive regret on \(\mathcal F_{\mathcal P}\)]
\label{lem:automatic_positive_regret_isolation_P}
Assume \ref{as:LP1}--\ref{as:LP2}. Then there exists a constant
\[
\varepsilon_{\mathrm{iso}}^{\mathcal P}>0
\]
such that for every \(p,q\in\mathcal P\),
\[
\mathcal G_p(R_q)\in\{0\}\cup[\varepsilon_{\mathrm{iso}}^{\mathcal P},\infty).
\]
\end{lemma}

\begin{definition}[Decision-relevant user diversity condition]
\label{def:decision_relevant_user_diversity}
We say that \(\mathcal P\) satisfies the decision-relevant user diversity if
\[
\liminf_{r\downarrow0}
\underline{\mathfrak d}_{\mathrm{reg}}(r;\varepsilon_{\mathrm{iso}}^{\mathcal P})>0.
\]
We say that decision-relevant user diversity fails if
\[
\liminf_{r\downarrow0}
\underline{\mathfrak d}_{\mathrm{reg}}(r;\varepsilon_{\mathrm{iso}}^{\mathcal P})=0.
\]
\end{definition}

The following sections will show that \(\mathcal P\) satisfying the decision-relevant user diversity condition in the sense of Definition~\ref{def:decision_relevant_user_diversity} is both necessary and sufficient for the bounded-regret/log-complexity behavior we study. 

Despite its importance, it is not immediately clear from the statement of Definition~\ref{def:decision_relevant_user_diversity} how to verify the condition. The next lemma gives a simple and intuitive sufficient condition for the decision-relevant user diversity condition to hold: full-rank variation in the user heads rules out collapse of the decision-relevant diversity modulus.

\begin{lemma}[Full-rank user heads imply decision-relevant user diversity]
\label{lem:pd_user_heads_suffice_for_diversity}
Assume \ref{as:LP1}--\ref{as:LP2}. If
\[
G_{\lambda,p}
\coloneqq
\EE_{I\sim\rho}\!\left[\lambda_{I,p}\lambda_{I,p}^{\top}\right]
\succ 0
\qquad
\text{for every }p\in\mathcal P,
\]
then \(\mathcal P\) has decision-relevant user diversity in the sense of Definition~\ref{def:decision_relevant_user_diversity}.
\end{lemma}

\begin{proof}
See Appendix~\ref{app:deferred_proofs}.
\end{proof}

This full-rank covariance condition in Lemma \ref{lem:pd_user_heads_suffice_for_diversity} says that the user-head population spans every latent representation direction, so no nonzero representation perturbation can be invisible to all users. It is the personalized-alignment analog of task-diversity assumptions in transfer learning, where source tasks must span the representation directions needed for downstream prediction \citep{du2021fewshot}.

\subsection{Main bounded-regret result}
\label{sec:umr_corollary}

The main positive result is a bounded-regret theorem for the exact greedy personalized ERM learner. The condition is decision-relevant user diversity, as formalized in Definition~\ref{def:decision_relevant_user_diversity}.

Let \(\varepsilon_{\mathrm{iso}}^{\mathcal P}>0\) denote the positive-regret isolation constant guaranteed by Lemma~\ref{lem:automatic_positive_regret_isolation_P}. Write \(\mathrm{ERM}\) for the greedy learner from Section~\ref{sec:greedy_learning}, specialized to exact empirical MNL minimization over the compact realizable score class \(\mathcal F_{\mathcal P}\): \(\widehat R_0\equiv0\), \(\pi_{i,1}=\pi_0\), and, for \(t\ge1\),
\[
\widehat R_t
\in
\arg\min_{R\in\mathcal F_{\mathcal P}}
\frac1t\sum_{s=1}^t
-\log P_R(y_s\mid x_s,i_s,\mathbf a_s),
\]
with policies updated by \eqref{eq:greedy_policy_update_personalized_mnl}. By Condition~\ref{as:LP1} and Lemma~\ref{lem:compact_continuity_consequences_P}, the argmin is nonempty on the full-probability support event; fix a measurable ERM selection rule.

\begin{theorem}[Bounded regret for greedy personalized alignment]
\label{thm:bounded_regret_under_umr}
Assume \ref{as:LP1}--\ref{as:LP2}. If \(\mathcal P\) has decision-relevant user diversity in the sense of Definition~\ref{def:decision_relevant_user_diversity}, then
\[
\sup_{p\in\mathcal P}\sup_{T\ge 1}
\Regret_{0,p}^{\mathrm{ERM}}(T)<\infty.
\]
More precisely, let \(\gamma_{\mathcal P,\varepsilon_{\mathrm{iso}}^{\mathcal P}}>0\) be the fixed-scale loss-gap witness supplied by Theorem~\ref{thm:fixed_scale_characterization} at \(\varepsilon_0=\varepsilon_{\mathrm{iso}}^{\mathcal P}\), and define
\[
C_{\mathrm{LR}}\coloneqq e^{2\eta B_{\mathcal P}},
\qquad
\gamma\coloneqq C_{\mathrm{LR}}^{-K}\gamma_{\mathcal P,\varepsilon_{\mathrm{iso}}^{\mathcal P}},
\]
\begin{align}
    N_\gamma
&\coloneqq
\mathcal N\!\big(\mathcal F_{\mathcal P},\gamma/32,\|\cdot\|_{\infty,\supp(\pi_0)}\big), \notag
\\
c_\gamma&\coloneqq \frac{\gamma^2}{128(\log K+2B_{\mathcal P})^2}. \notag
\end{align}
Then
\begin{align}
   &\sup_{p\in\mathcal P}\sup_{T\ge 1}
\Regret_{0,p}^{\mathrm{ERM}}(T) \notag
\\
&\le
\Delta_{\max}^{\mathcal P}
\left(
1
+
\left\lceil \frac{1}{c_\gamma}\log(2N_\gamma)\right\rceil
+
\frac{1}{e^{c_\gamma}-1}
\right). \notag
\end{align}
\end{theorem}

\begin{proof}[Proof sketch.]
Compactness isolates positive selector regret, user diversity yields a fixed truth-centered loss gap, KL-tilt likelihood-ratio control transfers that gap to the observed on-policy slates, and exact ERM plus concentration make substantial-regret rounds summable. For details, see Appendix~\ref{app:deferred_proofs}.
\end{proof}

\begin{theorem}[Logarithmic lower bound without user diversity]
\label{thm:log_lower_without_diversity}
Assume \ref{as:LP1}--\ref{as:LP2}. If decision-relevant user diversity fails in the sense of Definition~\ref{def:decision_relevant_user_diversity}, then, for all sufficiently large \(T\),
\[
\inf_{\mathsf A}
\sup_{p\in\mathcal P}
\Regret_{0,p}^{\mathsf A}(T)
\ge
\frac{\Delta_{\min}^{\mathcal P}\varepsilon_{\mathrm{iso}}^{\mathcal P}}
{8\Delta_{\max}^{\mathcal P}}\,
\log T.
\]
Here the infimum is over the uniformly reference-covered adaptive learner class formalized in Appendix~\ref{app:compact_continuity_details}, specialized to the same likelihood-ratio envelope \(e^{2\eta B_{\mathcal P}}\).
\end{theorem}

\begin{proof}[Proof sketch.]
The lower half of Theorem~\ref{thm:fixed_scale_characterization}, applied at \(\varepsilon_0=\varepsilon_{\mathrm{iso}}^{\mathcal P}\), constructs two head-close truths whose score separation is small but whose temperature-zero recommendations disagree on a set of positive \(d_0\times\rho\)-measure. A standard two-instance testing argument then forces \(\Omega(\log T)\) cumulative regret for any learner in the uniformly reference-covered class. The complete proof is in Appendix~\ref{app:bounded_regret_pf}.
\end{proof}

\section{Offline alignment}
\label{sec:offline_exact_erm_pi0}

We now analyze the offline analogue of the exact-ERM learner. To keep the result under the
same decision framework and structural assumptions as Section~\ref{sec:decision_framework}, we assume that the logged
slates are generated by the reference policy \(\pi_0\). Under this design, the offline population
objective coincides exactly with the truth-centered population loss \(\mathcal L_p\) defined in
Section~\ref{sec:truth_centered_population_risk}, so no additional coverage or concentrability assumption
is needed.

\subsection{Offline model and learner}

Recall the compact score class \(\mathcal F_{\mathcal P}=\{R_p:p\in\mathcal P\}\). Fix any \(R^\circ\in \mathcal F_{\mathcal P}\). Fix \(n\ge 1\) and a truth \(p\in\mathcal P\). The offline dataset is
\[
D_n=\{(X_s,I_s,A_{s,1},\dots,A_{s,K},Y_s)\}_{s=1}^n,
\]
where the observations are i.i.d. and satisfy
\begin{align}
    &(X_s,I_s)\sim d_0\times\rho, \notag
\\
\mathbf A_s&=(A_{s,1},\dots,A_{s,K})\sim \pi_0(\cdot\mid X_s)^{\otimes K}, \notag
\\
Y_s&\sim P_{R_p}(\cdot\mid X_s,I_s,\mathbf A_s). \notag
\end{align}
For \(R\in\mathcal F_{\mathcal P}\), define
\[
\ell_s(R)\coloneqq -\log P_R(Y_s\mid X_s,I_s,\mathbf A_s),
\;\;
\widehat{\mathcal L}_n(R)\coloneqq \frac1n\sum_{s=1}^n \ell_s(R).
\]
By Condition~\ref{as:LP1}, the class \(\mathcal F_{\mathcal P}\) is compact under
\(\|\cdot\|_{\infty,\supp(\pi_0)}\). Since \(A_{s,k}\in \supp(\pi_0(\cdot\mid X_s))\) almost surely
for every \(s\) and \(k\), each map \(R\mapsto \ell_s(R)\) is continuous under this norm on a
full-probability event. Hence, on that event, the empirical objective \(\widehat{\mathcal L}_n\) is
continuous and the offline exact ERM
\begin{equation}
\widehat R_n
\in
\arg\min_{R\in\mathcal F_{\mathcal P}} \widehat{\mathcal L}_n(R)
\label{eq:offline_exact_erm_over_P}
\end{equation}
exists. On the complementary null event, set \(\widehat R_n\coloneqq R^\circ\). We fix a measurable
exact-ERM selection rule in \eqref{eq:offline_exact_erm_over_P} on the full-probability event.

For every \(R\in\mathcal F_{\mathcal P}\),
\begin{equation}
\EE_p[\ell_s(R)]
=
\mathcal L_p(R),
\qquad s=1,\dots,n,
\label{eq:offline_population_risk_identity}
\end{equation}
because each offline observation is an i.i.d. draw from the reference-slate law used in the definition of
\(\mathcal L_p\).

\subsection{Logarithmic accuracy complexity}

\begin{theorem}[Offline exact-ERM: exponential control at scale \(\varepsilon_0\)]
\label{thm:offline_exact_erm_scale_eps}
Assume \ref{as:LP1}--\ref{as:LP2}, and fix \(\varepsilon_0>0\). Define
\[
K_{\varepsilon_0}
\coloneqq
\{(p,q)\in\mathcal P^2:\mathcal G_p(R_q)\ge \varepsilon_0\}.
\]
Set $\gamma_{\mathcal P,\varepsilon_0}
\coloneqq$
\[
\begin{cases}
1, & K_{\varepsilon_0}=\emptyset,\\[0.6em]
\displaystyle
\min\left\{
1,\,
\min_{(p,q)\in K_{\varepsilon_0}}
\bigl(\mathcal L_p(R_q)-\mathcal L_p(R_p)\bigr)
\right\},
& K_{\varepsilon_0}\neq\emptyset.
\end{cases}
\]
Then \(\gamma_{\mathcal P,\varepsilon_0}>0\). Define
\begin{align}
\ell_{\max}&\coloneqq \log K+2B_{\mathcal P}, \notag
\\
N_\gamma
&\coloneqq
\mathcal N\!\big(\mathcal F_{\mathcal P},\gamma_{\mathcal P,\varepsilon_0}/32,
\|\cdot\|_{\infty,\supp(\pi_0)}\big), \notag
\\
c_\gamma&\coloneqq
\frac{\gamma_{\mathcal P,\varepsilon_0}^2}{128\,\ell_{\max}^2}. \notag
\end{align}

Then \(N_\gamma<\infty\), \(c_\gamma>0\), and for every truth \(p\in\mathcal P\) and every
\(n\ge 1\),
\[
\PP_p\big(\mathcal G_p(\widehat R_n)\ge \varepsilon_0\big)
\le
2N_\gamma e^{-c_\gamma n}.
\]
\end{theorem}

\noindent\emph{Proof, including the positivity of \(\gamma_{\mathcal P,\varepsilon_0}\), deferred
to Appendix~\ref{app:offline_fixed_scale}.}

Taking \(\varepsilon_0\coloneqq \varepsilon_{\mathrm{iso}}^{\mathcal P}\) in
Theorem~\ref{thm:offline_exact_erm_scale_eps} and applying
Lemma~\ref{lem:automatic_positive_regret_isolation_P} yields the logarithmic accuracy
bounds below.

Consistent with Section~\ref{sec:temperature_zero_regret}, define the expected
temperature-zero regret of the offline output by
\[
\mathcal R_{0,p}^{\mathrm{off}}(n)
\coloneqq
\EE_p\big[\mathcal G_p(\widehat R_n)\big].
\]

\begin{corollary}[Offline logarithmic accuracy complexity]
\label{cor:offline_log_eps_umr}
Assume \ref{as:LP1}--\ref{as:LP2}. Let \(N_\gamma\) and \(c_\gamma\) be the constants from
Theorem~\ref{thm:offline_exact_erm_scale_eps} corresponding to
$
\varepsilon_0\coloneqq \varepsilon_{\mathrm{iso}}^{\mathcal P}.
$
Then for every \(n\ge 1\),
\[
\sup_{p\in\mathcal P}\mathcal R_{0,p}^{\mathrm{off}}(n)
\le
2\Delta_{\max}^{\mathcal P}N_\gamma e^{-c_\gamma n}.
\]
Consequently, for every
\(\varepsilon\in(0,2\Delta_{\max}^{\mathcal P}N_\gamma]\),
\[
n
\ge
\frac{1}{c_\gamma}
\log\!\left(\frac{2\Delta_{\max}^{\mathcal P}N_\gamma}{\varepsilon}\right)
\quad\Longrightarrow\quad
\sup_{p\in\mathcal P}\mathcal R_{0,p}^{\mathrm{off}}(n)\le \varepsilon.
\]
In particular, the offline sample complexity for expected temperature-zero regret at most
\(\varepsilon\) is \(O(\log(1/\varepsilon))\).
\end{corollary}

\noindent\emph{Proof deferred to Appendix~\ref{app:offline_umr_log_accuracy}.}

The matching offline lower bound is deferred to
Appendix~\ref{sec:offline_matching_lower_bound}; it shows that the
\(O(\log(1/\varepsilon))\) dependence in Corollary~\ref{cor:offline_log_eps_umr} matches the lower bound.

\begin{figure*}[t]
    \centering
    \includegraphics[width=0.92\textwidth]{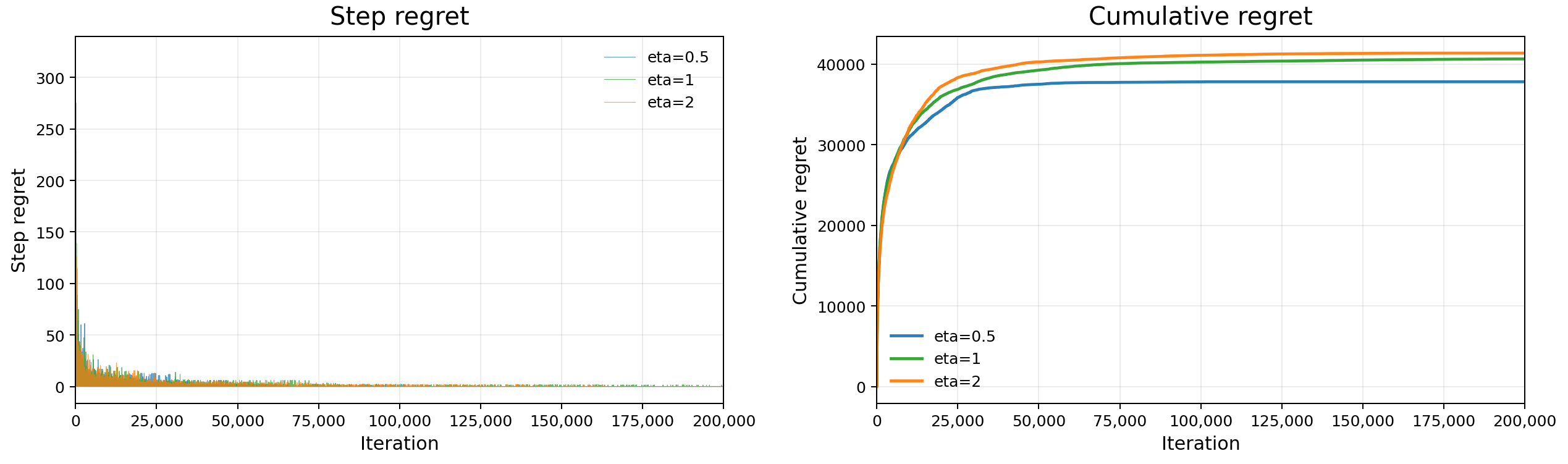}
    \caption{Online personalized alignment on the same full-diversity random instance (DRD=17.46) for
    \(\eta\in\{0.5,1,2\}\). The left panel shows temperature-zero regret per round and the
    right panel shows cumulative temperature-zero regret. The trajectories run for
    \(400{,}000\) rounds; the figure clips the horizontal axis to the first \(200{,}000\)
    rounds, where the plateau is already visible.}
    \label{fig:random_eta_comparison}
\end{figure*}

\section{Experiments}
\label{sec:experiments}

\subsection{Simulation experiments}
\label{sec:simulation_experiments}

To isolate the paper's main claim, we conduct controlled Bradley-Terry experiments in the literature \cite{wu2025greedy, kang2026demystifying}, except for the personalization component. The experiment is designed to test the qualitative prediction of
Theorem~\ref{thm:bounded_regret_under_umr}: when user heads provide decision-relevant diversity, greedy personalized learning should
accumulate regret only during an initial identification phase.

\paragraph{Setup.}
We use a bilinear personalized reward model with context/action dimension \(d=5\), latent
dimension \(J=10\), and \(U=10\) users. The ground-truth reward is
\[
R_i^\star(x,a)=\langle \lambda_i^\star,\phi^\star(x,a)\rangle,
\phi^\star(x,a)
=
\big(x^\top W_1^\star a,\ldots,x^\top W_J^\star a\big).
\]
For each user, the simulator constructs \(100\) contexts and \(100\) candidate actions, and then
fixes this bank for the full run. This constitutes 100,000 user-context-action cases in total. Appendix~\ref{app:simulation_details} gives the
implementation details.

We then run the greedy personalized alignment loop. At each round, the learner observes a user
and context, deploys the KL-tilted sampling policy
$
\pi_t(a\mid x_t,i_t)
\propto
\exp\{\eta\,\widehat R_t(x_t,a,i_t)\},
$
compares one sampled action against a uniformly sampled reference action, observes a binary
Bradley--Terry preference, and refits the shared bilinear reward model. We evaluate the induced temperature-zero regret and report \(\eta\in\{0.5,1,2\}\), horizon \(T=400{,}000\), and one trajectory per value of
\(\eta\) on the same problem instance.

\paragraph{Diversity diagnostic.}
To verify that the instance satisfies the decision-relevant user diversity condition
(Definition~\ref{def:decision_relevant_user_diversity}), we compute a finite-sample proxy of the
diversity modulus.  Let \(\mathcal{H}\) denote the \(10\%\) of realized user-context pairs with the
smallest top-two reward gap \(\max_a R_i^\star(x,a)-\max_{a\neq a^\star}R_i^\star(x,a)\)
(the hardest cases for action identification).
For each \((x_h,i_h)\in\mathcal{H}\), let
\(\Delta\phi_h = \phi^\star(x_h,a^{\star}_h)-\phi^\star(x_h,a^{(2)}_h)\in\mathbb{R}^J\)
be the representation difference between the best and second-best actions.  Define the empirical
hard-case matrix
$
\widehat{H}
=
\frac{1}{|\mathcal{H}|}\sum_{h\in\mathcal{H}}\Delta\phi_h\,\Delta\phi_h^\top,
$
and let \(\widehat{\Sigma}_\lambda\) be the centered empirical covariance of the true user heads
\(\{\lambda_i^\star\}_{i=1}^U\).  The decision-relevant diversity (DRD) diagnostic is
\[
\widehat{\mathfrak{d}}
=
\operatorname{tr}\!\left(\widehat{\Sigma}_\lambda\,\widehat{H}\right).
\]
On the realized instance the DRD diagnostic equals
\(17.46\), confirming that the
diversity condition holds.
\paragraph{Results.}
Figure~\ref{fig:random_eta_comparison} shows the same qualitative pattern for all three
regularization levels. One-step temperature-zero regret is concentrated early in training and
becomes sparse afterward, while cumulative regret quickly flattens. The experiment therefore
matches the main qualitative implication of the theory: after the personalized reward estimate
identifies the correct top action on most realized user-context pairs, continued softened
sampling does not translate into continued temperature-zero regret.

We also run an offline sample-size sweep to test the logarithmic accuracy complexity prediction
of Corollary~\ref{cor:offline_log_eps_umr}. Using the same bilinear Bradley--Terry setup with
\(d=5\), \(J=10\), \(100\) contexts and \(100\) actions per user, we vary the number of users
\(U\in\{10,50,100\}\).
Figure~\ref{fig:offline} shows mean temperature-zero regret on a log scale as a function of
sample size. All three curves decay approximately log-linearly, consistent with the exponential
bound \(\mathcal R_{0,p}^{\mathrm{off}}(n)\le 2\Delta_{\max}^{\mathcal P}N_\gamma e^{-c_\gamma n}\)
from Corollary~\ref{cor:offline_log_eps_umr}. 

\begin{figure}[htbp]
    \centering
    \includegraphics[width=1\columnwidth]{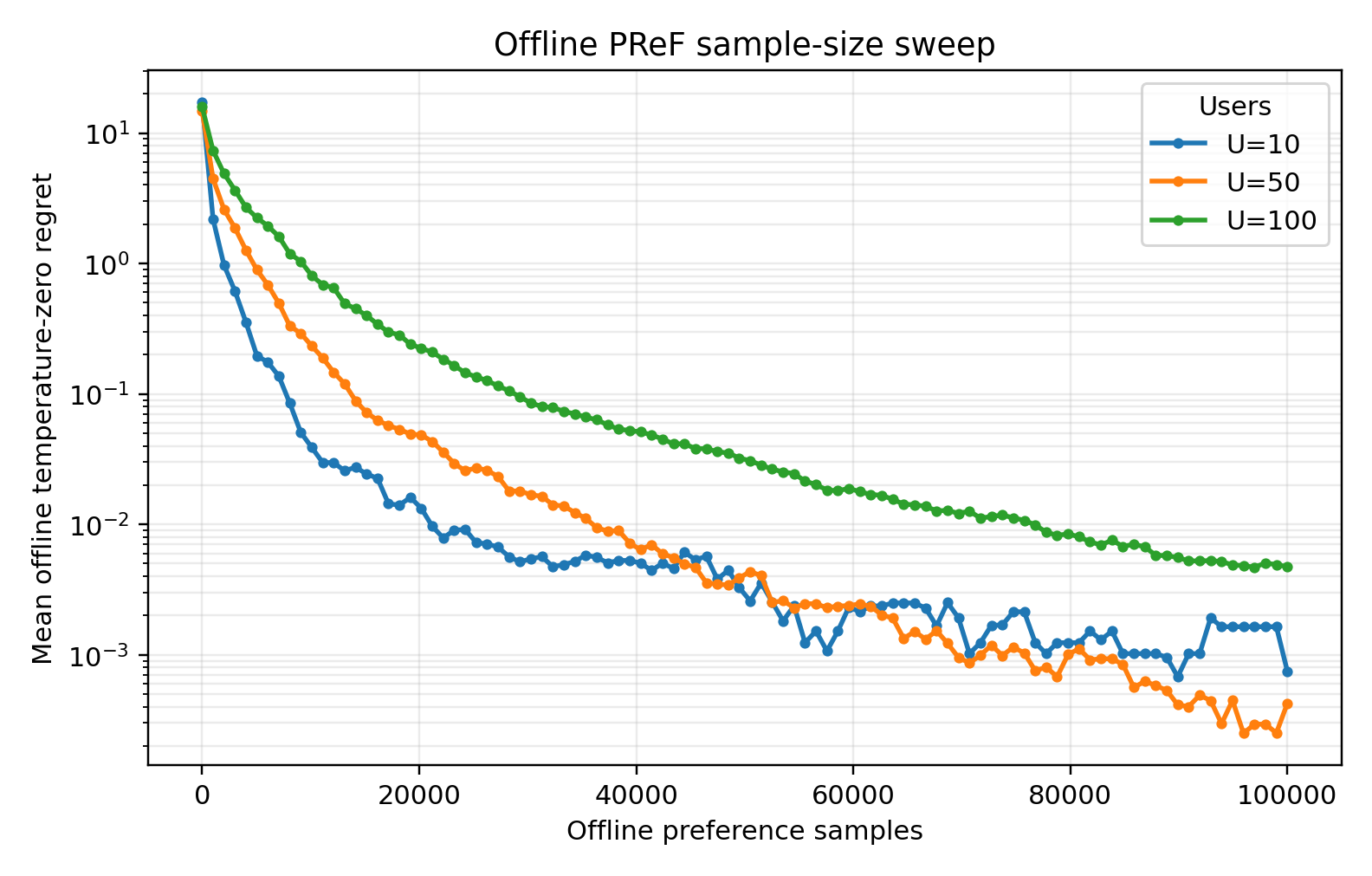}
    \caption{Offline sample-size sweep for \(U\in\{10,50,100\}\) users (\(d=5\), \(J=10\),
    100 contexts, 100 actions per user, log-linear \(y\)-axis). Mean temperature-zero regret
    decays exponentially with sample size across all user counts.}
    \label{fig:offline}
\end{figure}

\section{Conclusion}

We studied personalized alignment gave a sharp
characterization of when efficient learning is possible: the decision-relevant user diversity
condition is both necessary and sufficient for bounded online regret and \(O(\log(1/\varepsilon))\)
offline sample complexity. Simulations confirm that greedy personalized ERM accumulates regret only
during an initial identification phase, and that offline regret decays exponentially with sample size.

\bibliography{references}
\bibliographystyle{icml2026}

\newpage
\appendix
\onecolumn

\section{Online simulation reproducibility details}
\label{app:simulation_details}

This appendix records the exact simulation underlying
Figure~\ref{fig:random_eta_comparison}. 

\paragraph{Problem instance.}
All three trajectories use the same accepted problem seed \(1000\). The simulator first samples
bilinear components \(W_1^\star,\ldots,W_{10}^\star\), random user heads, and then constructs
per-user context and action banks with \(100\) contexts and \(100\) actions per user. The
minimum-gap constructor enforces a raw gap target \(0.01\), and the true heads are scaled by
\(100\). On the realized bank, the minimum top-two reward gap is \(1.0029\), the fifth
percentile is \(1.1200\), the median is \(2.2178\), and the mean is \(2.8848\). The user-head
second moment has minimum eigenvalue \(3.0451\). The reported decision-relevant diversity
diagnostic is \(17.4568\), or \(0.001746\) after dividing by the squared head scale.
The diagnostic is computed as \(\widehat{\mathfrak{d}}=\operatorname{tr}(\widehat{\Sigma}_\lambda\,\widehat{H})\),
where \(\widehat{\Sigma}_\lambda\) is the empirical centered user-head covariance and \(\widehat{H}\)
is the average outer product of representation differences \(\Delta\phi_h=\phi^\star(x_h,a^\star_h)-\phi^\star(x_h,a^{(2)}_h)\)
restricted to the \(10\%\) of user-context pairs with the smallest top-two reward gap (the ``hard'' subset).

\paragraph{Online data collection and fitting.}
At each round, the user is sampled uniformly from the ten users and the context is sampled
uniformly from that user's context bank. The first action is sampled from the learner's
KL-tilted policy with the specified value of \(\eta\), and the second action is sampled from the
uniform reference distribution. The binary label is drawn from the Bradley--Terry probability
\[
\sigma\!\left(R_i^\star(x,a_1)-R_i^\star(x,a_2)\right).
\]
In these runs, \(\eta\) affects only the online sampling distribution; it is not multiplied into
the preference label logit. The learner fits the rank-\(10\) bilinear Bradley--Terry model by
maximum likelihood using the gradient-SVD initializer, alternating representation/head updates,
Newton updates for the user heads, ridge \(10^{-3}\), maximum \(40\) representation updates and
\(25\) head updates per fit, and tolerance \(10^{-9}\). The refit schedule is proportional with
progress divisor \(5000\), and the fit history is not truncated.

\paragraph{Evaluation.}
Regret is evaluated on the realized online arrivals. For each round \(t\), the stored one-step
quantity is
\[
r_t
=
\max_{a\in\cA_{i_t}}R_{i_t}^\star(x_t,a)
-
R_{i_t}^\star\!\left(x_t,\arg\max_{a\in\cA_{i_t}}\widehat R_t(x_t,a,i_t)\right),
\]
and cumulative regret is \(G_T=\sum_{t=1}^T r_t\). Table~\ref{tab:random_eta_summary_app}
summarizes the checkpoints used in Figure~\ref{fig:random_eta_comparison}. Since each
\(\eta\) is represented by one long trajectory, the figure is intended as a qualitative
sanity check of the bounded-regret behavior rather than as an uncertainty-quantified benchmark.

\begin{table}[h]
\centering
\caption{Summary of the three online trajectories in Figure~\ref{fig:random_eta_comparison}.}
\label{tab:random_eta_summary_app}
\begin{tabular}{lrrrr}
\toprule
\(\eta\) & \(G_{200\mathrm{k}}\) & \(G_{400\mathrm{k}}\) & last \(r_t>0\) & subst. frac. \\
\midrule
0.5 & 37,805 & 37,805 & 102,781 & 1.48\% \\
1.0 & 40,630 & 40,902 & 379,720 & 1.80\% \\
2.0 & 41,351 & 41,362 & 225,496 & 1.88\% \\
\bottomrule
\end{tabular}
\end{table}

\section{Offline sample-size sweep reproducibility details}
\label{app:offline_simulation_details}

This appendix records the exact simulation underlying Figure~\ref{fig:offline}.

\paragraph{Problem instances.}
Each of the three curves uses a separate problem instance constructed with the same accepted
problem seed \(1000\) and the same bilinear structure: context/action dimension \(d=5\), latent
dimension \(J=10\), and \(100\) contexts and \(100\) candidate actions per user.  The number of
users varies across the three instances: \(U\in\{10,50,100\}\).  The minimum-gap constructor
enforces a raw gap target of \(0.01\) at the base scale, and the true user heads are rescaled
to a common head scale of \(20\) at run time.
Table~\ref{tab:offline_instance_stats} reports the realized bank statistics for each instance.

\begin{table}[h]
\centering
\caption{Realized problem-instance statistics for the offline sweep (Figure~\ref{fig:offline}).
All gap and DRD values are reported at the run-time head scale of \(20\).}
\label{tab:offline_instance_stats}
\begin{tabular}{lrrr}
\toprule
 & \(U=10\) & \(U=50\) & \(U=100\) \\
\midrule
Min top-two gap          & 0.2006 & 0.2000 & 0.2001 \\
5th-pct top-two gap      & 0.2240 & 0.2247 & 0.2221 \\
Median top-two gap       & 0.4436 & 0.4686 & 0.4473 \\
Mean top-two gap         & 0.5770 & 0.5959 & 0.5773 \\
Head 2nd-moment min eig  & 0.1218 & 13.125 & 22.757 \\
DRD \(\widehat{\mathfrak{d}}\) & 0.6983 & 1.4750 & 1.5473 \\
DRD (scale-free)         & 0.001746 & 0.003687 & 0.003868 \\
\bottomrule
\end{tabular}
\end{table}

\paragraph{Offline data collection and fitting.}
Each instance uses \(n=100{,}000\) preference samples logged from the uniform reference
distribution over the per-user action bank.  Each preference is a binary Bradley--Terry label
drawn from
\[
\sigma\!\left(R_i^\star(x,a_1)-R_i^\star(x,a_2)\right),
\]
where both actions are drawn uniformly at random from the user's action bank.  The learner fits
the rank-\(10\) bilinear Bradley--Terry model by maximum likelihood using the gradient-SVD
initializer, alternating representation and head updates, Newton updates for user heads, ridge
\(10^{-3}\), maximum \(40\) representation updates and \(25\) head updates per fit, and
tolerance \(10^{-9}\).  The refit schedule is proportional with progress divisor \(5{,}000\).

\paragraph{Evaluation.}
For each of \(100\) evenly spaced prefix sizes \(n\in\{0,1010,2020,\ldots,100{,}000\}\), the
model is refit on the first \(n\) samples and temperature-zero regret is evaluated on every
realized user-context-action triple in the bank.  The quantity reported in
Figure~\ref{fig:offline} is mean temperature-zero regret, averaged uniformly over all
user-context pairs.

\section{Related works}
\label{sec:related_works}

\paragraph{Alignment from preference feedback.}
Reinforcement learning from human feedback learns from pairwise or slate-level human comparisons and then optimizes a language-model policy against the learned preference signal, often under a KL constraint to a reference model \citep{christiano2017deep,ziegler2019fine,stiennon2020learning,ouyang2022training}. Direct preference optimization and related objectives remove the explicit reward-modeling stage by exploiting the closed-form relation between KL-regularized rewards and optimal policies, while recent theoretical work studies broader pairwise-preference and game-theoretic formulations \citep{rafailov2023direct,azar2024general,munos2024nash}. Our paper keeps an explicit reward representation because the central question is not only how to optimize a policy from preferences, but when personalized preference data statistically identifies the decision-relevant directions of a shared reward representation.

\paragraph{Pluralistic and personalized alignment.}
A growing body of work argues that standard alignment to a single aggregate preference distribution is insufficient when users disagree systematically \citep{bakker2022fine,sorensen2024roadmap,conitzer2024social,kirk2024prism}. Empirical and benchmark work further shows that heterogeneous, culturally dependent, or individual-specific preferences create evaluation and data-collection challenges that are not captured by one-size-fits-all reward modeling \citep{zollo2025personallm}. Our work takes this pluralistic motivation as the starting point, but focuses on a narrower statistical question: when does user heterogeneity help or fail to help a greedy alignment learner identify the temperature-zero recommendation for each user?

\paragraph{Low-dimensional personalized reward models.}
Several recent methods model personalized rewards through a low-dimensional latent structure, including low-rank reward modeling, reward factorization, latent-variable personalization, ideal-point or mixture models, and shared reward features \citep{bose2025lore,shenfeld2025language,poddar2024personalizing,chen2024pal,barreto2025capturing}. These works provide strong empirical and architectural evidence that user-specific rewards can often be represented through shared factors and user-specific coefficients. Our model uses the same shared-representation/user-head inductive bias, but our contribution is theoretical: we characterize the decision-relevant user-diversity condition under which that structure is sufficient for bounded online temperature-zero regret.

\paragraph{Policy-level and inference-time personalization.}
A complementary line personalizes the policy or decoding procedure directly, for example by training multiple preference-specialized policies and merging them, conditioning on group preferences, optimizing multiple objectives, or steering generation at inference time \citep{jang2023personalized,zhao2023group,zhou2024beyond,yang2024rewards,chen2024pad}. These approaches address the implementation side of personalization: how to produce outputs that reflect a specified user, group, or reward tradeoff. Our analysis is orthogonal to that design problem: we study the sampling and identifiability problem induced by online preference feedback, and show that the relevant diversity is not generic demographic variation but variation spanning the representation directions that can alter greedy recommendations.

\paragraph{Online and iterative alignment.}
Recent online RLHF theory studies iterative preference learning under KL constraints, both in reward-based Bradley--Terry settings and in more general preference-oracle models \citep{xiong2024iterative,ye2024online}. Greedy or empirically driven sampling has also been shown to be unexpectedly effective for RLHF-style objectives, and recent temperature-zero analyses separate the cost of identifying the best response from the exploration induced by softened deployment policies \citep{wu2025greedy,kang2026demystifying}. We extend this decision-centric view to personalized alignment: the learner must not only collect informative actions, but also observe a user population whose heads reveal the shared representation directions that matter for individualized recommendations.

\paragraph{Implicit exploration and diversity in contextual bandits.}
The possibility that greedy learning succeeds without explicit optimism is closely related to contextual-bandit results under covariate diversity, smoothed contexts, natural exploration, good representations, and anti-concentration \citep{bastani2021mostly,kannan2018smoothed,hao2020adaptive,papini2021leveraging,raghavan2020greedy,kim2024local}. These conditions show that randomness in contexts or representation geometry can supply the exploration that a greedy policy does not deliberately enforce. Our user-diversity modulus is the personalized-preference analogue of such self-exploration conditions, but it is restricted to alternatives that both remain statistically nearby and change the temperature-zero decision.

\paragraph{Preference models, ambiguous comparisons, and margins.}
The Bradley--Terry and multinomial-logit models are standard probabilistic models for comparisons and choices \citep{bradley1952rank,luce1959individual}. In modern RLHF data, however, pairwise labels can be noisy, low-margin, tied, or ambiguous, and recent work proposes data-quality adjustment, reward-margin modeling, granular feedback, or explicit treatment of ties \citep{wang2024reward,qin2024towards,kim2024margin,liu2024rewardties}. Our uniform supportwise gap condition should be read in this decision-theoretic spirit: the theory isolates stable recommendation decisions and does not require that every weak human comparison be perfectly separable.

\paragraph{Offline preference learning and coverage.}
Offline preference optimization methods such as DPO and related theoretical frameworks rely on the logged preference distribution containing enough information about the policies or responses being evaluated \citep{rafailov2023direct,azar2024general,xiong2024iterative}. In our offline setting, slates are logged from the reference policy \(\pi_0\), so the population objective coincides with the truth-centered loss used in the online proof and no additional concentrability factor is needed. This yields exponential fixed-scale control and \(O(\log(1/\varepsilon))\) expected temperature-zero regret complexity, while the matching two-instance lower bound shows that the logarithmic dependence is sharp on nontrivial subclasses.

\paragraph{Summary of the distinction.}
Prior work establishes that personalized alignment is empirically important, that low-dimensional reward structure is a useful modeling principle, and that online preference optimization can be statistically efficient. This paper connects these threads by proving a necessary-and-sufficient condition for personalized greedy alignment: bounded online temperature-zero regret holds exactly when nearby decision-relevant representation alternatives are separated by the user-head diversity of the population. The same compact decision-theoretic framework also explains the offline analogue, where reference-logged preference data gives logarithmic accuracy complexity and an exponential testing lower bound.

\section{Technical lemmas}
\label{app:supporting_lemmas}
\label{app:linear_neural_examples}

Several technical lemmas in this appendix are user-indexed versions of the centering, MNL-loss, slate-domination, support-upgrade, and exact-ERM concentration lemmas in \citet{kang2026demystifying}. We retain the statements to make the personalized notation unambiguous. When the proof is unchanged after replacing a context \(x\) by a user-context pair \((x,i)\), we cite the non-personalized proof rather than repeat it.

\begin{lemma}[$\pi_0$-centering is w.l.o.g.\ for MNL likelihoods]
\label{lem:pi0_centering_wlog_psi}
Fix any score function $R:\cX\times\cA\times\cU\to\R$ and define its $\pi_0$-centered version
\[
\bar R(x,a,i)\ \coloneqq\ R(x,a,i)-m_R(x,i),
\qquad
m_R(x,i)\ \coloneqq\ \EE_{a'\sim\pi_0(\cdot\mid x)}[R(x,a',i)].
\]
Then for every $(x,i)$ and every slate $\mathbf a=(a_1,\dots,a_K)$:
\begin{enumerate}[label=(\roman*),leftmargin=2em]
\item $P_R(\cdot\mid x,i,\mathbf a)=P_{\bar R}(\cdot\mid x,i,\mathbf a)$,
\item $\pi_R(\cdot\mid x,i)=\pi_{\bar R}(\cdot\mid x,i)$ for KL-tilts,
\item $\arg\max_{a\in\supp(\pi_0(\cdot\mid x))}R(x,a,i)=\arg\max_{a\in\supp(\pi_0(\cdot\mid x))}\bar R(x,a,i)$.
\end{enumerate}
\end{lemma}

\begin{proof}
This is the user-indexed version of the \(\pi_0\)-centering invariance lemma in \citet{kang2026demystifying}. Conditional on a fixed \((x,i)\), the centering term \(m_R(x,i)\) is independent of the action. It therefore cancels in every MNL probability ratio and in the normalizing constant of the KL tilt, and leaves supportwise argmax sets unchanged.
\end{proof}

\begin{lemma}[Boundedness and Lipschitzness of the MNL log-loss]
\label{lem:mnl_loss_basic}
Let $\ell(\mathbf v,y)=\log(\sum_{k=1}^K e^{v_k})-v_y$ be the MNL negative log-likelihood.
If $\max_k v_k-\min_k v_k\le 2B$, then for every $y$,
\[
0\ \le\ \ell(\mathbf v,y)\ \le\ \log K + 2B.
\]
Moreover, for any $\mathbf v,\mathbf v'\in\R^K$ and any $y$,
\[
|\ell(\mathbf v,y)-\ell(\mathbf v',y)|\ \le\ 2\|\mathbf v-\mathbf v'\|_\infty.
\]
\end{lemma}

\begin{proof}
This is the MNL log-loss envelope and Lipschitz lemma used in \citet{kang2026demystifying}. The bound follows from \(\log\sum_k e^{v_k}\le \log K+\max_k v_k\) and \(v_y\ge \min_k v_k\). The Lipschitz claim follows because \(\nabla_{\mathbf v}\ell(\mathbf v,y)=\softmax(\mathbf v)-e_y\) has \(\ell_1\)-norm at most \(2\), so the mean-value theorem gives the displayed \(2\|\cdot\|_\infty\) bound.
\end{proof}

\begin{lemma}[One-step excess loss equals choice-model KL]
\label{lem:excess_is_kl_support}
Fix any realized $(x,i,\mathbf a)$, any truth score $S^\star$, and any candidate score $S$.
Let
\[
p^\star=P_{S^\star}(\cdot\mid x,i,\mathbf a),
\qquad
p=P_S(\cdot\mid x,i,\mathbf a).
\]
Then
\[
\EE_{y\sim p^\star}\big[\ell(\mathbf v_S,y)-\ell(\mathbf v_{S^\star},y)\big]
=
\KL(p^\star\|p),
\]
where
\[
\mathbf v_S=(S(x,a_1,i),\dots,S(x,a_K,i)),
\qquad
\mathbf v_{S^\star}=(S^\star(x,a_1,i),\dots,S^\star(x,a_K,i)).
\]
\end{lemma}

\begin{proof}
This is the standard log-score identity used in \citet{kang2026demystifying}: under the true MNL law \(p^\star\), the expected excess negative log-likelihood of a candidate law \(p\) equals \(\sum_k p_k^\star\log(p_k^\star/p_k)=\KL(p^\star\|p)\).
\end{proof}

\begin{lemma}[Slate expectation domination]
\label{lem:slate_domination_likelihood_ratio}
Fix any $x$ and any constant $C\ge1$. Let $\pi(\cdot\mid x)$ be a distribution satisfying
\[
\pi(a\mid x)\ge C^{-1}\pi_0(a\mid x)
\qquad
\pi_0(\cdot\mid x)\text{-a.s.}
\]
Let $\mathbf A\sim \pi(\cdot\mid x)^{\otimes K}$ and
$\mathbf A_0\sim \pi_0(\cdot\mid x)^{\otimes K}$.
Then for every nonnegative measurable $g$,
\[
\EE\big[g(\mathbf A)\big]\ \ge\ C^{-K}\,\EE\big[g(\mathbf A_0)\big].
\]
\end{lemma}

\begin{proof}
This is the slate-domination lemma from \citet{kang2026demystifying}. The coordinatewise lower bound \(d\pi/d\pi_0\ge C^{-1}\) implies \(d\pi^{\otimes K}/d\pi_0^{\otimes K}\ge C^{-K}\) on slates. Integrating the nonnegative function \(g\) with respect to the two product measures gives the claim.
\end{proof}

\begin{lemma}[Likelihood-ratio bound for KL tilts]
\label{lem:greedy_kl_tilt_lr}
Fix a tilt parameter \(\eta>0\). Let \(S:\cX\times\cA\times\cU\to\R\) satisfy
\[
|S(x,a,i)|\le B
\qquad
\pi_0(\cdot\mid x)\text{-a.s. }a,
\quad \forall (x,i)\in\cX\times\cU.
\]
Let
\[
\pi_S(a\mid x,i)
=
\frac{\pi_0(a\mid x)\exp(\eta S(x,a,i))}
{\int_{\cA}\pi_0(a'\mid x)\exp(\eta S(x,a',i))\,da'}.
\]
Then
\[
e^{-2\eta B}
\le
\frac{d\pi_S(\cdot\mid x,i)}{d\pi_0(\cdot\mid x)}(a)
\le
e^{2\eta B},
\qquad
\pi_0(\cdot\mid x)\text{-a.s.}
\]
for every \(x\in\cX\) and \(i\in\cU\).
\end{lemma}

\begin{proof}
Fix \((x,i)\). Since \(|S(x,a,i)|\le B\) for \(\pi_0(\cdot\mid x)\)-a.e.\ \(a\), the normalizing constant
\[
Z_{S}(x,i)
\coloneqq
\int_{\cA}\pi_0(a'\mid x)\exp(\eta S(x,a',i))\,da'
\]
satisfies
\[
e^{-\eta B}\le Z_S(x,i)\le e^{\eta B}.
\]
Therefore, for \(\pi_0(\cdot\mid x)\)-a.e.\ \(a\),
\[
e^{-2\eta B}
\le
\frac{e^{\eta S(x,a,i)}}{Z_S(x,i)}
\le
e^{2\eta B}.
\]
Since
\[
\frac{d\pi_S(\cdot\mid x,i)}{d\pi_0(\cdot\mid x)}(a)
=
\frac{e^{\eta S(x,a,i)}}{Z_S(x,i)},
\]
the claim follows.
\end{proof}

\begin{lemma}[Expected choice-KL lower bound]
\label{lem:expected_kl_ge_var_support}
Fix $(x,i)$ and a distribution $q(\cdot\mid x)$ on $\cA$ such that
$q(\cdot\mid x)\ll \pi_0(\cdot\mid x)$.
Sample a slate $\mathbf a=(A_1,\dots,A_K)$ i.i.d.\ from $q(\cdot\mid x)$.
Let
\[
\Delta(a)=S(x,a,i)-S^\star(x,a,i)
\]
for a truth score $S^\star$ and a candidate score $S$, and assume
\[
\sup_{a\in \supp(\pi_0(\cdot\mid x))} |S(x,a,i)|\le B,
\qquad
\sup_{a\in \supp(\pi_0(\cdot\mid x))} |S^\star(x,a,i)|\le B.
\]
Then
\[
\EE_{\mathbf a\sim q(\cdot\mid x)^{\otimes K}}
\Big[
\KL\big(P_{S^\star}(\cdot\mid x,i,\mathbf a)\,\|\,P_S(\cdot\mid x,i,\mathbf a)\big)
\Big]
\ \ge\
c_{\mathrm{mnl}}\,
\Var_{a\sim q(\cdot\mid x)}\!\big(\Delta(a)\big),
\]
with
\[
c_{\mathrm{mnl}}=\frac{e^{-2B}}{2}\cdot\frac{K-1}{K}.
\]
\end{lemma}

\begin{proof}
Because $q(\cdot\mid x)\ll \pi_0(\cdot\mid x)$, the displayed bounds on $S$ and $S^\star$ hold
for each sampled $A_k$, $q$-a.s. Write
\[
\Delta_k=\Delta(A_k).
\]
Fix a realized slate $\mathbf a=(a_1,\dots,a_K)$ in this full-measure set and set
\[
\mathbf u=(S^\star(x,a_1,i),\dots,S^\star(x,a_K,i)),
\qquad
\mathbf v=(S(x,a_1,i),\dots,S(x,a_K,i)).
\]
Each coordinate of $\mathbf u$ and $\mathbf v$ lies in $[-B,B]$, so every point on the line segment
joining $\mathbf u$ and $\mathbf v$ has coordinate range at most $2B$.

Let
\[
A(\mathbf z)=\log\Big(\sum_{k=1}^K e^{z_k}\Big).
\]
The induced MNL distributions satisfy
\[
\KL(P_{\mathbf u}\|P_{\mathbf v})
=
A(\mathbf v)-A(\mathbf u)-\langle \nabla A(\mathbf u),\mathbf v-\mathbf u\rangle.
\]
By Taylor's theorem with integral remainder, there exists a point $\widetilde{\mathbf z}$ on the
segment joining $\mathbf u$ and $\mathbf v$ such that
\[
\KL(P_{\mathbf u}\|P_{\mathbf v})
=
\frac12(\mathbf v-\mathbf u)^\top \nabla^2 A(\widetilde{\mathbf z})(\mathbf v-\mathbf u).
\]
Let
\[
\widetilde p=\softmax(\widetilde{\mathbf z}).
\]
Then
\[
\nabla^2 A(\widetilde{\mathbf z})=\diag(\widetilde p)-\widetilde p\,\widetilde p^\top,
\]
so for any $z\in\R^K$,
\[
z^\top\nabla^2 A(\widetilde{\mathbf z})z
=
\Var_{k\sim \widetilde p}(z_k).
\]
Applying this with $z_k=\Delta_k$ gives
\[
\KL(P_{\mathbf u}\|P_{\mathbf v})
=
\frac12\,\Var_{k\sim \widetilde p}(\Delta_k).
\]

Since $\max_k \widetilde z_k-\min_k \widetilde z_k\le 2B$, each softmax coordinate satisfies
\[
\widetilde p_k
=
\frac{e^{\widetilde z_k}}{\sum_{\ell=1}^K e^{\widetilde z_\ell}}
\ge
\frac{e^{-B}}{K e^{B}}
=
\frac{e^{-2B}}{K}.
\]
Using the variational representation of variance,
\[
\Var_{k\sim \widetilde p}(z_k)
=
\inf_{b\in\R}\sum_{k=1}^K \widetilde p_k (z_k-b)^2
\ge
\frac{e^{-2B}}{K}\inf_{b\in\R}\sum_{k=1}^K (z_k-b)^2
=
e^{-2B}\Var_{k\sim \mathrm{Unif}(1,\dots,K)}(z_k).
\]
Therefore
\[
\KL(P_{\mathbf u}\|P_{\mathbf v})
\ge
\frac{e^{-2B}}{2}\Var_{k\sim\mathrm{Unif}(1,\dots,K)}(\Delta_k).
\]

Now take expectation over the random slate. Since $A_1,\dots,A_K$ are i.i.d.\ from
$q(\cdot\mid x)$, the random variables $\Delta(A_1),\dots,\Delta(A_K)$ are i.i.d. Denote
\[
\mu=\EE_{A\sim q(\cdot\mid x)}[\Delta(A)],
\qquad
\sigma^2=\Var_{A\sim q(\cdot\mid x)}(\Delta(A)).
\]
Then
\[
\EE\Big[\Var_{k\sim\mathrm{Unif}}(\Delta(A_k))\Big]
=
\EE\Big[\frac1K\sum_{k=1}^K \Delta(A_k)^2 - \Big(\frac1K\sum_{k=1}^K \Delta(A_k)\Big)^2\Big].
\]
The first term equals $\mu^2+\sigma^2$.
For the second term,
\[
\EE\Big[\Big(\frac1K\sum_{k=1}^K \Delta(A_k)\Big)^2\Big]
=
\Var\Big(\frac1K\sum_{k=1}^K \Delta(A_k)\Big)+\mu^2
=
\frac{\sigma^2}{K}+\mu^2.
\]
Hence
\[
\EE\Big[\Var_{k\sim\mathrm{Unif}}(\Delta(A_k))\Big]
=
\Big(\mu^2+\sigma^2\Big)-\Big(\mu^2+\frac{\sigma^2}{K}\Big)
=
\frac{K-1}{K}\sigma^2.
\]
Combining the previous displays gives
\[
\EE_{\mathbf a\sim q(\cdot\mid x)^{\otimes K}}
\Big[
\KL\big(P_{S^\star}(\cdot\mid x,i,\mathbf a)\,\|\,P_S(\cdot\mid x,i,\mathbf a)\big)
\Big]
\ge
\frac{e^{-2B}}{2}\cdot\frac{K-1}{K}\cdot
\Var_{a\sim q(\cdot\mid x)}(\Delta(a)).
\]
\end{proof}

\begin{lemma}[Misrecommendation implies a supportwise score error]
\label{lem:misrec_implies_sup_error}
Let $S^\star$ and $S$ be centered scores, fix $(x,i)$, and assume that the supportwise maximizer
\[
a^\star\in\arg\max_{a\in\supp(\pi_0(\cdot\mid x))} S^\star(x,a,i)
\]
is unique and satisfies the gap condition
\[
S^\star(x,a^\star,i)
-
\sup_{a\in\supp(\pi_0(\cdot\mid x)),\ a\neq a^\star}
S^\star(x,a,i)
\ \ge\ \Delta_{\min}.
\]
If
\[
a_S(x,i)\neq a^\star,
\]
then
\[
\sup_{a\in\supp(\pi_0(\cdot\mid x))}
|S(x,a,i)-S^\star(x,a,i)|
\ \ge\ \frac{\Delta_{\min}}{2}.
\]
\end{lemma}

\begin{proof}
Let
\[
\hat a=a_S(x,i).
\]
By optimality of $\hat a$ under $S$,
\[
S(x,\hat a,i)\ \ge\ S(x,a^\star,i).
\]
By the assumed supportwise gap for $S^\star$,
\[
S^\star(x,a^\star,i)\ \ge\ S^\star(x,\hat a,i)+\Delta_{\min}.
\]
Subtract the second inequality from the first:
\[
\big(S(x,\hat a,i)-S^\star(x,\hat a,i)\big)
-
\big(S(x,a^\star,i)-S^\star(x,a^\star,i)\big)
\ \ge\ \Delta_{\min}.
\]
Therefore at least one of the two terms has absolute value at least $\Delta_{\min}/2$, so
\[
\sup_{a\in\supp(\pi_0(\cdot\mid x))}
|S(x,a,i)-S^\star(x,a,i)|
\ \ge\ \frac{\Delta_{\min}}{2}.
\]
\end{proof}

\begin{lemma}[Softmax KL upper bound]
\label{lem:softmax_kl_upper_bound}
For any \(u,v\in\R^K\), let \(p_u=\softmax(u)\) and \(p_v=\softmax(v)\). Then
\[
\KL(p_u\|p_v)\le \frac12 \|u-v\|_2^2.
\]
\end{lemma}

\begin{proof}
Let
\[
A(z)\coloneqq \log\!\Big(\sum_{k=1}^K e^{z_k}\Big).
\]
Then \(\nabla A(z)=\softmax(z)\), so \(p_u=\nabla A(u)\) and \(p_v=\nabla A(v)\).
Moreover,
\[
\KL(p_u\|p_v)=A(v)-A(u)-\langle \nabla A(u),v-u\rangle.
\]
By Taylor's theorem with integral remainder,
\[
\KL(p_u\|p_v)
=
\int_0^1 (1-s)\,(v-u)^\top \nabla^2A(u+s(v-u))(v-u)\,ds.
\]
Now
\[
\nabla^2A(z)=\diag(\softmax(z))-\softmax(z)\softmax(z)^\top.
\]
For any \(w\in\R^K\),
\[
w^\top \nabla^2A(z)w
=
\Var_{J\sim \softmax(z)}(w_J)
\le
\EE_{J\sim \softmax(z)}[w_J^2]
\le
\|w\|_2^2.
\]
Hence \(\|\nabla^2A(z)\|_{\mathrm{op}}\le 1\) for all \(z\), and therefore
\[
\KL(p_u\|p_v)
\le
\int_0^1 (1-s)\,ds\ \|u-v\|_2^2
=
\frac12 \|u-v\|_2^2.
\]
\end{proof}

\begin{lemma}[Deviation tail bound via an $\epsilon$-net]
\label{lem:bt_tail_support}
Assume \(\mathcal F\) is bounded in \(\|\cdot\|_{\infty,\supp(\pi_0)}\) by \(B\), and that at each
round every slate coordinate is sampled from a distribution absolutely continuous with respect to
\(\pi_0(\cdot\mid x)\).
Let
\[
\mathcal L_t(R)\ \coloneqq\ \frac1t\sum_{s=1}^t \EE[\ell_s(R)\mid \mathscr H_{s-1}],
\qquad
b_t\ \coloneqq\ \sup_{R\in\mathcal F}\big|\widehat{\mathcal L}_t(R)-\mathcal L_t(R)\big|.
\]
Fix \(\epsilon>0\) and let \(\mathcal C_\epsilon\) be a finite \(\epsilon\)-net of \(\mathcal F\)
in \(\|\cdot\|_{\infty,\supp(\pi_0)}\).
Then for every \(t\ge 1\) and every \(u>0\),
\[
\PP\big(b_t \ge u + 4\epsilon\big)
\ \le\
2\,|\mathcal C_\epsilon|\,
\exp\!\left(-\frac{t u^2}{2\ell_{\max}^2}\right),
\qquad
\ell_{\max}\coloneqq \log K+2B.
\]
\end{lemma}

\begin{proof}
This is the user-indexed version of the \(\epsilon\)-net martingale concentration lemma in \citet{kang2026demystifying}. For each fixed score, the centered losses form a bounded martingale-difference sequence, so Azuma--Hoeffding gives the displayed tail bound on a finite net. The supportwise norm controls realized losses because every slate coordinate is absolutely continuous with respect to \(\pi_0(\cdot\mid x_s)\); Lemma~\ref{lem:mnl_loss_basic} then transfers the net bound to all of \(\mathcal F\), producing the additional \(4\epsilon\) term.
\end{proof}

\begin{lemma}[Compact neural realizations induce a compact centered score class]
\label{lem:compact_nn_to_score}
Assume \(\Omega_\psi\subset\R^d\) is compact, \(\psi(x,a)\in\Omega_\psi\) for every \((x,a)\) with \(a\in \cA_0(x)\), and
\[
\mathcal H\subset C(\Omega_\psi;\R^J)
\]
is compact under
\[
\|h\|_{\infty,\Omega_\psi}\coloneqq \sup_{z\in\Omega_\psi}\|h(z)\|_2.
\]
Define
\[
\phi_h(x,a)=h(\psi(x,a)),
\qquad
\bar\phi_h(x,a)=\phi_h(x,a)-\EE_{A\sim\pi_0(\cdot\mid x)}[\phi_h(x,A)],
\]
and
\[
R_h(x,a,i)=\langle \lambda_i^\star,\bar\phi_h(x,a)\rangle.
\]
Let
\[
L_\lambda\coloneqq \sup_{i\in\cU}\|\lambda_i^\star\|_2.
\]
Then the sets \(\{\phi_h:h\in\mathcal H\}\) and \(\{\bar\phi_h:h\in\mathcal H\}\) are compact under
\(\|\cdot\|_{\infty,\supp(\pi_0)}\), and
\[
\mathcal F_{\mathrm{sh}}=\{R_h:h\in\mathcal H\}
\]
is compact under \(\|\cdot\|_{\infty,\supp(\pi_0)}\).
\end{lemma}

\begin{proof}
Let
\[
\Omega_{\mathrm{sa}}\coloneqq \{(x,a): x\in\cX,\ a\in \supp(\pi_0(\cdot\mid x))\}.
\]
Define the composition map
\[
T_1:C(\Omega_\psi;\R^J)\to \ell^\infty(\Omega_{\mathrm{sa}};\R^J),
\qquad
(T_1 h)(x,a)\coloneqq h(\psi(x,a)).
\]
Because \(\psi(x,a)\in \Omega_\psi\) for every \((x,a)\in \Omega_{\mathrm{sa}}\),
\[
\|T_1 h-T_1 h'\|_{\infty,\supp(\pi_0)}
=
\sup_{(x,a)\in\Omega_{\mathrm{sa}}}\|h(\psi(x,a))-h'(\psi(x,a))\|_2
\le
\sup_{z\in\Omega_\psi}\|h(z)-h'(z)\|_2
=
\|h-h'\|_{\infty,\Omega_\psi}.
\]
Hence \(T_1\) is continuous, so \(\{\phi_h:h\in\mathcal H\}=T_1(\mathcal H)\) is compact.

Next define the centering map
\[
T_2:\ell^\infty(\Omega_{\mathrm{sa}};\R^J)\to \ell^\infty(\Omega_{\mathrm{sa}};\R^J),
\qquad
(T_2\phi)(x,a)\coloneqq \phi(x,a)-\EE_{A\sim\pi_0(\cdot\mid x)}[\phi(x,A)].
\]
For any bounded \(\phi,\phi'\),
\begin{align*}
\|T_2\phi-T_2\phi'\|_{\infty,\supp(\pi_0)}
&\le
\sup_{(x,a)\in\Omega_{\mathrm{sa}}}\|\phi(x,a)-\phi'(x,a)\|_2
+
\sup_{x\in\cX}
\left\|
\EE_{A\sim\pi_0(\cdot\mid x)}[\phi(x,A)-\phi'(x,A)]
\right\|_2\\
&\le
2\|\phi-\phi'\|_{\infty,\supp(\pi_0)}.
\end{align*}
Thus \(T_2\) is continuous, so \(\{\bar\phi_h:h\in\mathcal H\}=T_2(T_1(\mathcal H))\) is compact.

Finally define
\[
T_3:\ell^\infty(\Omega_{\mathrm{sa}};\R^J)\to \ell^\infty(\Omega_{\mathrm{sa}}\times \cU;\R),
\qquad
(T_3\bar\phi)(x,a,i)\coloneqq \langle \lambda_i^\star,\bar\phi(x,a)\rangle.
\]
Then
\[
\|T_3\bar\phi-T_3\bar\phi'\|_{\infty,\supp(\pi_0)}
=
\sup_{(x,a,i)} |\langle \lambda_i^\star,\bar\phi(x,a)-\bar\phi'(x,a)\rangle|
\le
L_\lambda \|\bar\phi-\bar\phi'\|_{\infty,\supp(\pi_0)}.
\]
So \(T_3\) is continuous, and
\[
\mathcal F_{\mathrm{sh}}
=
\{R_h:h\in\mathcal H\}
=
T_3(T_2(T_1(\mathcal H)))
\]
is compact as the continuous image of a compact set.
\end{proof}

\begin{lemma}[Compact head-domain parameter class and continuity of the induced score maps]
\label{lem:compact_joint_gf_head}
Assume \(\Theta\) is compact, \(\Omega_\psi\subset \R^d\) is compact, \(\psi(x,a)\in\Omega_\psi\) for every \((x,a)\) with \(a\in \cA_0(x)\), and the map
\[
(\theta,z)\longmapsto h_\theta(z)
\]
is continuous on \(\Theta\times\Omega_\psi\).
Define
\[
\bar\phi_{h_\theta}(x,a)
=
h_\theta(\psi(x,a))
-
\EE_{A\sim\pi_0(\cdot\mid x)}[h_\theta(\psi(x,A))].
\]
Let
\[
M_{\bar\phi}\coloneqq \sup_{\theta\in\Theta}\|\bar\phi_{h_\theta}\|_{\infty,\supp(\pi_0)},
\]
let \(\Xi\subset \R^{J\times U}\) be compact, let
\[
L_{\mathrm{head}}
\coloneqq
\sup_{\Lambda=[\lambda_1,\dots,\lambda_U]\in\Xi}
\max_{i\in\cU}\|\lambda_i\|_2,
\]
and let
\[
\mathcal M\subseteq \Theta\times \Xi
\]
be closed. For \((\theta,\Lambda)\in\mathcal M\), define
\[
R_{\theta,\Lambda}(x,a,i)\coloneqq \langle \lambda_i,\bar\phi_{h_\theta}(x,a)\rangle.
\]
Then:
\begin{enumerate}[label=(\roman*),leftmargin=2em]
\item \(\mathcal M\) is compact.
\item The map
\[
(\theta,\Lambda)\longmapsto R_{\theta,\Lambda}
\]
from \(\mathcal M\) into the score space equipped with
\(\|\cdot\|_{\infty,\supp(\pi_0)}\) is continuous.
\item The induced score class
\[
\mathcal F_{\mathcal M}
=
\{R_{\theta,\Lambda}:(\theta,\Lambda)\in\mathcal M\}
\]
is compact under \(\|\cdot\|_{\infty,\supp(\pi_0)}\).
\item Every \(R_{\theta,\Lambda}\in\mathcal F_{\mathcal M}\) satisfies
\[
\|R_{\theta,\Lambda}\|_{\infty,\supp(\pi_0)}\le L_{\mathrm{head}}M_{\bar\phi}.
\]
\end{enumerate}
\end{lemma}

\begin{proof}
Because \((\theta,z)\mapsto h_\theta(z)\) is continuous on the compact set
\(\Theta\times\Omega_\psi\), it is uniformly bounded there. Since
\(\psi(x,a)\in\Omega_\psi\) on the reference support,
\[
\sup_{\theta\in\Theta}\|\bar\phi_{h_\theta}\|_{\infty,\supp(\pi_0)}
\le
2\sup_{(\theta,z)\in\Theta\times\Omega_\psi}\|h_\theta(z)\|_2<\infty.
\]
Thus \(M_{\bar\phi}<\infty\). Since \(\Xi\) is compact and \(\Lambda\mapsto \max_{i\in\cU}\|\lambda_i\|_2\) is continuous, \(L_{\mathrm{head}}<\infty\). Because \(\Theta\) is compact and \(\Xi\) is compact,
\[
\Theta\times \Xi
\]
is compact. Since \(\mathcal M\) is closed, it is compact. This proves (i).

By continuity of \((\theta,z)\mapsto h_\theta(z)\) on the compact set
\(\Theta\times \Omega_\psi\), the map \(\theta\mapsto h_\theta\) is continuous under the uniform
norm on \(\Omega_\psi\). Therefore, if \(\theta_n\to\theta\) in \(\Theta\), then
\[
\sup_{z\in\Omega_\psi}\|h_{\theta_n}(z)-h_\theta(z)\|_2\to 0.
\]
Equivalently,
\[
\|h_{\theta_n}\circ\psi-h_\theta\circ\psi\|_{\infty,\supp(\pi_0)}\to 0.
\]
Using the same centering estimate as in Lemma~\ref{lem:compact_nn_to_score},
\[
\|\bar\phi_{h_{\theta_n}}-\bar\phi_{h_\theta}\|_{\infty,\supp(\pi_0)}
\le
2\|h_{\theta_n}\circ\psi-h_\theta\circ\psi\|_{\infty,\supp(\pi_0)}
\to 0.
\]

Now let \((\theta_n,\Lambda_n)\to(\theta,\Lambda)\) in \(\mathcal M\).
Write \(\Lambda_n=[\lambda_{n,1},\dots,\lambda_{n,U}]\) and
\(\Lambda=[\lambda_1,\dots,\lambda_U]\). Since \(\Lambda_n,\Lambda\in\Xi\),
\[
\|\lambda_{n,i}\|_2\le L_{\mathrm{head}},
\qquad
\|\lambda_i\|_2\le L_{\mathrm{head}}.
\]
Hence
\begin{align*}
\|R_{\theta_n,\Lambda_n}-R_{\theta,\Lambda}\|_{\infty,\supp(\pi_0)}
&=
\sup_{(x,a,i)}
\left|
\langle \lambda_{n,i},\bar\phi_{h_{\theta_n}}(x,a)\rangle
-
\langle \lambda_i,\bar\phi_{h_\theta}(x,a)\rangle
\right|\\
&\le
\sup_{(x,a,i)}
\left|
\langle \lambda_{n,i},\bar\phi_{h_{\theta_n}}(x,a)-\bar\phi_{h_\theta}(x,a)\rangle
\right|
+
\sup_{(x,a,i)}
\left|
\langle \lambda_{n,i}-\lambda_i,\bar\phi_{h_\theta}(x,a)\rangle
\right|\\
&\le
L_{\mathrm{head}}\|\bar\phi_{h_{\theta_n}}-\bar\phi_{h_\theta}\|_{\infty,\supp(\pi_0)}
+
M_{\bar\phi}\max_{i\in\cU}\|\lambda_{n,i}-\lambda_i\|_2
\to 0.
\end{align*}
Thus (ii) holds. Since the continuous image of a compact set is compact, (iii) follows.

For (iv), for every \((x,a,i)\),
\[
|R_{\theta,\Lambda}(x,a,i)|
=
|\langle \lambda_i,\bar\phi_{h_\theta}(x,a)\rangle|
\le
\|\lambda_i\|_2\,\|\bar\phi_{h_\theta}(x,a)\|_2
\le
L_{\mathrm{head}}M_{\bar\phi}.
\]
\end{proof}

\subsection{Compact-continuity consequences}
\label{app:compact_continuity_details}

\begin{lemma}[Consequences of compact continuity]
\label{lem:compact_continuity_consequences_P}
Assume Condition~\ref{as:LP1}. Then \(\mathcal M\) is compact, the map
\[
(\theta,\Lambda)\longmapsto R_{\theta,\Lambda}
\]
from \(\mathcal M\) into the score space equipped with
\(\|\cdot\|_{\infty,\supp(\pi_0)}\) is continuous, \(\mathcal P\) is compact, the restricted map
\(p\mapsto R_p\) is continuous on \(\mathcal P\), and the induced score class
\[
\mathcal F_{\mathcal P}=\{R_p:p\in\mathcal P\}
\]
is compact under \(\|\cdot\|_{\infty,\supp(\pi_0)}\). In particular,
\[
B_{\mathcal P}<\infty,
\qquad
M_{\bar\phi,\mathcal P}<\infty,
\qquad
\Delta_{\max}^{\mathcal P}\le 2B_{\mathcal P}<\infty.
\]
Moreover, for every finite sample whose slate actions lie in the reference support, the maps
\(R\mapsto \ell_s(R)\) are continuous on \(\mathcal F_{\mathcal P}\). Hence the online and offline exact-ERM objectives over \(\mathcal F_{\mathcal P}\) admit minimizers on their full-probability support events.
\end{lemma}

\begin{proof}
The compactness of \(\Theta\), \(\Xi\), and closedness of \(\mathcal M\) imply compactness of \(\mathcal M\). Lemma~\ref{lem:compact_joint_gf_head} gives continuity of \((\theta,\Lambda)\mapsto R_{\theta,\Lambda}\) under the supportwise supremum norm. Since \(\mathcal P\) is nonempty and closed in the compact set \(\mathcal M\), it is compact; the continuity of the restricted map follows by restriction, and \(\mathcal F_{\mathcal P}\) is compact as a continuous image of a compact set. The representation envelope is finite by compactness and Lemma~\ref{lem:compact_joint_gf_head}. The score envelope is finite by compactness of \(\mathcal F_{\mathcal P}\), and the displayed bound on \(\Delta_{\max}^{\mathcal P}\) follows from
\[
R_p(X,a_p(X,I),I)-R_p(X,a,I)
\le
|R_p(X,a_p(X,I),I)|+|R_p(X,a,I)|
\le 2B_{\mathcal P}
\]
on the full-measure support event. Continuity of finite-sample MNL losses follows from continuity of the MNL log-loss in the realized score vector and the fact that realized actions lie in \(\cA_0(x)\) almost surely under any sampling rule absolutely continuous with respect to \(\pi_0\). The existence of exact ERM minimizers is then Weierstrass' theorem on the compact class \(\mathcal F_{\mathcal P}\), with an arbitrary fixed default on the complementary null event.
\end{proof}

\begin{definition}[\(\beta\)-admissible learner]
\label{def:beta_admissible_learner}
A learning rule \(\mathsf A\) is \(\beta\)-admissible if, at each round \(t\), conditional on the history \(H_{t-1}\) and the current pair \((X_t,I_t)\), the slate coordinates \(A_{t,1},\dots,A_{t,K}\) are sampled i.i.d.\ from a distribution
\[
\pi_t^{\mathsf A}(\cdot\mid X_t,I_t,H_{t-1})
\]
satisfying
\begin{equation}
\beta^{-1}
\le
\frac{d\pi_t^{\mathsf A}(\cdot\mid x,i,h)}{d\pi_0(\cdot\mid x)}(a)
\le
\beta,
\qquad
\pi_0(\cdot\mid x)\text{-a.s.}
\label{eq:beta_admissible_learners}
\end{equation}
for every realized history \(h\), context \(x\), user \(i\), and round \(t\). Let \(\mathfrak A_\beta\) denote the class of all \(\beta\)-admissible learners.
\end{definition}

\section{Proofs for Sections~\ref{sec:decision_framework} and~\ref{sec:online_alignment}}
\label{app:deferred_proofs}

\begin{lemma}[Expected one-step regret versus disagreement mass]
\label{lem:regret_vs_disagreement}
Assume \ref{as:LP1}--\ref{as:LP2}. For every \(p,q\in\mathcal P\),
\[
\Delta_{\min}^{\mathcal P}\,
(d_0\times\rho)\{a_q\neq a_p\}
\le
\mathcal G_p(R_q)
\le
\Delta_{\max}^{\mathcal P}\,
(d_0\times\rho)\{a_q\neq a_p\}.
\]
\end{lemma}

\begin{proof}[Proof of Lemma~\ref{lem:regret_vs_disagreement}]
Fix \(p,q\in\mathcal P\). Define
\[
E_{p,q}
\coloneqq
\{(x,i)\in\cX\times\cU: a_q(x,i)\neq a_p(x,i)\}
\]
and the pointwise regret
\[
g_{p,q}(x,i)
\coloneqq
R_p(x,a_p(x,i),i)-R_p(x,a_q(x,i),i).
\]
By Condition~\ref{as:LP2}, there is a measurable set \(G_p\subseteq\cX\times\cU\) with
\((d_0\times\rho)(G_p)=1\) such that the stated supportwise gap for truth \(p\) holds at every
\((x,i)\in G_p\). If \((x,i)\notin E_{p,q}\), then \(a_q(x,i)=a_p(x,i)\) and hence
\(g_{p,q}(x,i)=0\). If \((x,i)\in E_{p,q}\cap G_p\), then the gap condition gives
\[
g_{p,q}(x,i)\ge \Delta_{\min}^{\mathcal P}.
\]
By definition of \(\Delta_{\max}^{\mathcal P}\), after possibly intersecting \(G_p\) with another full \(d_0\times\rho\)-measure set, one also has
\[
g_{p,q}(x,i)\le \Delta_{\max}^{\mathcal P}
\]
on \(G_p\). Therefore, on \(G_p\),
\[
\Delta_{\min}^{\mathcal P}\,\mathbf 1_{E_{p,q}}(x,i)
\le
g_{p,q}(x,i)
\le
\Delta_{\max}^{\mathcal P}\,\mathbf 1_{E_{p,q}}(x,i).
\]
Since \(G_p\) has full \(d_0\times\rho\)-measure, the same display holds \(d_0\times\rho\)-a.s.
Taking expectations with respect to \((X,I)\sim d_0\times\rho\) yields the claim.
\end{proof}

\begin{lemma}[Stability of the temperature-zero selector on \(\mathcal P\)]
\label{lem:stability_selector_P}
Assume \ref{as:LP2}. If \(p,q\in\mathcal P\) satisfy
\[
\|R_q-R_p\|_{\infty,\supp(\pi_0)}<\frac{\Delta_{\min}^{\mathcal P}}{2},
\]
then
\[
a_q(X,I)=a_p(X,I)
\qquad d_0\times\rho\text{-a.s.}
\]
\end{lemma}

\begin{proof}
Let
\[
E_{p,q}\coloneqq\{(x,i):a_q(x,i)\neq a_p(x,i)\}.
\]
Assume, toward a contradiction, that \((d_0\times\rho)(E_{p,q})>0\). By
Condition~\ref{as:LP2}, there is a full-measure set \(G_p\) on which the supportwise gap for truth
\(p\) holds. Hence \((d_0\times\rho)(E_{p,q}\cap G_p)>0\), so choose a point in this intersection.
At that point Lemma~\ref{lem:misrec_implies_sup_error} applies with truth score \(S^\star=R_p\)
and candidate score \(S=R_q\), yielding
\[
\|R_q-R_p\|_{\infty,\supp(\pi_0)}
\ge
\frac{\Delta_{\min}^{\mathcal P}}{2},
\]
which contradicts the strict inequality in the hypothesis. Hence \(a_q=a_p\) almost surely.
\end{proof}

\begin{lemma}[Continuity of the truth-centered regret and risk maps]
\label{lem:continuity_truth_centered_maps}
Assume \ref{as:LP1}--\ref{as:LP2}. Then the maps
\[
(p,q)\longmapsto \mathcal G_p(R_q)
\qquad\text{and}\qquad
(p,q)\longmapsto \mathcal L_p(R_q)-\mathcal L_p(R_p)
\]
are continuous on \(\mathcal P^2\).
\end{lemma}

\begin{proof}
This is the personalized analogue of the compact-continuity argument in \citet{kang2026demystifying}. 
By Condition~\ref{as:LP1}, the map
\[
p\longmapsto R_p
\]
is continuous from \(\mathcal P\) into the score space endowed with
\(\|\cdot\|_{\infty,\supp(\pi_0)}\).

We first prove continuity of \((p,q)\mapsto \mathcal G_p(R_q)\). Let \((p_n,q_n)\to(p,q)\) in
\(\mathcal P^2\). Then
\[
\|R_{p_n}-R_p\|_{\infty,\supp(\pi_0)}\to0,
\qquad
\|R_{q_n}-R_q\|_{\infty,\supp(\pi_0)}\to0.
\]
By Lemma~\ref{lem:stability_selector_P}, for all sufficiently large \(n\),
\[
a_{p_n}=a_p
\quad\text{and}\quad
a_{q_n}=a_q
\qquad
d_0\times\rho\text{-a.s.}
\]
Hence for all sufficiently large \(n\),
\[
\mathcal G_{p_n}(R_{q_n})
=
\EE\Big[
R_{p_n}(X,a_p(X,I),I)-R_{p_n}(X,a_q(X,I),I)
\Big].
\]
The integrand is uniformly bounded by \(2B_{\mathcal P}\) and converges pointwise almost surely to
\[
R_p(X,a_p(X,I),I)-R_p(X,a_q(X,I),I).
\]
Dominated convergence therefore yields
\[
\mathcal G_{p_n}(R_{q_n})\to \mathcal G_p(R_q).
\]

We next prove continuity of \((p,q)\mapsto \mathcal L_p(R_q)-\mathcal L_p(R_p)\). Fix
\((x,i,\mathbf a)\). Because both the softmax map and the MNL loss are continuous in the score
vector, the integrand
\[
(p,q,x,i,\mathbf a)
\longmapsto
\EE_{Y\sim P_{R_p}(\cdot\mid x,i,\mathbf a)}
\big[\ell(\mathbf v_{R_q},Y)-\ell(\mathbf v_{R_p},Y)\big]
\]
is continuous in \((p,q)\). By Lemma~\ref{lem:mnl_loss_basic}, this integrand is uniformly bounded
by \(2(\log K+2B_{\mathcal P})\). Another application of dominated convergence yields the claimed
continuity.
\end{proof}

\begin{proof}[Proof of Lemma~\ref{lem:automatic_positive_regret_isolation_P}]
This is the user-indexed selector-isolation argument adapted from \citet{kang2026demystifying}. 
Define an equivalence relation on \(\mathcal P\) by
\[
p\sim q
\qquad\Longleftrightarrow\qquad
a_p(X,I)=a_q(X,I)
\quad
d_0\times\rho\text{-a.s.}
\]
Let \([p]\) denote the corresponding selector class.

We first show that only finitely many selector classes can occur. Suppose, toward a contradiction,
that there are infinitely many distinct selector classes. Then we may choose
\[
p_1,p_2,p_3,\dots \in \mathcal P
\]
such that \([p_i]\neq [p_j]\) whenever \(i\neq j\). For each \(i\neq j\), one has
\[
(d_0\times\rho)\{a_{p_i}\neq a_{p_j}\}>0.
\]
Hence, by the contrapositive of Lemma~\ref{lem:stability_selector_P},
\[
\|R_{p_i}-R_{p_j}\|_{\infty,\supp(\pi_0)}
\ge
\frac{\Delta_{\min}^{\mathcal P}}{2}
\qquad
\text{for all }i\neq j.
\]
Thus \(\{R_{p_i}:i\ge 1\}\) is an infinite \(\Delta_{\min}^{\mathcal P}/2\)-separated subset of
\(\mathcal F_{\mathcal P}\). But \(\mathcal F_{\mathcal P}\) is compact under
\(\|\cdot\|_{\infty,\supp(\pi_0)}\), since \(\mathcal P\) is compact by Condition~\ref{as:LP1}
and \(p\mapsto R_p\) is continuous. Hence \(\mathcal F_{\mathcal P}\) is totally bounded, so no
such infinite separated subset can exist. Therefore only finitely many selector classes occur.

Let
\[
a^{(1)},\dots,a^{(m)}
\]
be representatives of these finitely many selector classes.

If \(m=1\), then every pair \(p,q\in\mathcal P\) satisfies
\[
a_p(X,I)=a_q(X,I)
\qquad
d_0\times\rho\text{-a.s.}
\]
and therefore
\[
\mathcal G_p(R_q)=0
\qquad
\forall p,q\in\mathcal P.
\]
In this case the conclusion holds for any positive choice of
\[
\varepsilon_{\mathrm{iso}}^{\mathcal P}>0;
\]
for concreteness, take
\[
\varepsilon_{\mathrm{iso}}^{\mathcal P}\coloneqq \Delta_{\min}^{\mathcal P}.
\]

Assume now that \(m\ge 2\). Define
\[
\delta_{\mathcal P}
\coloneqq
\min_{1\le i<j\le m} (d_0\times\rho)\{a^{(i)}\neq a^{(j)}\}.
\]
Because the classes are distinct modulo \(d_0\times\rho\)-a.s.\ equality, every term in this
finite minimum is strictly positive, hence
\[
\delta_{\mathcal P}>0.
\]
Set
\[
\varepsilon_{\mathrm{iso}}^{\mathcal P}
\coloneqq
\Delta_{\min}^{\mathcal P}\,\delta_{\mathcal P}.
\]

Fix any \(p,q\in\mathcal P\). If \(a_q=a_p\) \(d_0\times\rho\)-a.s., then by definition
\[
\mathcal G_p(R_q)=0.
\]
Otherwise \(a_q\) and \(a_p\) belong to two distinct selector classes, so
\[
(d_0\times\rho)\{a_q\neq a_p\}\ge \delta_{\mathcal P}.
\]
Applying Lemma~\ref{lem:regret_vs_disagreement} gives
\[
\mathcal G_p(R_q)
\ge
\Delta_{\min}^{\mathcal P}\,(d_0\times\rho)\{a_q\neq a_p\}
\ge
\Delta_{\min}^{\mathcal P}\,\delta_{\mathcal P}
=
\varepsilon_{\mathrm{iso}}^{\mathcal P}.
\]
Therefore
\[
\mathcal G_p(R_q)\in\{0\}\cup[\varepsilon_{\mathrm{iso}}^{\mathcal P},\infty),
\]
as claimed.
\end{proof}

\begin{lemma}[Supportwise upgrade of almost-sure equality]
\label{lem:support_upgrade_personalized}
Let \((\cA,d)\) be a separable metric space, let \(\mu\) be a Borel probability measure on \(\cA\),
and let \(S\coloneqq \supp(\mu)\) be its topological support. If \(f:S\to\R\) is continuous and
\(f=0\) \(\mu\)-a.s., then \(f\equiv 0\) on \(S\). More generally, if \(f,g:S\to\R\) are continuous
and \(f=g\) \(\mu\)-a.s., then \(f=g\) on \(S\).
\end{lemma}

\begin{proof}
The argument is the topological-support upgrade used in \citet{kang2026demystifying}, written here for the user-indexed application. Since \(\cA\) is separable metric, it is second countable. Let \(O\coloneqq \cA\setminus S\). For every \(a\in O\), by definition of topological support there is an open neighborhood \(V_a\) of \(a\) with \(\mu(V_a)=0\). The open cover \(\{V_a:a\in O\}\) has a countable subcover, so \(\mu(O)=0\).

Assume that \(f=0\) \(\mu\)-a.s. on \(S\), but that \(f(a_0)\neq0\) for some \(a_0\in S\). By continuity of \(f\) on the subspace \(S\), there are \(\delta>0\) and an open set \(U\subseteq\cA\) containing \(a_0\) such that \(|f(a)|\ge\delta\) for every \(a\in U\cap S\). Because \(a_0\in S=\supp(\mu)\), every open neighborhood of \(a_0\) has positive \(\mu\)-mass, so \(\mu(U)>0\). Since \(\mu(\cA\setminus S)=0\), we also have \(\mu(U\cap S)>0\), contradicting \(f=0\) \(\mu\)-a.s. Hence \(f\equiv0\) on \(S\). Applying this to \(f-g\) proves the equality statement.
\end{proof}

\begin{lemma}[Zero truth-centered loss identifies the personalized score on the reference support]
\label{lem:zero_loss_implies_zero_regret_P}
Assume \ref{as:LP1}. Then for every truth \(p\in\mathcal P\) and every \(q\in\mathcal P\),
\[
\mathcal L_p(R_q)=\mathcal L_p(R_p)
\quad\Longrightarrow\quad
R_q(x,a,i)=R_p(x,a,i)
\quad
\forall a\in \cA_0(x),
\quad d_0\times\rho\text{-a.e. }(x,i).
\]
In particular,
\[
\mathcal G_p(R_q)=0.
\]
\end{lemma}

\begin{proof}
This is the user-indexed supportwise-identification argument adapted from \citet{kang2026demystifying}. 
Fix a truth \(p\in\mathcal P\) and a candidate \(q\in\mathcal P\). Define
\[
K_{p,q}(x,i,\mathbf a)
\coloneqq
\KL\!\Big(P_{R_p}(\cdot\mid x,i,\mathbf a)\,\Big\|\,P_{R_q}(\cdot\mid x,i,\mathbf a)\Big).
\]
By Lemma~\ref{lem:excess_is_kl_support},
\[
\mathcal L_p(R_q)-\mathcal L_p(R_p)
=
\EE\big[K_{p,q}(X,I,\mathbf A)\big],
\]
where
\[
(X,I)\sim d_0\times\rho,
\qquad
\mathbf A\sim \pi_0(\cdot\mid X)^{\otimes K}.
\]
Assume \(\mathcal L_p(R_q)=\mathcal L_p(R_p)\). Since \(K_{p,q}\ge0\), Tonelli's theorem gives a
full-measure set \(G\subseteq\cX\times\cU\) such that for every \((x,i)\in G\),
\[
K_{p,q}(x,i,\mathbf a)=0
\qquad
\pi_0(\cdot\mid x)^{\otimes K}\text{-a.e. }\mathbf a.
\]

Fix \((x,i)\in G\) such that both action sections
\[
a\longmapsto R_p(x,a,i),
\qquad
a\longmapsto R_q(x,a,i)
\]
are continuous on \(\cA_0(x)\). By Condition~\ref{as:LP1}, this holds for
\(d_0\times\rho\)-a.e.\ \((x,i)\). Define
\[
h_{x,i}(a)\coloneqq R_q(x,a,i)-R_p(x,a,i),
\qquad a\in \cA_0(x).
\]
For every slate \(\mathbf a=(a_1,\dots,a_K)\) in the full-measure set above, the two MNL laws are
equal. Comparing the odds of coordinates \(1\) and \(2\) gives
\[
\exp\{h_{x,i}(a_1)-h_{x,i}(a_2)\}=1,
\]
and hence
\[
h_{x,i}(a_1)-h_{x,i}(a_2)=0.
\]
The event \(\{h_{x,i}(a_1)\neq h_{x,i}(a_2)\}\) depends only on the first two slate coordinates. Integrating out coordinates \(3,\dots,K\) therefore gives
\[
\pi_0(\cdot\mid x)^{\otimes2}\big(h_{x,i}(a)\neq h_{x,i}(b)\big)=0.
\]
Thus, if \(A,B\overset{\mathrm{i.i.d.}}{\sim}\pi_0(\cdot\mid x)\), then
\[
h_{x,i}(A)=h_{x,i}(B)
\qquad\text{a.s.}
\]
By the compactness consequence of Condition~\ref{as:LP1}, both \(R_p\) and \(R_q\) are supportwise bounded, so \(h_{x,i}(A)\in L^2\). Hence
\[
0=\EE\big[(h_{x,i}(A)-h_{x,i}(B))^2\big]
=2\,\Var(h_{x,i}(A)).
\]
Thus there exists \(c_{x,i}\in\R\) such that
\[
h_{x,i}(a)=c_{x,i}
\qquad
\pi_0(\cdot\mid x)\text{-a.s. }a.
\]
Both scores are \(\pi_0\)-centered, so
\[
0=
\EE_{A\sim\pi_0(\cdot\mid x)}[h_{x,i}(A)]
=c_{x,i}.
\]
Therefore \(h_{x,i}=0\) \(\pi_0(\cdot\mid x)\)-a.s. By
Lemma~\ref{lem:support_upgrade_personalized},
\[
h_{x,i}(a)=0
\qquad
\forall a\in \cA_0(x).
\]
This proves the claimed supportwise equality for \(d_0\times\rho\)-a.e.\ \((x,i)\). Consequently
\(a_q(X,I)=a_p(X,I)\) under the common measurable tie-breaking convention, and hence
\(\mathcal G_p(R_q)=0\).
\end{proof}

\begin{lemma}[Positive far-pair truth-centered risk gap]
\label{lem:positive_far_pair_risk_gap_P}
Assume \ref{as:LP1}--\ref{as:LP2}. For every \(r_0>0\), define
\[
\gamma_{\mathrm{far}}^{\mathcal P}(r_0;\varepsilon_0)
\coloneqq
\inf_{\substack{
p,q\in\mathcal P\\
\max_i\|\lambda_{i,q}-\lambda_{i,p}\|_2\ge r_0\\
\mathcal G_p(R_q)\ge \varepsilon_0
}}
\Big(
\mathcal L_p(R_q)-\mathcal L_p(R_p)
\Big),
\qquad
\inf\emptyset\coloneqq +\infty.
\]
Then
\[
\gamma_{\mathrm{far}}^{\mathcal P}(r_0;\varepsilon_0)>0.
\]
\end{lemma}

\begin{proof}[Proof of Lemma~\ref{lem:positive_far_pair_risk_gap_P}]
By \ref{as:LP1}, \(\mathcal P\) is compact, so \(\mathcal P^2\) is compact. The map
\[
(p,q)\longmapsto \max_i\|\lambda_{i,q}-\lambda_{i,p}\|_2
\]
is continuous, and by Lemma~\ref{lem:continuity_truth_centered_maps}, the map
\[
(p,q)\longmapsto \mathcal G_p(R_q)
\]
is continuous. Hence the constrained set
\[
K_{\mathrm{far}}(r_0;\varepsilon_0)
\coloneqq
\Big\{(p,q)\in\mathcal P^2:
\max_i\|\lambda_{i,q}-\lambda_{i,p}\|_2\ge r_0,
\mathcal G_p(R_q)\ge \varepsilon_0
\Big\}
\]
is compact.

If \(K_{\mathrm{far}}(r_0;\varepsilon_0)=\emptyset\), the claim is trivial by the
\(+\infty\) convention. Otherwise the continuous map
\[
(p,q)\longmapsto \mathcal L_p(R_q)-\mathcal L_p(R_p)
\]
attains its infimum on \(K_{\mathrm{far}}(r_0;\varepsilon_0)\) at some pair
\((\bar p,\bar q)\in K_{\mathrm{far}}(r_0;\varepsilon_0)\).

If that attained infimum were zero, then
\[
\mathcal L_{\bar p}(R_{\bar q})=\mathcal L_{\bar p}(R_{\bar p}),
\]
and Lemma~\ref{lem:zero_loss_implies_zero_regret_P} would imply
\[
\mathcal G_{\bar p}(R_{\bar q})=0,
\]
contradicting \((\bar p,\bar q)\in K_{\mathrm{far}}(r_0;\varepsilon_0)\), where
\(\mathcal G_{\bar p}(R_{\bar q})\ge \varepsilon_0\). Hence the attained infimum is strictly
positive.
\end{proof}

For the moduli comparison arguments below, also define the finite representation envelope
\[
M_{\bar\phi,\mathcal P}
\coloneqq
\sup_{p\in\mathcal P}\|\bar\phi_{h_{\theta_p}}\|_{\infty,\supp(\pi_0)}
<\infty
\]
and the MNL curvature constant
\[
c_{\mathrm{mnl},\mathcal P}
\coloneqq
\frac{e^{-2B_{\mathcal P}}}{2}\cdot \frac{K-1}{K}.
\]

\begin{proposition}[Representation-diversity versus score separation]
\label{prop:dreg_creg_comparison}
For \(p\in\mathcal P\), define
\[
\mathfrak c_{\mathrm{reg},p}(r;\varepsilon_0)
\coloneqq
\inf_{q\in\mathcal Q_{\mathrm{reg}}(p,r;\varepsilon_0)}
\EE\!\left[
\Delta_{q\mid p}(X,A,I)^2
\right],
\qquad
\inf\emptyset=+\infty,
\]
where \((X,I)\sim d_0\times\rho\), \(A\sim\pi_0(\cdot\mid X)\), and
\[
\Delta_{q\mid p}(x,a,i)\coloneqq R_q(x,a,i)-R_p(x,a,i).
\]
Let
\[
\underline{\mathfrak c}_{\mathrm{reg}}(r;\varepsilon_0)
\coloneqq
\inf_{p\in\mathcal P}\mathfrak c_{\mathrm{reg},p}(r;\varepsilon_0).
\]
Assume \ref{as:LP1}. Then for every \(r>0\),
\[
\frac12\,\underline{\mathfrak d}_{\mathrm{reg}}(r;\varepsilon_0)
-
M_{\bar\phi,\mathcal P}^2r^2
\le
\underline{\mathfrak c}_{\mathrm{reg}}(r;\varepsilon_0)
\le
2\,\underline{\mathfrak d}_{\mathrm{reg}}(r;\varepsilon_0)
+
2M_{\bar\phi,\mathcal P}^2r^2.
\]
\end{proposition}

\begin{proposition}[Score separation versus truth-centered risk separation]
\label{prop:creg_Gammareg_comparison}
For \(p\in\mathcal P\), define
\[
\Gamma_{\mathrm{reg},p}(r;\varepsilon_0)
\coloneqq
\inf_{q\in\mathcal Q_{\mathrm{reg}}(p,r;\varepsilon_0)}
\Big(
\mathcal L_p(R_q)-\mathcal L_p(R_p)
\Big),
\qquad
\inf\emptyset=+\infty,
\]
and
\[
\underline{\Gamma}_{\mathrm{reg}}(r;\varepsilon_0)
\coloneqq
\inf_{p\in\mathcal P}\Gamma_{\mathrm{reg},p}(r;\varepsilon_0).
\]
Assume \ref{as:LP1}. Then for every \(r>0\),
\[
c_{\mathrm{mnl},\mathcal P}\,
\underline{\mathfrak c}_{\mathrm{reg}}(r;\varepsilon_0)
\le
\underline{\Gamma}_{\mathrm{reg}}(r;\varepsilon_0)
\le
\frac K2\,
\underline{\mathfrak c}_{\mathrm{reg}}(r;\varepsilon_0).
\]
\end{proposition}

\begin{lemma}[Equivalence of the three liminf conditions]
\label{lem:equivalence_three_liminf_conditions}
Assume \ref{as:LP1}. Then
\[
\liminf_{r\downarrow0}
\underline{\mathfrak d}_{\mathrm{reg}}(r;\varepsilon_0)>0
\iff
\liminf_{r\downarrow0}
\underline{\mathfrak c}_{\mathrm{reg}}(r;\varepsilon_0)>0
\iff
\liminf_{r\downarrow0}
\underline{\Gamma}_{\mathrm{reg}}(r;\varepsilon_0)>0.
\]
\end{lemma}

\begin{remark}[Proof organization]
The equivalence in Theorem~\ref{thm:fixed_scale_characterization} separates into a head-close and a head-far argument. The head-close part is the shell comparison in Lemma~\ref{lem:equivalence_three_liminf_conditions}; the head-far part is handled by Lemma~\ref{lem:positive_far_pair_risk_gap_P}. The key identifiability input is Lemma~\ref{lem:zero_loss_implies_zero_regret_P}, which adapts the supportwise-identification technique of \citet{kang2026demystifying} to user-indexed centered scores by using action-continuity on the topological reference support and Lemma~\ref{lem:support_upgrade_personalized}.
\end{remark}

\begin{proof}[Proof of Proposition~\ref{prop:dreg_creg_comparison}]
Fix \(p\in\mathcal P\). If \(\mathcal Q_{\mathrm{reg}}(p,r;\varepsilon_0)=\emptyset\), then the corresponding shell infima are \(+\infty\) and the inequalities are interpreted in the extended-real sense. Otherwise, fix
\[
q\in \mathcal Q_{\mathrm{reg}}(p,r;\varepsilon_0).
\]
Abbreviate
\[
g\coloneqq g_{q\mid p},
\qquad
\Delta\coloneqq \Delta_{q\mid p}.
\]
For every \((x,a,i)\), write
\[
\Delta(x,a,i)
=
\underbrace{\langle \lambda_{i,p},g(x,a)\rangle}_{u(x,a,i)}
+
\underbrace{\langle \lambda_{i,q}-\lambda_{i,p},\bar\phi_{h_{\theta_q}}(x,a)\rangle}_{v(x,a,i)}.
\]
Then
\[
\EE[u(X,A,I)^2]
=
\EE\!\left[g(X,A)^\top G_{\lambda,p}g(X,A)\right].
\]
Moreover,
\[
|v(x,a,i)|
\le
\|\lambda_{i,q}-\lambda_{i,p}\|_2\,\|\bar\phi_{h_{\theta_q}}(x,a)\|_2
\le
r\,M_{\bar\phi,\mathcal P},
\]
so
\[
\EE[v(X,A,I)^2]\le M_{\bar\phi,\mathcal P}^2r^2.
\]
Using
\[
(u+v)^2\ge \frac12 u^2-v^2,
\qquad
(u+v)^2\le 2u^2+2v^2,
\]
we obtain
\[
\EE[\Delta(X,A,I)^2]
\ge
\frac12\EE[u^2]-\EE[v^2]
\ge
\frac12
\EE\!\left[g(X,A)^\top G_{\lambda,p}g(X,A)\right]
-
M_{\bar\phi,\mathcal P}^2r^2,
\]
and
\[
\EE[\Delta(X,A,I)^2]
\le
2\EE[u^2]+2\EE[v^2]
\le
2
\EE\!\left[g(X,A)^\top G_{\lambda,p}g(X,A)\right]
+
2M_{\bar\phi,\mathcal P}^2r^2.
\]
Taking the infimum over
\[
q\in \mathcal Q_{\mathrm{reg}}(p,r;\varepsilon_0),
\]
and then over \(p\in\mathcal P\), proves the result.
\end{proof}

\begin{proof}[Proof of Proposition~\ref{prop:creg_Gammareg_comparison}]
Fix \(p\in\mathcal P\). If \(\mathcal Q_{\mathrm{reg}}(p,r;\varepsilon_0)=\emptyset\), then the corresponding shell infima are \(+\infty\) and the inequalities are interpreted in the extended-real sense. Otherwise, fix \(q\in\mathcal Q_{\mathrm{reg}}(p,r;\varepsilon_0)\) and write
\[
\Delta(x,a,i)=R_q(x,a,i)-R_p(x,a,i).
\]
Because both \(R_q\) and \(R_p\) are \(\pi_0\)-centered,
\[
\EE_{A\sim\pi_0(\cdot\mid x)}[\Delta(x,A,i)]=0
\qquad \forall (x,i),
\]
hence
\[
\Var_{A\sim\pi_0(\cdot\mid x)}(\Delta(x,A,i))
=
\EE_{A\sim\pi_0(\cdot\mid x)}[\Delta(x,A,i)^2].
\]

Applying Lemma~\ref{lem:expected_kl_ge_var_support} with truth \(S^\star=R_p\), candidate
\(S=R_q\), and sampling distribution \(q(\cdot\mid x)=\pi_0(\cdot\mid x)\), we obtain
\begin{align*}
\mathcal L_p(R_q)-\mathcal L_p(R_p)
&=
\EE\Big[
\KL\big(P_{R_p}(\cdot\mid X,I,\mathbf A)\,\|\,P_{R_q}(\cdot\mid X,I,\mathbf A)\big)
\Big]\\
&\ge
c_{\mathrm{mnl},\mathcal P}
\EE_{(X,I)\sim d_0\times\rho}
\Big[
\Var_{A\sim\pi_0(\cdot\mid X)}\big(\Delta(X,A,I)\big)
\Big]\\
&=
 c_{\mathrm{mnl},\mathcal P}
\EE\big[\Delta(X,A,I)^2\big].
\end{align*}

For the upper bound, condition on a realized \((x,i,\mathbf a)\) and write
\[
v_k=R_p(x,a_k,i),
\qquad
v'_k=R_q(x,a_k,i),
\qquad
\delta_k=v'_k-v_k.
\]
By Lemma~\ref{lem:softmax_kl_upper_bound},
\[
\KL\big(P_{R_p}(\cdot\mid x,i,\mathbf a)\,\|\,P_{R_q}(\cdot\mid x,i,\mathbf a)\big)
\le
\frac12\sum_{k=1}^K \delta_k^2.
\]
Taking expectation over \((X,I)\sim d_0\times\rho\) and
\(\mathbf A\sim\pi_0(\cdot\mid X)^{\otimes K}\), and using the i.i.d. structure of the slate,
we obtain
\[
\mathcal L_p(R_q)-\mathcal L_p(R_p)
\le
\frac12\sum_{k=1}^K
\EE\big[\Delta(X,A_k,I)^2\big]
=
\frac K2\,\EE\big[\Delta(X,A,I)^2\big].
\]
Taking the infimum over the shell and then over truths proves the proposition.
\end{proof}

\begin{proof}[Proof of Lemma~\ref{lem:equivalence_three_liminf_conditions}]
From Proposition~\ref{prop:dreg_creg_comparison} we have, for every \(r>0\),
\[
\underline{\mathfrak c}_{\mathrm{reg}}(r;\varepsilon_0)
\ge
\frac12\,\underline{\mathfrak d}_{\mathrm{reg}}(r;\varepsilon_0)-M_{\bar\phi,\mathcal P}^2r^2
\]
and also
\[
\underline{\mathfrak d}_{\mathrm{reg}}(r;\varepsilon_0)
\ge
\frac12\,\underline{\mathfrak c}_{\mathrm{reg}}(r;\varepsilon_0)-M_{\bar\phi,\mathcal P}^2r^2,
\]
where the second display is obtained by rearranging the upper bound in
Proposition~\ref{prop:dreg_creg_comparison}. Sending \(r\downarrow0\) shows that
\[
\liminf_{r\downarrow0}\underline{\mathfrak d}_{\mathrm{reg}}(r;\varepsilon_0)>0
\iff
\liminf_{r\downarrow0}\underline{\mathfrak c}_{\mathrm{reg}}(r;\varepsilon_0)>0.
\]
Likewise, Proposition~\ref{prop:creg_Gammareg_comparison} gives
\[
\underline{\Gamma}_{\mathrm{reg}}(r;\varepsilon_0)
\ge
c_{\mathrm{mnl},\mathcal P}\,\underline{\mathfrak c}_{\mathrm{reg}}(r;\varepsilon_0)
\]
and
\[
\underline{\mathfrak c}_{\mathrm{reg}}(r;\varepsilon_0)
\ge
\frac{2}{K}\,\underline{\Gamma}_{\mathrm{reg}}(r;\varepsilon_0).
\]
Sending \(r\downarrow0\) yields the second equivalence.
\end{proof}

\begin{proof}[Proof of Lemma~\ref{lem:pd_user_heads_suffice_for_diversity}]
Write \(p=(\theta_p,\Lambda_p)\), where
\(\Lambda_p=[\lambda_{1,p},\dots,\lambda_{U,p}]\). Since \(\cU\) is finite,
\[
G_{\lambda,p}
=
\sum_{i\in\cU}\rho(i)\lambda_{i,p}\lambda_{i,p}^{\top}.
\]
By Condition~\ref{as:LP1}, \(\mathcal P\) is compact, and the map
\[
p\longmapsto G_{\lambda,p}
\]
is continuous. Since \(G_{\lambda,p}\succ0\) for every \(p\in\mathcal P\), compactness gives the
uniform eigenvalue lower bound
\[
\alpha_\lambda
\coloneqq
\inf_{p\in\mathcal P}\lambda_{\min}(G_{\lambda,p})
>0.
\]

Let
\[
\varepsilon_0\coloneqq \varepsilon_{\mathrm{iso}}^{\mathcal P}.
\]
Assume, toward a contradiction, that decision-relevant user diversity fails.
By the definition of the liminf and the shell infimum, there exist \(r_n\downarrow0\),
\(p_n\in\mathcal P\), and
\[
q_n\in\mathcal Q_{\mathrm{reg}}(p_n,r_n;\varepsilon_0)
\]
such that
\[
d_n
\coloneqq
\EE\!\left[
g_{q_n\mid p_n}(X,A)^\top
G_{\lambda,p_n}
g_{q_n\mid p_n}(X,A)
\right]
\longrightarrow0.
\]
Here \((X,I)\sim d_0\times\rho\) and \(A\sim\pi_0(\cdot\mid X)\). The lower eigenvalue bound
implies
\[
\EE\!\left[\|g_{q_n\mid p_n}(X,A)\|_2^2\right]
\le
\frac{d_n}{\alpha_\lambda}
\longrightarrow0.
\]

For
\[
\Delta_n(x,a,i)\coloneqq R_{q_n}(x,a,i)-R_{p_n}(x,a,i),
\]
write
\[
\Delta_n(x,a,i)
=
\langle\lambda_{i,p_n},g_{q_n\mid p_n}(x,a)\rangle
+
\langle\lambda_{i,q_n}-\lambda_{i,p_n},
\bar\phi_{h_{\theta_{q_n}}}(x,a)\rangle.
\]
The compactness consequences of Condition~\ref{as:LP1} give finite constants
\[
L_{\lambda,\mathcal P}
\coloneqq
\sup_{p\in\mathcal P}\max_{i\in\cU}\|\lambda_{i,p}\|_2
<\infty,
\qquad
M_{\bar\phi,\mathcal P}<\infty.
\]
Since \(q_n\in\mathcal Q_{\mathrm{reg}}(p_n,r_n;\varepsilon_0)\),
\[
\max_i\|\lambda_{i,q_n}-\lambda_{i,p_n}\|_2\le r_n.
\]
Therefore
\[
\EE[\Delta_n(X,A,I)^2]
\le
2L_{\lambda,\mathcal P}^2\,
\EE\!\left[\|g_{q_n\mid p_n}(X,A)\|_2^2\right]
+
2M_{\bar\phi,\mathcal P}^2r_n^2
\longrightarrow0.
\]
Using the same MNL upper-bound argument as in Proposition~\ref{prop:creg_Gammareg_comparison},
\[
\mathcal L_{p_n}(R_{q_n})-\mathcal L_{p_n}(R_{p_n})
\le
\frac K2\,\EE[\Delta_n(X,A,I)^2]
\longrightarrow0.
\]

By compactness of \(\mathcal P\), pass to a subsequence such that
\[
p_n\to p_\infty,
\qquad
q_n\to q_\infty.
\]
Lemma~\ref{lem:continuity_truth_centered_maps} gives
\[
\mathcal L_{p_\infty}(R_{q_\infty})
-
\mathcal L_{p_\infty}(R_{p_\infty})
=0.
\]
By Lemma~\ref{lem:zero_loss_implies_zero_regret_P},
\[
\mathcal G_{p_\infty}(R_{q_\infty})=0.
\]
On the other hand, since \(q_n\in\mathcal Q_{\mathrm{reg}}(p_n,r_n;\varepsilon_0)\),
\[
\mathcal G_{p_n}(R_{q_n})\ge\varepsilon_0
\qquad
\text{for every }n.
\]
Applying Lemma~\ref{lem:continuity_truth_centered_maps} again yields
\[
\mathcal G_{p_\infty}(R_{q_\infty})\ge\varepsilon_0,
\]
which contradicts \(\mathcal G_{p_\infty}(R_{q_\infty})=0\). Hence the failure case in
Definition~\ref{def:decision_relevant_user_diversity} is impossible, and \(\mathcal P\) has
decision-relevant user diversity.
\end{proof}

\begin{theorem}[Exact-ERM control of \(\varepsilon_0\)-substantial rounds from a loss gap]
\label{thm:eps_substantial_rounds_from_loss_gap}
Let \(\mathcal F\) be a score class that is compact under
\(\|\cdot\|_{\infty,\supp(\pi_0)}\), and assume
\[
\sup_{R\in\mathcal F}\|R\|_{\infty,\supp(\pi_0)}\le B.
\]
Fix a truth score \(S^\star\in\mathcal F\), let \(a^\star=a_{S^\star}\), and define
\[
\mathcal G^\star(R)
\coloneqq
\EE_{(X,I)\sim d_0\times\rho}
\Big[
S^\star(X,a^\star(X,I),I)-S^\star(X,a_R(X,I),I)
\Big].
\]
Assume the exact ERM and realized temperature-zero selectors are chosen measurably, and that at
every round each slate coordinate is sampled from a distribution absolutely continuous with respect
to \(\pi_0(\cdot\mid x)\).

Suppose there exists \(\gamma>0\) such that for every round \(t\ge 1\), every realized history up to
time \(t-1\), and every \(R\in\mathcal F\) satisfying
\[
\mathcal G^\star(R)\ge \varepsilon_0,
\]
one has
\begin{equation}
\EE\big[\ell_t(R)-\ell_t(S^\star)\mid \mathscr H_{t-1}\big]\ \ge\ \gamma.
\label{eq:loss_gap_eps_generic}
\end{equation}
Let
\[
\ell_{\max}\coloneqq \log K + 2B,
\qquad
N_\gamma \coloneqq \mathcal N\!\big(\mathcal F,\gamma/32,\|\cdot\|_{\infty,\supp(\pi_0)}\big),
\qquad
c_\gamma \coloneqq \frac{\gamma^2}{128\,\ell_{\max}^2}.
\]
Then:
\begin{enumerate}[label=(\roman*),leftmargin=2em]
\item For every \(t\ge 1\),
\[
\PP\big(\mathcal G^\star(\widehat R_t)\ge \varepsilon_0\big)
\le
2N_\gamma e^{-c_\gamma t}.
\]

\item With probability \(1\), only finitely many \(t\) satisfy
\[
\mathcal G^\star(\widehat R_t)\ge \varepsilon_0.
\]

\item Defining
\[
N_{\varepsilon_0}(\infty)\coloneqq \sum_{t=0}^\infty \mathbf 1\{\mathcal G^\star(\widehat R_t)\ge \varepsilon_0\},
\]
one has
\[
\EE[N_{\varepsilon_0}(\infty)]
\le
1
+
\left\lceil \frac{1}{c_\gamma}\log(2N_\gamma)\right\rceil
+
\frac{1}{e^{c_\gamma}-1}.
\]
\end{enumerate}
\end{theorem}

\begin{proof}
This is the fixed-scale exact-ERM concentration argument of \citet{kang2026demystifying}, with observations \((X,I,\mathbf A,Y)\) in place of the non-personalized \((X,\mathbf A,Y)\). The proof uses three ingredients. First, exact ERM implies the pathwise inequality
\[
\mathcal L_t(\widehat R_t)-\mathcal L_t(S^\star)\le 2b_t,
\]
where \(b_t\) is the uniform empirical-to-conditional-risk deviation. Second, on the event \(\mathcal G^\star(\widehat R_t)\ge\varepsilon_0\), the assumed conditional loss gap gives \(\mathcal L_t(\widehat R_t)-\mathcal L_t(S^\star)\ge\gamma\); hence that event is contained in \(\{b_t\ge\gamma/4\}\). Third, Lemma~\ref{lem:bt_tail_support} with \(\epsilon=\gamma/32\) and \(u=\gamma/8\) gives
\[
\PP\{b_t\ge\gamma/4\}\le 2N_\gamma e^{-c_\gamma t}.
\]
This proves part (i). Parts (ii) and (iii) follow by Borel--Cantelli and by summing \(\min\{1,2N_\gamma e^{-c_\gamma t}\}\) over \(t\), exactly as in the cited non-personalized fixed-scale lemma.
\end{proof}

\begin{lemma}[History-KL upper bound for two realizable instances]
\label{lem:history_kl_upper_bound_two_instances}
Fix two realizable true scores \(R^0\) and \(R^1\), and fix a constant \(\beta\ge1\). Consider any adaptive learner such that,
at each round \(t\), conditional on \((H_{t-1},X_t,I_t)\), the slate coordinates
\(A_{t,1},\dots,A_{t,K}\) are sampled i.i.d.\ from a distribution
\[
\pi_{I_t,t}(\cdot\mid X_t,H_{t-1})
\]
satisfying
\[
\frac{d\pi_{i,t}(\cdot\mid x,h)}{d\pi_0(\cdot\mid x)}(a)\le \beta,
\qquad
\pi_0(\cdot\mid x)\text{-a.s.}
\]
for every \((i,t,x,h)\).
Define
\[
\delta_{01}(x,a,i)\coloneqq R^1(x,a,i)-R^0(x,a,i),
\]
and
\[
\kappa_{01}
\coloneqq
K\beta\,
\EE_{X\sim d_0,\ I\sim\rho,\ A\sim\pi_0(\cdot\mid X)}
\!\left[\delta_{01}(X,A,I)^2\right].
\]
Then for every \(t\ge 1\),
\[
\KL\!\left(\mathrm{Law}_{R^0}(H_{t-1})\ \|\ \mathrm{Law}_{R^1}(H_{t-1})\right)
\le
\frac{t-1}{2}\,\kappa_{01}.
\]
\end{lemma}

\begin{proof}
By the chain rule for KL,
\begin{align*}
\KL\!\left(\mathrm{Law}_{R^0}(H_{t-1})\ \|\ \mathrm{Law}_{R^1}(H_{t-1})\right)
=
\sum_{s=1}^{t-1}
\EE_{R^0}\Big[
&\KL\Big(
\mathrm{Law}_{R^0}(X_s,I_s,\mathbf A_s,Y_s\mid H_{s-1})\\
&\hspace{4em}\big\|\ 
\mathrm{Law}_{R^1}(X_s,I_s,\mathbf A_s,Y_s\mid H_{s-1})
\Big)
\Big].
\end{align*}
Given \(H_{s-1}\), the conditional laws of \((X_s,I_s,\mathbf A_s)\) are the same under both
instances: \((X_s,I_s)\sim d_0\times \rho\), and \(\mathbf A_s\) is generated by the learner from
the same kernel. Therefore only \(Y_s\) contributes:
\[
\KL\Big(
\mathrm{Law}_{R^0}(X_s,I_s,\mathbf A_s,Y_s\mid H_{s-1})
\ \big\|\
\mathrm{Law}_{R^1}(X_s,I_s,\mathbf A_s,Y_s\mid H_{s-1})
\Big)
=
\EE_{R^0}\!\left[
\KL\!\left(
P_{R^0}(\cdot\mid X_s,I_s,\mathbf A_s)\ \|\ P_{R^1}(\cdot\mid X_s,I_s,\mathbf A_s)
\right)
\Bigm| H_{s-1}
\right].
\]
By Lemma~\ref{lem:softmax_kl_upper_bound},
\[
\KL\!\left(
P_{R^0}(\cdot\mid x,i,\mathbf a)\ \|\ P_{R^1}(\cdot\mid x,i,\mathbf a)
\right)
\le
\frac12 \sum_{k=1}^K \delta_{01}(x,a_k,i)^2.
\]
Hence
\[
\EE_{R^0}\!\left[
\KL\!\left(
P_{R^0}(\cdot\mid X_s,I_s,\mathbf A_s)\ \|\ P_{R^1}(\cdot\mid X_s,I_s,\mathbf A_s)
\right)
\Bigm| H_{s-1}
\right]
\le
\frac12 \sum_{k=1}^K
\EE_{R^0}\!\left[
\delta_{01}(X_s,A_{s,k},I_s)^2
\Bigm| H_{s-1}
\right].
\]
For each \(k\),
\[
\EE_{R^0}\!\left[
\delta_{01}(X_s,A_{s,k},I_s)^2
\Bigm| H_{s-1}
\right]
\le
\beta\,
\EE_{X\sim d_0,\ I\sim\rho,\ A\sim\pi_0(\cdot\mid X)}
\!\left[\delta_{01}(X,A,I)^2\right].
\]
Therefore
\[
\EE_{R^0}\!\left[
\KL\!\left(
P_{R^0}(\cdot\mid X_s,I_s,\mathbf A_s)\ \|\ P_{R^1}(\cdot\mid X_s,I_s,\mathbf A_s)
\right)
\Bigm| H_{s-1}
\right]
\le
\frac12 \kappa_{01}.
\]
Summing over \(s=1,\dots,t-1\) proves the claim.
\end{proof}

\begin{theorem}[Two-instance minimax lower bound]
\label{thm:two_instance_minimax_lower_bound}
Let \(R^0\) and \(R^1\) be two realizable true scores, and assume both instances satisfy a
supportwise gap condition with the same constant \(\Delta_{\min}>0\): for \(m\in\{0,1\}\), the
supportwise maximizer \(a_{R^m}(X,I)\) is unique and
\[
R^m\!\big(X,a_{R^m}(X,I),I\big)
-
\sup_{a\in \supp(\pi_0(\cdot\mid X)),\ a\neq a_{R^m}(X,I)}
R^m(X,a,I)
\ge
\Delta_{\min}
\qquad
d_0\times\rho\text{-a.s.}
\]
Let
\[
D_{01}
\coloneqq
\{(x,i)\in \cX\times\cU:\ a_{R^0}(x,i)\neq a_{R^1}(x,i)\},
\qquad
s_{01}\coloneqq (d_0\times\rho)(D_{01}),
\]
and let \(\kappa_{01}\) be as in Lemma~\ref{lem:history_kl_upper_bound_two_instances}.
Then for every adaptive learner satisfying the sampling condition of
Lemma~\ref{lem:history_kl_upper_bound_two_instances},
\[
\max\big\{\Regret_0^{(0)}(T),\Regret_0^{(1)}(T)\big\}
\ge
\frac{\Delta_{\min}}{2}
\sum_{t=1}^T
\left(
s_{01}-\frac12\sqrt{(t-1)\kappa_{01}}
\right)_+,
\]
where \(\Regret_0^{(m)}(T)\) denotes the expected temperature-zero recommendation regret when
\(R^m\) is the true score.
\end{theorem}

\begin{proof}
Let \(H_{t-1}\) be the interaction history up to round \(t-1\). Realize any learner randomization
at round \(t\) by an auxiliary random seed \(U_t\), independent of the environment and of the
past, and write
\[
\hat a_t=\hat a_t(H_{t-1},X_t,I_t,U_t).
\]
Fix \(t\ge 1\), and define the event
\[
B_t
\coloneqq
\left\{
(X_t,I_t)\in D_{01},
\ \hat a_t=a_{R^0}(X_t,I_t)
\right\}.
\]

Let \(G_0\) and \(G_1\) be full \(d_0\times\rho\)-measure sets on which the supportwise gap
condition holds for instances \(0\) and \(1\), respectively. Since the marginal law of
\((X_t,I_t)\) is \(d_0\times\rho\) under both instances, intersecting with \(G_0\) or \(G_1\) does
not change the probabilities below.

Under instance \(0\), on the event \((D_{01}\setminus B_t)\cap G_0\), the learner recommends an
action different from the unique optimal action \(a_{R^0}(X_t,I_t)\). By the gap condition, the
round-\(t\) regret is therefore at least \(\Delta_{\min}\). Hence
\[
\EE_{R^0}[r_t]\ge \Delta_{\min}\,\PP_{R^0}(D_{01}\setminus B_t),
\]
where \(r_t\) denotes the round-\(t\) temperature-zero regret under instance \(0\).

Under instance \(1\), on the event \(B_t\cap G_1\), the learner recommends
\(a_{R^0}(X_t,I_t)\neq a_{R^1}(X_t,I_t)\), so again by the gap condition, the round-\(t\) regret is
at least \(\Delta_{\min}\). Therefore
\[
\EE_{R^1}[r_t]\ge \Delta_{\min}\,\PP_{R^1}(B_t).
\]
Adding the last two displays gives
\[
\EE_{R^0}[r_t]+\EE_{R^1}[r_t]
\ge
\Delta_{\min}\Big(\PP_{R^0}(D_{01}\setminus B_t)+\PP_{R^1}(B_t)\Big).
\]
Since \(\PP_{R^0}(D_{01})=s_{01}\),
\[
\PP_{R^0}(D_{01}\setminus B_t)+\PP_{R^1}(B_t)
=
s_{01}-\PP_{R^0}(B_t)+\PP_{R^1}(B_t)
\ge
s_{01}-\big|\PP_{R^0}(B_t)-\PP_{R^1}(B_t)\big|.
\]
Now let
\[
Z_t\coloneqq (H_{t-1},X_t,I_t,U_t).
\]
Since \(B_t\in \sigma(Z_t)\),
\[
\big|\PP_{R^0}(B_t)-\PP_{R^1}(B_t)\big|
\le
\operatorname{TV}\!\left(
\mathrm{Law}_{R^0}(Z_t),\mathrm{Law}_{R^1}(Z_t)
\right).
\]
Moreover, \(Z_t\) is obtained from \(H_{t-1}\) by passing through the same Markov kernel under
both instances, so data processing for total variation gives
\[
\operatorname{TV}\!\left(
\mathrm{Law}_{R^0}(Z_t),\mathrm{Law}_{R^1}(Z_t)
\right)
\le
\operatorname{TV}\!\left(
\mathrm{Law}_{R^0}(H_{t-1}),\mathrm{Law}_{R^1}(H_{t-1})
\right).
\]
By Pinsker's inequality and
Lemma~\ref{lem:history_kl_upper_bound_two_instances},
\[
\operatorname{TV}\!\left(
\mathrm{Law}_{R^0}(H_{t-1}),\mathrm{Law}_{R^1}(H_{t-1})
\right)
\le
\sqrt{
\frac12\,
\KL\!\left(
\mathrm{Law}_{R^0}(H_{t-1})\ \|\ \mathrm{Law}_{R^1}(H_{t-1})
\right)
}
\le
\frac12\sqrt{(t-1)\kappa_{01}}.
\]
Combining the previous displays yields
\[
\EE_{R^0}[r_t]+\EE_{R^1}[r_t]
\ge
\Delta_{\min}
\left(
s_{01}-\frac12\sqrt{(t-1)\kappa_{01}}
\right)_+.
\]
Summing over \(t=1,\dots,T\) and dividing by \(2\) gives the result.
\end{proof}

For a fixed regret scale \(\varepsilon_0>0\), define the number of \(\varepsilon_0\)-substantial rounds by
\begin{equation}
N_{p,\varepsilon_0}(T)
\coloneqq
\sum_{t=0}^{T-1}\mathbf 1\big\{\mathcal G_p(\widehat R_t)\ge \varepsilon_0\big\},
\label{eq:eps_substantial_rounds_def}
\end{equation}
and
\[
N_{p,\varepsilon_0}(\infty)
\coloneqq
\sum_{t=0}^{\infty}\mathbf 1\big\{\mathcal G_p(\widehat R_t)\ge \varepsilon_0\big\}.
\]

\begin{theorem}[Fixed-scale minimax characterization]
\label{thm:fixed_scale_characterization}
Assume \ref{as:LP1}--\ref{as:LP2}, and fix \(\varepsilon_0>0\). Then the following are equivalent:
\begin{enumerate}[label=(\roman*),leftmargin=2em]
\item
\[
\liminf_{r\downarrow0}
\underline{\mathfrak d}_{\mathrm{reg}}(r;\varepsilon_0)>0.
\]

\item
\[
\liminf_{r\downarrow0}
\underline{\Gamma}_{\mathrm{reg}}(r;\varepsilon_0)>0.
\]

\item There exists a finite constant \(\gamma_{\mathcal P,\varepsilon_0}>0\) such that for every pair
\(p,q\in\mathcal P\),
\[
\mathcal G_p(R_q)\ge \varepsilon_0
\quad\Longrightarrow\quad
\mathcal L_p(R_q)-\mathcal L_p(R_p)\ge \gamma_{\mathcal P,\varepsilon_0}.
\]
\end{enumerate}

If these equivalent conditions hold, fix any such finite witness
\(\gamma_{\mathcal P,\varepsilon_0}\) and define
\[
\beta\coloneqq e^{2\eta B_{\mathcal P}},
\qquad
\gamma\coloneqq \beta^{-K}\gamma_{\mathcal P,\varepsilon_0},
\qquad
\ell_{\max}\coloneqq \log K + 2B_{\mathcal P},
\]
\[
N_\gamma
\coloneqq
\mathcal N\!\big(\mathcal F_{\mathcal P},\gamma/32,\|\cdot\|_{\infty,\supp(\pi_0)}\big),
\qquad
c_\gamma\coloneqq \frac{\gamma^2}{128\,\ell_{\max}^2}.
\]
Then \(N_\gamma<\infty\), and the exact-ERM greedy learner satisfies, for every truth
\(p\in\mathcal P\) and every \(t\ge 1\),
\[
\PP_p\big(\mathcal G_p(\widehat R_t)\ge \varepsilon_0\big)
\le
2N_\gamma e^{-c_\gamma t}.
\]
Consequently,
\[
\sup_{p\in\mathcal P}\EE_p[N_{p,\varepsilon_0}(\infty)]
\le
1
+
\left\lceil \frac{1}{c_\gamma}\log(2N_\gamma)\right\rceil
+
\frac{1}{e^{c_\gamma}-1}
<\infty,
\]
and, for every truth \(p\in\mathcal P\), with \(\PP_p\)-probability \(1\), only finitely many rounds satisfy
\(\mathcal G_p(\widehat R_t)\ge \varepsilon_0\).

If instead the equivalent conditions fail, then for all sufficiently large \(T\),
\[
\inf_{\mathsf A\in\mathfrak A_\beta}
\sup_{p\in\mathcal P}
\Regret_{0,p}^{\mathsf A}(T)
\ge
\frac{\Delta_{\min}^{\mathcal P}\varepsilon_0}{8\Delta_{\max}^{\mathcal P}}\,
\log T,
\]
where \(\mathfrak A_\beta\) is the class of \(\beta\)-admissible learners in Definition~\ref{def:beta_admissible_learner}, with \(\beta=e^{2\eta B_{\mathcal P}}\).
\end{theorem}

\begin{proof}[Proof of Theorem~\ref{thm:fixed_scale_characterization}]
We first prove the equivalence of (i), (ii), and (iii).

The equivalence of (i) and (ii) is exactly Lemma~\ref{lem:equivalence_three_liminf_conditions}.

We next prove
\[
\text{(iii)}\Longrightarrow \text{(ii)}.
\]
Assume there exists a finite \(\gamma_{\mathcal P,\varepsilon_0}>0\) such that for every \(p,q\in\mathcal P\),
\[
\mathcal G_p(R_q)\ge \varepsilon_0
\quad\Longrightarrow\quad
\mathcal L_p(R_q)-\mathcal L_p(R_p)\ge \gamma_{\mathcal P,\varepsilon_0}.
\]
Then for every \(r>0\), every \(p\in\mathcal P\), and every
\[
q\in \mathcal Q_{\mathrm{reg}}(p,r;\varepsilon_0),
\]
we have
\[
\mathcal L_p(R_q)-\mathcal L_p(R_p)\ge \gamma_{\mathcal P,\varepsilon_0}.
\]
Hence
\[
\Gamma_{\mathrm{reg},p}(r;\varepsilon_0)\ge \gamma_{\mathcal P,\varepsilon_0}
\qquad
\forall p\in\mathcal P,
\]
and therefore
\[
\underline{\Gamma}_{\mathrm{reg}}(r;\varepsilon_0)\ge \gamma_{\mathcal P,\varepsilon_0}
\qquad
\forall r>0.
\]
This implies
\[
\liminf_{r\downarrow0}
\underline{\Gamma}_{\mathrm{reg}}(r;\varepsilon_0)>0.
\]

We now prove
\[
\text{(ii)}\Longrightarrow \text{(iii)}.
\]
Assume
\[
\liminf_{r\downarrow0}
\underline{\Gamma}_{\mathrm{reg}}(r;\varepsilon_0)>0.
\]
Choose \(r_0>0\) small enough that
\[
\underline{\Gamma}_{\mathrm{reg}}(r_0;\varepsilon_0)>0,
\]
and set
\[
\gamma_{\mathcal P,\varepsilon_0}
\coloneqq
\min\Big\{
1,
\underline{\Gamma}_{\mathrm{reg}}(r_0;\varepsilon_0),
\gamma_{\mathrm{far}}^{\mathcal P}(r_0;\varepsilon_0)
\Big\}.
\]
By Lemma~\ref{lem:positive_far_pair_risk_gap_P},
\[
\gamma_{\mathrm{far}}^{\mathcal P}(r_0;\varepsilon_0)>0,
\]
so \(\gamma_{\mathcal P,\varepsilon_0}\in(0,1]\) is finite.

Fix any \(p,q\in\mathcal P\) such that
\[
\mathcal G_p(R_q)\ge \varepsilon_0.
\]
If
\[
\max_i\|\lambda_{i,q}-\lambda_{i,p}\|_2\le r_0,
\]
then
\[
q\in \mathcal Q_{\mathrm{reg}}(p,r_0;\varepsilon_0),
\]
so
\[
\mathcal L_p(R_q)-\mathcal L_p(R_p)
\ge
\Gamma_{\mathrm{reg},p}(r_0;\varepsilon_0)
\ge \gamma_{\mathcal P,\varepsilon_0}.
\]
If instead
\[
\max_i\|\lambda_{i,q}-\lambda_{i,p}\|_2\ge r_0,
\]
then by definition of \(\gamma_{\mathrm{far}}^{\mathcal P}(r_0;\varepsilon_0)\),
\[
\mathcal L_p(R_q)-\mathcal L_p(R_p)
\ge
\gamma_{\mathcal P,\varepsilon_0}.
\]
Therefore in all cases,
\[
\mathcal L_p(R_q)-\mathcal L_p(R_p)\ge \gamma_{\mathcal P,\varepsilon_0}.
\]
This proves (iii).

We now prove the upper statement under the equivalent conditions.

By Condition~\ref{as:LP1}, \(p\mapsto R_p\) is continuous on compact \(\mathcal P\), so
\[
\mathcal F_{\mathcal P}\coloneqq \{R_p:p\in\mathcal P\}
\]
is compact under \(\|\cdot\|_{\infty,\supp(\pi_0)}\). In particular,
\[
N_\gamma
=
\mathcal N\!\big(\mathcal F_{\mathcal P},\gamma/32,\|\cdot\|_{\infty,\supp(\pi_0)}\big)
<\infty.
\]

Fix any truth \(p\in\mathcal P\). Consider any \(q\in\mathcal P\) satisfying
\[
\mathcal G_p(R_q)\ge \varepsilon_0.
\]
Condition on \(\mathscr H_{t-1}\). By Lemma~\ref{lem:excess_is_kl_support},
\[
\EE\big[\ell_t(R_q)-\ell_t(R_p)\mid \mathscr H_{t-1}\big]
=
\EE\Big[
\KL\big(P_{R_p}(\cdot\mid X_t,I_t,\mathbf A_t)\,\|\,P_{R_q}(\cdot\mid X_t,I_t,\mathbf A_t)\big)
\ \Big|\ \mathscr H_{t-1}
\Big].
\]
The integrand is nonnegative. Since the exact-ERM greedy learner samples from KL-tilted policies generated by scores bounded by \(B_{\mathcal P}\), Lemma~\ref{lem:greedy_kl_tilt_lr} gives the lower likelihood-ratio bound with \(\beta=e^{2\eta B_{\mathcal P}}\). Applying Lemma~\ref{lem:slate_domination_likelihood_ratio} gives
\[
\EE\big[\ell_t(R_q)-\ell_t(R_p)\mid \mathscr H_{t-1}\big]
\ge
\beta^{-K}\big(\mathcal L_p(R_q)-\mathcal L_p(R_p)\big)
\ge
\beta^{-K}\gamma_{\mathcal P,\varepsilon_0}.
\]
Thus the loss-gap hypothesis of
Theorem~\ref{thm:eps_substantial_rounds_from_loss_gap} holds with
\[
S^\star=R_p,
\qquad
\mathcal F=\mathcal F_{\mathcal P},
\qquad
\gamma=\beta^{-K}\gamma_{\mathcal P,\varepsilon_0}.
\]
Applying Theorem~\ref{thm:eps_substantial_rounds_from_loss_gap} yields, for every \(t\ge 1\),
\[
\PP_p\big(\mathcal G_p(\widehat R_t)\ge \varepsilon_0\big)
\le
2N_\gamma e^{-c_\gamma t},
\]
and also
\[
\EE_p[N_{p,\varepsilon_0}(\infty)]
\le
1
+
\left\lceil \frac{1}{c_\gamma}\log(2N_\gamma)\right\rceil
+
\frac{1}{e^{c_\gamma}-1}.
\]
Since the constants depend only on \(\mathcal P\), the bound is uniform in \(p\in\mathcal P\).
The almost-sure finiteness of \(\varepsilon_0\)-substantial rounds also follows from
Theorem~\ref{thm:eps_substantial_rounds_from_loss_gap}.

We now prove the lower statement when the equivalent conditions fail. Set
\[
\beta\coloneqq e^{2\eta B_{\mathcal P}}.
\]
Suppose the equivalent conditions fail. By
Lemma~\ref{lem:equivalence_three_liminf_conditions}, this is equivalent to
\[
\liminf_{r\downarrow0}
\underline{\mathfrak c}_{\mathrm{reg}}(r;\varepsilon_0)=0.
\]
Hence there exist a sequence \(r_n\downarrow0\), truths \(p_n\in\mathcal P\), and competitors
\[
q_n\in \mathcal Q_{\mathrm{reg}}(p_n,r_n;\varepsilon_0)
\]
such that
\[
\EE\!\left[\Delta_{q_n\mid p_n}(X,A,I)^2\right]\to 0.
\]
Because
\[
q_n\in \mathcal Q_{\mathrm{reg}}(p_n,r_n;\varepsilon_0),
\]
we have
\[
\mathcal G_{p_n}(R_{q_n})\ge \varepsilon_0.
\]
By Lemma~\ref{lem:regret_vs_disagreement},
\[
s_n
\coloneqq
(d_0\times\rho)\{a_{q_n}\neq a_{p_n}\}
\ge
\frac{\varepsilon_0}{\Delta_{\max}^{\mathcal P}}
\eqqcolon s_0>0.
\]
Define
\[
\kappa_n
\coloneqq
K\beta\,
\EE\!\left[\Delta_{q_n\mid p_n}(X,A,I)^2\right].
\]
Then \(\kappa_n\to 0\).

Fix any \(\beta\)-admissible learner \(\mathsf A\). By
Theorem~\ref{thm:two_instance_minimax_lower_bound}, applied to the pair of truths
\[
R^0=R_{p_n},
\qquad
R^1=R_{q_n},
\]
we have
\[
\max\Big\{\Regret_{0,p_n}^{\mathsf A}(T),\Regret_{0,q_n}^{\mathsf A}(T)\Big\}
\ge
\frac{\Delta_{\min}^{\mathcal P}}{2}
\sum_{t=1}^T
\left(
s_n-\frac12\sqrt{(t-1)\kappa_n}
\right)_+.
\]
Now choose \(n=n(T)\) large enough that
\[
\kappa_{n(T)}\le \frac{s_0^2}{\log T}.
\]
Then for every \(t=1,\dots,\lfloor \log T\rfloor\),
\[
\frac12\sqrt{(t-1)\kappa_{n(T)}}\le \frac12 s_0,
\]
and hence
\[
\left(
s_{n(T)}-\frac12\sqrt{(t-1)\kappa_{n(T)}}
\right)_+
\ge
\frac{s_0}{2}.
\]
Therefore
\begin{align*}
\max\Big\{\Regret_{0,p_{n(T)}}^{\mathsf A}(T),\Regret_{0,q_{n(T)}}^{\mathsf A}(T)\Big\}
&\ge
\frac{\Delta_{\min}^{\mathcal P}}{2}
\sum_{t=1}^{\lfloor \log T\rfloor} \frac{s_0}{2}\\
&=
\frac{\Delta_{\min}^{\mathcal P}s_0}{4}\,\lfloor \log T\rfloor.
\end{align*}
For all sufficiently large \(T\),
\[
\lfloor \log T\rfloor\ge \frac12 \log T,
\]
so
\[
\max\Big\{\Regret_{0,p_{n(T)}}^{\mathsf A}(T),\Regret_{0,q_{n(T)}}^{\mathsf A}(T)\Big\}
\ge
\frac{\Delta_{\min}^{\mathcal P}s_0}{8}\,\log T
=
\frac{\Delta_{\min}^{\mathcal P}\varepsilon_0}{8\Delta_{\max}^{\mathcal P}}\,
\log T.
\]
Since \(p_{n(T)},q_{n(T)}\in\mathcal P\), this implies
\[
\sup_{p\in\mathcal P}\Regret_{0,p}^{\mathsf A}(T)
\ge
\frac{\Delta_{\min}^{\mathcal P}\varepsilon_0}{8\Delta_{\max}^{\mathcal P}}\,
\log T
\]
for all sufficiently large \(T\). Taking the infimum over \(\mathsf A\in\mathfrak A_\beta\) proves the lower bound.
\end{proof}

\subsection{Proof of the bounded-regret theorem}
\label{app:bounded_regret_pf}

\begin{proof}[Proof of Theorem~\ref{thm:bounded_regret_under_umr}]
Let
\[
\varepsilon_0\coloneqq \varepsilon_{\mathrm{iso}}^{\mathcal P}.
\]
By Theorem~\ref{thm:fixed_scale_characterization}, the exact-ERM greedy learner satisfies, for every truth \(p\in\mathcal P\) and every \(t\ge 1\),
\[
\PP_p\bigl(\mathcal G_p(\widehat R_t)\ge \varepsilon_0\bigr)
\le
2N_\gamma e^{-c_\gamma t}.
\]
Fix \(p\in\mathcal P\) and \(T\ge 1\). The initialization round is handled separately because \(\widehat R_0\equiv0\) need not belong to \(\mathcal F_{\mathcal P}\):
\[
0\le \mathcal G_p(\widehat R_0)\le \Delta_{\max}^{\mathcal P}.
\]
For every \(t\ge1\), \(\widehat R_t\in\mathcal F_{\mathcal P}\), so Lemma~\ref{lem:automatic_positive_regret_isolation_P} implies
\[
\mathcal G_p(\widehat R_t)>0
\quad\Longrightarrow\quad
\mathcal G_p(\widehat R_t)\ge \varepsilon_0.
\]
Since \(0\le \mathcal G_p(\widehat R_t)\le \Delta_{\max}^{\mathcal P}\),
\[
\mathcal G_p(\widehat R_t)
\le
\Delta_{\max}^{\mathcal P}\,\mathbf 1\{\mathcal G_p(\widehat R_t)\ge \varepsilon_0\}
\qquad
\forall t\ge1.
\]
Therefore
\[
\Regret_{0,p}^{\mathrm{ERM}}(T)
\le
\Delta_{\max}^{\mathcal P}
+
\Delta_{\max}^{\mathcal P}
\sum_{t=1}^{T-1}
\PP_p\bigl(\mathcal G_p(\widehat R_t)\ge \varepsilon_0\bigr).
\]
Set
\[
t_0\coloneqq \left\lceil \frac{1}{c_\gamma}\log(2N_\gamma)\right\rceil.
\]
Then
\[
\sum_{t=1}^{\infty}
\PP_p\bigl(\mathcal G_p(\widehat R_t)\ge \varepsilon_0\bigr)
\le
\sum_{t=1}^{\infty}\min\{1,2N_\gamma e^{-c_\gamma t}\}
\le
t_0+\frac{1}{e^{c_\gamma}-1}.
\]
Thus
\[
\sup_{p\in\mathcal P}\sup_{T\ge1}
\Regret_{0,p}^{\mathrm{ERM}}(T)
\le
\Delta_{\max}^{\mathcal P}
\left(
1+
\left\lceil \frac{1}{c_\gamma}\log(2N_\gamma)\right\rceil+
\frac{1}{e^{c_\gamma}-1}
\right).
\]
\end{proof}

\begin{proof}[Proof of Theorem~\ref{thm:log_lower_without_diversity}]
Apply the lower-bound part of Theorem~\ref{thm:fixed_scale_characterization} with
\[
\varepsilon_0=\varepsilon_{\mathrm{iso}}^{\mathcal P}
\]
and with \(\mathfrak A_\beta\) specialized to
\[
\beta=e^{2\eta B_{\mathcal P}}.
\]
This gives exactly the displayed logarithmic lower bound for the learner class specified in the theorem.
\end{proof}

\section{Proofs for Section~\ref{sec:offline_exact_erm_pi0}}
\label{app:proofs_offline_alignment}

\subsection{Exponential control at a fixed regret scale}
\label{app:offline_fixed_scale}

The proof uses the following automatic fixed-scale gap. It is a compactness-identifiability
consequence of Conditions~\ref{as:LP1}--\ref{as:LP2}, not an additional diversity
assumption.

\begin{lemma}[Automatic fixed-scale truth-centered loss gap]
\label{lem:automatic_fixed_scale_loss_gap_P}
Assume \ref{as:LP1}--\ref{as:LP2}, and fix \(\varepsilon_0>0\). Define
\[
K_{\varepsilon_0}
\coloneqq
\{(p,q)\in\mathcal P^2:\mathcal G_p(R_q)\ge \varepsilon_0\},
\]
and
\[
\gamma_{\mathcal P,\varepsilon_0}
\coloneqq
\begin{cases}
1, & K_{\varepsilon_0}=\emptyset,\\[0.6em]
\displaystyle
\min\left\{
1,\,
\min_{(p,q)\in K_{\varepsilon_0}}
\bigl(\mathcal L_p(R_q)-\mathcal L_p(R_p)\bigr)
\right\},
& K_{\varepsilon_0}\neq\emptyset.
\end{cases}
\]
Then \(\gamma_{\mathcal P,\varepsilon_0}\) is finite and strictly positive, and for every pair
\(p,q\in\mathcal P\),
\[
\mathcal G_p(R_q)\ge \varepsilon_0
\quad\Longrightarrow\quad
\mathcal L_p(R_q)-\mathcal L_p(R_p)
\ge \gamma_{\mathcal P,\varepsilon_0}.
\]
Equivalently, since \(\mathcal F_{\mathcal P}=\{R_q:q\in\mathcal P\}\), for every truth
\(p\in\mathcal P\) and every score \(R\in\mathcal F_{\mathcal P}\),
\[
\mathcal G_p(R)\ge \varepsilon_0
\quad\Longrightarrow\quad
\mathcal L_p(R)-\mathcal L_p(R_p)
\ge \gamma_{\mathcal P,\varepsilon_0}.
\]
\end{lemma}

\begin{proof}
By Lemma~\ref{lem:compact_continuity_consequences_P}, \(\mathcal P\) is compact and the
induced score class \(\mathcal F_{\mathcal P}\) is compact under
\(\|\cdot\|_{\infty,\supp(\pi_0)}\). Hence \(\mathcal P^2\) is compact. By
Lemma~\ref{lem:continuity_truth_centered_maps}, the maps
\[
(p,q)\longmapsto \mathcal G_p(R_q)
\qquad\text{and}\qquad
(p,q)\longmapsto \mathcal L_p(R_q)-\mathcal L_p(R_p)
\]
are continuous on \(\mathcal P^2\).

Define the closed subset
\[
K_{\varepsilon_0}
\coloneqq
\{(p,q)\in\mathcal P^2:\mathcal G_p(R_q)\ge \varepsilon_0\}.
\]
Since \(K_{\varepsilon_0}\) is a closed subset of compact \(\mathcal P^2\), it is compact. If
\(K_{\varepsilon_0}=\emptyset\), the implication is vacuous, so
\(\gamma_{\mathcal P,\varepsilon_0}=1\) works.

Assume \(K_{\varepsilon_0}\neq\emptyset\). The continuous function
\[
H(p,q)\coloneqq \mathcal L_p(R_q)-\mathcal L_p(R_p)
\]
attains its minimum on \(K_{\varepsilon_0}\). Moreover \(H(p,q)\ge0\) for all \(p,q\), since
Lemma~\ref{lem:excess_is_kl_support} gives
\[
\mathcal L_p(R_q)-\mathcal L_p(R_p)
=
\EE\!\left[
\KL\!\left(
P_{R_p}(\cdot\mid X,I,\mathbf A)
\,\middle\|\,
P_{R_q}(\cdot\mid X,I,\mathbf A)
\right)
\right].
\]
If the attained minimum were zero, then for some \((\bar p,\bar q)\in K_{\varepsilon_0}\),
\[
\mathcal L_{\bar p}(R_{\bar q})=\mathcal L_{\bar p}(R_{\bar p}).
\]
By Lemma~\ref{lem:zero_loss_implies_zero_regret_P}, this would imply
\[
\mathcal G_{\bar p}(R_{\bar q})=0,
\]
contradicting the definition of \(K_{\varepsilon_0}\), because \(\varepsilon_0>0\). Therefore
\[
m_{\varepsilon_0}
\coloneqq
\min_{(p,q)\in K_{\varepsilon_0}}
\bigl(\mathcal L_p(R_q)-\mathcal L_p(R_p)\bigr)
>0.
\]
Taking
\[
\gamma_{\mathcal P,\varepsilon_0}\coloneqq \min\{1,m_{\varepsilon_0}\}
\]
gives the claimed positive finite witness. The equivalent score-class formulation follows from
\(\mathcal F_{\mathcal P}=\{R_q:q\in\mathcal P\}\).
\end{proof}

Under \(\pi_0\)-logging, the offline excess population loss is exactly the truth-centered loss gap
\(\mathcal L_p(R)-\mathcal L_p(R_p)\). Thus the online likelihood-ratio factor does not appear.

\begin{proof}[Proof of Theorem~\ref{thm:offline_exact_erm_scale_eps}]
The positivity of the constant \(\gamma_{\mathcal P,\varepsilon_0}\) defined in the theorem follows
from Lemma~\ref{lem:automatic_fixed_scale_loss_gap_P}. Because \(\mathcal F_{\mathcal P}\) is
compact under \(\|\cdot\|_{\infty,\supp(\pi_0)}\), the covering number \(N_\gamma\) is finite. Since
\(\gamma_{\mathcal P,\varepsilon_0}>0\) and \(\ell_{\max}<\infty\), one also has \(c_\gamma>0\).

Fix a truth \(p\in\mathcal P\). Work on the full-probability event on which all sampled slate
actions lie in the corresponding reference supports and the selected \(\widehat R_n\) is an exact
minimizer of \(\widehat{\mathcal L}_n\) over \(\mathcal F_{\mathcal P}\). The complementary event
has \(\PP_p\)-probability zero and does not affect the probability bound below.

Let
\[
\mathcal E_s
\coloneqq
\sigma\!\big((X_r,I_r,\mathbf A_r,Y_r)_{r=1}^s\big),
\qquad s=0,1,\dots,n,
\]
with \(\mathcal E_0\) trivial, and define the uniform deviation
\[
b_n
\coloneqq
\sup_{R\in\mathcal F_{\mathcal P}}
\big|\widehat{\mathcal L}_n(R)-\mathcal L_p(R)\big|.
\]
By \eqref{eq:offline_population_risk_identity}, for every \(R\in\mathcal F_{\mathcal P}\) and every
\(s=1,\dots,n\),
\[
\EE_p[\ell_s(R)\mid \mathcal E_{s-1}]=\mathcal L_p(R)
\qquad\text{a.s.}
\]
Hence the process \(\mathcal L_n(R)\) from Lemma~\ref{lem:bt_tail_support} is exactly
\(\mathcal L_p(R)\) in the present offline setting.

Since \(\widehat R_n\) is an exact ERM,
\[
\widehat{\mathcal L}_n(\widehat R_n)\le \widehat{\mathcal L}_n(R_p).
\]
Therefore, pathwise,
\begin{align*}
\mathcal L_p(\widehat R_n)-\mathcal L_p(R_p)
&\le
\widehat{\mathcal L}_n(\widehat R_n)+b_n-\big(\widehat{\mathcal L}_n(R_p)-b_n\big)\\
&=
\widehat{\mathcal L}_n(\widehat R_n)-\widehat{\mathcal L}_n(R_p)+2b_n\\
&\le 2b_n.
\end{align*}
Now define
\[
E_n^{\varepsilon_0}
\coloneqq
\{\mathcal G_p(\widehat R_n)\ge \varepsilon_0\}.
\]
On the event \(E_n^{\varepsilon_0}\), the score \(\widehat R_n\) belongs to
\(\mathcal F_{\mathcal P}\), and therefore, by the score-class form of
Lemma~\ref{lem:automatic_fixed_scale_loss_gap_P},
\[
\mathcal L_p(\widehat R_n)-\mathcal L_p(R_p)\ge \gamma_{\mathcal P,\varepsilon_0}.
\]
Combining with the previous display yields
\[
E_n^{\varepsilon_0}
\subseteq
\left\{b_n\ge \frac{\gamma_{\mathcal P,\varepsilon_0}}{2}\right\}
\subseteq
\left\{b_n\ge \frac{\gamma_{\mathcal P,\varepsilon_0}}{4}\right\}.
\]
Apply Lemma~\ref{lem:bt_tail_support} with
\[
\mathcal F=\mathcal F_{\mathcal P},
\qquad
B=B_{\mathcal P},
\qquad
t=n,
\qquad
\epsilon=\frac{\gamma_{\mathcal P,\varepsilon_0}}{32},
\qquad
u=\frac{\gamma_{\mathcal P,\varepsilon_0}}{8}.
\]
Because \(\mathcal L_n(R)=\mathcal L_p(R)\) here and \(u+4\epsilon=\gamma_{\mathcal P,\varepsilon_0}/4\),
Lemma~\ref{lem:bt_tail_support} gives
\[
\PP_p\!\left(b_n\ge \frac{\gamma_{\mathcal P,\varepsilon_0}}{4}\right)
\le
2N_\gamma
\exp\!\left(-\frac{n(\gamma_{\mathcal P,\varepsilon_0}/8)^2}{2\ell_{\max}^2}\right)
=
2N_\gamma e^{-c_\gamma n}.
\]
Therefore
\[
\PP_p(E_n^{\varepsilon_0})
\le
2N_\gamma e^{-c_\gamma n},
\]
as claimed.
\end{proof}

\subsection{Logarithmic accuracy complexity}
\label{app:offline_umr_log_accuracy}

\begin{proof}[Proof of Corollary~\ref{cor:offline_log_eps_umr}]
Let
\[
\varepsilon_0\coloneqq \varepsilon_{\mathrm{iso}}^{\mathcal P}
\]
be the constant from Lemma~\ref{lem:automatic_positive_regret_isolation_P}.
Fix any truth \(p\in\mathcal P\). Since \(\mathcal F_{\mathcal P}=\{R_q:q\in\mathcal P\}\),
Lemma~\ref{lem:automatic_positive_regret_isolation_P} implies that for every score
\(R\in\mathcal F_{\mathcal P}\),
\[
\mathcal G_p(R)\in\{0\}\cup[\varepsilon_0,\infty).
\]
Hence
\[
\mathcal G_p(\widehat R_n)>0
\quad\Longrightarrow\quad
\mathcal G_p(\widehat R_n)\ge \varepsilon_0.
\]
Since also
\[
0\le \mathcal G_p(\widehat R_n)\le \Delta_{\max}^{\mathcal P},
\]
we obtain the pointwise bound
\[
\mathcal G_p(\widehat R_n)
\le
\Delta_{\max}^{\mathcal P}\,\mathbf 1\{\mathcal G_p(\widehat R_n)\ge \varepsilon_0\}.
\]
Taking expectations and applying Theorem~\ref{thm:offline_exact_erm_scale_eps} gives
\[
\mathcal R_{0,p}^{\mathrm{off}}(n)
\le
\Delta_{\max}^{\mathcal P}\,\PP_p\big(\mathcal G_p(\widehat R_n)\ge \varepsilon_0\big)
\le
2\Delta_{\max}^{\mathcal P}N_\gamma e^{-c_\gamma n}.
\]
Now take the supremum over \(p\in\mathcal P\). The sample-complexity statement follows by solving
\[
2\Delta_{\max}^{\mathcal P}N_\gamma e^{-c_\gamma n}\le \varepsilon
\]
for \(n\).
\end{proof}

\subsection{Zero-regret identification after a logarithmic burn-in}
\label{sec:offline_zero_regret_identification}
\label{app:offline_zero_regret_identification}

The preceding expected-regret bound has a sharper decision-level interpretation. Because positive
temperature-zero regret is isolated on the compact parameter class, controlling the fixed scale
\(\varepsilon_{\mathrm{iso}}^{\mathcal P}\) is equivalent to controlling the event of any
nonzero temperature-zero regret. Thus, under Conditions~\ref{as:LP1}--\ref{as:LP2}, offline ERM
does not merely drive the mean regret down continuously; after a logarithmic burn-in it selects a
zero-regret temperature-zero recommendation with high probability.

\begin{corollary}[Offline zero-regret identification]
\label{cor:offline_zero_regret_identification}
Assume \ref{as:LP1}--\ref{as:LP2}. Let \(N_\gamma\) and \(c_\gamma\) be the constants from
Theorem~\ref{thm:offline_exact_erm_scale_eps} corresponding to
\[
\varepsilon_0\coloneqq \varepsilon_{\mathrm{iso}}^{\mathcal P}.
\]
Let \(\widehat R_n\) be the offline exact ERM trained on \(n\) \(\pi_0\)-logged samples. Then, for
all \(n\ge1\),
\[
\sup_{p\in\mathcal P}
\PP_p\!\left(\mathcal G_p(\widehat R_n)>0\right)
\le
2N_\gamma e^{-c_\gamma n}.
\]
Consequently,
\[
\sup_{p\in\mathcal P}
\mathcal R_{0,p}^{\mathrm{off}}(n)
\le
2\Delta_{\max}^{\mathcal P}N_\gamma e^{-c_\gamma n}.
\]
Equivalently, for every \(\delta\in(0,1)\), if
\[
n\ge
\frac1{c_\gamma}\log\!\left(\frac{2N_\gamma}{\delta}\right),
\]
then
\[
\inf_{p\in\mathcal P}
\PP_p\!\left(\mathcal G_p(\widehat R_n)=0\right)
\ge 1-\delta.
\]
Moreover, if \(\widehat R_n\) is computed from the first \(n\) observations of one infinite i.i.d.
\(\pi_0\)-logged sample sequence, then for every fixed truth \(p\in\mathcal P\),
\[
\PP_p\!\left(
\exists N<\infty\ \text{such that }
\mathcal G_p(\widehat R_n)=0\ \forall n\ge N
\right)=1.
\]
\end{corollary}

\begin{proof}[Proof of Corollary~\ref{cor:offline_zero_regret_identification}]
Let
\[
\varepsilon_0\coloneqq \varepsilon_{\mathrm{iso}}^{\mathcal P}.
\]
Fix any truth \(p\in\mathcal P\). Since \(\widehat R_n\in\mathcal F_{\mathcal P}\) on the
full-probability support event and the default value on the complementary null event was chosen in
\(\mathcal F_{\mathcal P}\), Lemma~\ref{lem:automatic_positive_regret_isolation_P} gives
\[
\mathcal G_p(\widehat R_n)>0
\quad\Longrightarrow\quad
\mathcal G_p(\widehat R_n)\ge \varepsilon_0.
\]
Therefore
\[
\PP_p\!\left(\mathcal G_p(\widehat R_n)>0\right)
\le
\PP_p\!\left(\mathcal G_p(\widehat R_n)\ge\varepsilon_0\right).
\]
Applying Theorem~\ref{thm:offline_exact_erm_scale_eps} at
\(\varepsilon_0=\varepsilon_{\mathrm{iso}}^{\mathcal P}\) yields
\[
\PP_p\!\left(\mathcal G_p(\widehat R_n)>0\right)
\le
2N_\gamma e^{-c_\gamma n}.
\]
The bound is uniform in \(p\), so taking the supremum over \(p\in\mathcal P\) proves the first
claim.

The expected-regret bound follows from the pointwise envelope
\[
0\le \mathcal G_p(\widehat R_n)\le \Delta_{\max}^{\mathcal P}
\]
and the preceding event bound:
\[
\mathcal R_{0,p}^{\mathrm{off}}(n)
=
\EE_p[\mathcal G_p(\widehat R_n)]
\le
\Delta_{\max}^{\mathcal P}
\PP_p\!\left(\mathcal G_p(\widehat R_n)>0\right)
\le
2\Delta_{\max}^{\mathcal P}N_\gamma e^{-c_\gamma n}.
\]
Taking the supremum over \(p\in\mathcal P\) gives the displayed uniform expected-regret bound.

For the high-probability statement, solve
\[
2N_\gamma e^{-c_\gamma n}\le \delta
\]
for \(n\). This gives
\[
n\ge \frac1{c_\gamma}\log\!\left(\frac{2N_\gamma}{\delta}\right),
\]
and hence, uniformly over \(p\in\mathcal P\),
\[
\PP_p\!\left(\mathcal G_p(\widehat R_n)=0\right)
=
1-
\PP_p\!\left(\mathcal G_p(\widehat R_n)>0\right)
\ge 1-\delta.
\]

Finally, suppose that \(\widehat R_n\) is computed from the first \(n\) observations of one infinite
i.i.d.\ \(\pi_0\)-logged sample sequence. For any fixed \(p\in\mathcal P\), the first part gives
\[
\sum_{n=1}^\infty
\PP_p\!\left(\mathcal G_p(\widehat R_n)>0\right)
\le
2N_\gamma\sum_{n=1}^\infty e^{-c_\gamma n}
<\infty.
\]
By the first Borel--Cantelli lemma,
\[
\PP_p\!\left(\mathcal G_p(\widehat R_n)>0\ \text{infinitely often}\right)=0.
\]
Equivalently, with \(\PP_p\)-probability one, there exists a finite random index \(N\) such that
\[
\mathcal G_p(\widehat R_n)=0
\qquad
\forall n\ge N.
\]
This proves the almost-sure eventual-zero statement.
\end{proof}

\begin{remark}[Interpretation of the offline sample-size sweep]
Corollary~\ref{cor:offline_zero_regret_identification} is the formal version of the empirical
``threshold'' pattern seen in offline sample-size sweeps. It does not assert a deterministic sample
size after which every possible logged dataset has zero regret. Instead, it asserts that the
probability of any nonzero temperature-zero regret decays exponentially in the number of logged
preference samples; along a nested infinite offline sample path, nonzero-regret ERM outputs occur
only finitely often almost surely. Therefore an empirical mean over repeated offline runs can drop
to the numerical evaluation floor once \(n\) exceeds the logarithmic burn-in scale
\(c_\gamma^{-1}\log N_\gamma\), provided Conditions~\ref{as:LP1}--\ref{as:LP2} hold.
\end{remark}

\begin{remark}[Why no extra assumption appears]
The fixed-scale truth-centered loss gap used by
Theorem~\ref{thm:offline_exact_erm_scale_eps} is automatic under
Conditions~\ref{as:LP1}--\ref{as:LP2} by
Lemma~\ref{lem:automatic_fixed_scale_loss_gap_P}. The online proof in
Theorem~\ref{thm:fixed_scale_characterization} needed the exploration bound only to convert such a
truth-centered loss gap into a conditional per-round loss gap under the learner's sampled slates,
which introduced a likelihood-ratio factor through
Lemma~\ref{lem:slate_domination_likelihood_ratio}. In the present offline theorem, the logged
slates are already sampled from \(\pi_0(\cdot\mid X)^{\otimes K}\), which is exactly the sampling
law defining \(\mathcal L_p\). Therefore that step becomes an identity, and no additional coverage
or fixed-scale minimax assumption is required.
\end{remark}

\subsection{A matching offline lower bound}
\label{sec:offline_matching_lower_bound}
\label{app:offline_matching_lower_bound_pf}

We now show that the logarithmic dependence on the target accuracy obtained in
Corollary~\ref{cor:offline_log_eps_umr} is unimprovable in general. The result is a two-instance
testing lower bound for the same \(\pi_0\)-logged offline model as in
Section~\ref{sec:offline_exact_erm_pi0}. It does not require failure of decision-relevant
diversity; on the contrary, it applies even when the upper-bound conditions hold, and therefore
shows that the \(\log(1/\varepsilon)\) dependence is sharp up to constants on every nontrivial
two-point subclass.

A possibly randomized offline learner \(\mathsf A\) is realized by an auxiliary seed \(U\),
independent of the logged sample, and outputs a centered score estimate
\[
\widehat R_n^{\mathsf A}=\mathsf A(D_n,U).
\]
We restrict attention to learners for which there exists a measurable selector
\[
(D_n,U,x,i)\longmapsto a_{\widehat R_n^{\mathsf A}}(x,i)
\in
\arg\max_{a\in\supp(\pi_0(\cdot\mid x))}\widehat R_n^{\mathsf A}(x,a,i).
\]
Its expected temperature-zero regret under truth \(p\in\mathcal P\) is
\[
\mathcal R_{0,p}^{\mathsf A}(n)
\coloneqq
\EE_p\big[\mathcal G_p(\widehat R_n^{\mathsf A})\big].
\]

The proof uses the one-sample KL identity and testing inequality proved below.

\begin{lemma}[One-sample KL identity for \(\pi_0\)-logged offline data]
\label{lem:offline_one_sample_kl_pi0}
Let
\[
\mathsf P_p^{\mathrm{off}}
\coloneqq
\Law_p(X,I,\mathbf A,Y)
\]
denote the law of one offline observation under truth \(p\in\mathcal P\), where
\[
(X,I)\sim d_0\times\rho,
\qquad
\mathbf A\sim \pi_0(\cdot\mid X)^{\otimes K},
\qquad
Y\sim P_{R_p}(\cdot\mid X,I,\mathbf A).
\]
Then for every \(p,q\in\mathcal P\),
\[
\KL\!\left(\mathsf P_p^{\mathrm{off}}\,\middle\|\,\mathsf P_q^{\mathrm{off}}\right)
=
\mathcal L_p(R_q)-\mathcal L_p(R_p).
\]
\end{lemma}

\begin{proof}[Proof of Lemma~\ref{lem:offline_one_sample_kl_pi0}]
Under both truths \(p\) and \(q\), the marginal law of \((X,I,\mathbf A)\) is the same, namely
\[
(X,I)\sim d_0\times\rho,
\qquad
\mathbf A\sim \pi_0(\cdot\mid X)^{\otimes K}.
\]
Hence the chain rule for KL gives
\[
\KL\!\left(\mathsf P_p^{\mathrm{off}}\,\middle\|\,\mathsf P_q^{\mathrm{off}}\right)
=
\EE_p\Big[
\KL\!\Big(
P_{R_p}(\cdot\mid X,I,\mathbf A)\,\Big\|\,P_{R_q}(\cdot\mid X,I,\mathbf A)
\Big)
\Big].
\]
By Lemma~\ref{lem:excess_is_kl_support}, for every realized \((x,i,\mathbf a)\),
\[
\KL\!\Big(
P_{R_p}(\cdot\mid x,i,\mathbf a)\,\Big\|\,P_{R_q}(\cdot\mid x,i,\mathbf a)
\Big)
=
\EE\big[\ell(\mathbf v_{R_q},Y)-\ell(\mathbf v_{R_p},Y)\mid x,i,\mathbf a\big],
\]
where \(Y\sim P_{R_p}(\cdot\mid x,i,\mathbf a)\). Taking expectation over the common law of
\((X,I,\mathbf A)\) yields
\[
\KL\!\left(\mathsf P_p^{\mathrm{off}}\,\middle\|\,\mathsf P_q^{\mathrm{off}}\right)
=
\mathcal L_p(R_q)-\mathcal L_p(R_p).
\]
\end{proof}

\begin{lemma}[Testing lower bound from KL]
\label{lem:testing_lower_bound_from_kl}
For any probability measures \(P,Q\) on the same measurable space and any measurable event \(A\),
\[
P(A)+Q(A^c)\ge \frac12 e^{-\KL(P\|Q)}.
\]
Consequently,
\[
P(A)+Q(A^c)\ge
\frac12 \exp\!\big(-\min\{\KL(P\|Q),\KL(Q\|P)\}\big).
\]
\end{lemma}

\begin{proof}[Proof of Lemma~\ref{lem:testing_lower_bound_from_kl}]
For any event \(A\),
\[
P(A)+Q(A^c)
=
1-\big(Q(A)-P(A)\big)
\ge
1-\operatorname{TV}(P,Q),
\]
where
\[
\operatorname{TV}(P,Q)\coloneqq \sup_B |P(B)-Q(B)|.
\]
Let
\[
\operatorname{BC}(P,Q)\coloneqq \int \sqrt{dP\,dQ}
\]
be the Bhattacharyya coefficient. Writing densities \(p,q\) with respect to a common dominating
measure, Cauchy--Schwarz gives
\begin{align*}
\operatorname{TV}(P,Q)
&=
\frac12\int |p-q|\\
&=
\frac12\int |\sqrt p-\sqrt q|(\sqrt p+\sqrt q)\\
&\le
\frac12
\left(\int (\sqrt p-\sqrt q)^2\right)^{1/2}
\left(\int (\sqrt p+\sqrt q)^2\right)^{1/2}\\
&=
\sqrt{1-\operatorname{BC}(P,Q)^2}.
\end{align*}

If \(\KL(P\|Q)=\infty\), then the first displayed claim is immediate because its right-hand side is
\(0\). Assume henceforth that \(\KL(P\|Q)<\infty\). Then \(P\ll Q\), and Jensen's inequality gives
\[
-\frac12\KL(P\|Q)
=
\EE_P\!\left[\log \sqrt{\frac{dQ}{dP}}\right]
\le
\log \EE_P\!\left[\sqrt{\frac{dQ}{dP}}\right]
=
\log \operatorname{BC}(P,Q).
\]
Therefore
\[
\operatorname{BC}(P,Q)^2\ge e^{-\KL(P\|Q)},
\]
and hence
\[
\operatorname{TV}(P,Q)\le \sqrt{1-e^{-\KL(P\|Q)}}.
\]
Using \(1-\sqrt{1-u}\ge u/2\) for \(u\in[0,1]\) with \(u=e^{-\KL(P\|Q)}\), we obtain
\[
1-\operatorname{TV}(P,Q)\ge \frac12 e^{-\KL(P\|Q)}.
\]
Combining with the first display proves
\[
P(A)+Q(A^c)\ge \frac12 e^{-\KL(P\|Q)}.
\]
Applying the same argument with \(P\) and \(Q\) interchanged yields the symmetric version.
\end{proof}

\begin{theorem}[Offline two-instance exponential lower bound]
\label{thm:offline_two_instance_exp_lower}
Assume \ref{as:LP1}--\ref{as:LP2}. Fix any two truths \(p^0,p^1\in\mathcal P\), and write
\[
a_m(x,i)\coloneqq a_{p^m}(x,i),
\qquad m\in\{0,1\}.
\]
Define the disagreement region
\[
\mathcal D_{01}
\coloneqq
\{(x,i)\in\cX\times\cU: a_0(x,i)\neq a_1(x,i)\},
\qquad
s_{01}\coloneqq (d_0\times\rho)(\mathcal D_{01}),
\]
and the symmetric one-sample offline KL gap
\[
\kappa_{01}^{\mathrm{off}}
\coloneqq
\min\Big\{
\mathcal L_{p^0}(R_{p^1})-\mathcal L_{p^0}(R_{p^0}),
\;
\mathcal L_{p^1}(R_{p^0})-\mathcal L_{p^1}(R_{p^1})
\Big\}.
\]
Then every possibly randomized offline learner \(\mathsf A\) satisfying the measurable-selector
requirement above satisfies, for every \(n\ge 1\),
\[
\max\Big\{
\mathcal R_{0,p^0}^{\mathsf A}(n),
\mathcal R_{0,p^1}^{\mathsf A}(n)
\Big\}
\ge
\frac{\Delta_{\min}^{\mathcal P}\,s_{01}}{4}\,
e^{-n\kappa_{01}^{\mathrm{off}}}.
\]
\end{theorem}

\begin{proof}[Proof of Theorem~\ref{thm:offline_two_instance_exp_lower}]
Let \(U\) be the learner's auxiliary random seed, independent of all data and truths. Let
\[
D_n=(Z_1,\dots,Z_n),
\qquad
Z_s=(X_s,I_s,\mathbf A_s,Y_s),
\]
be the logged sample of size \(n\). By definition of \(\mathcal G_p\), on an enlarged probability
space we may realize a fresh evaluation pair \((X,I)\sim d_0\times\rho\), independent of
\((D_n,U)\).

Write
\[
\widehat a_n(x,i)\coloneqq a_{\widehat R_n^{\mathsf A}}(x,i).
\]
Define the test
\[
\psi(D_n,U,X,I)
\coloneqq
\begin{cases}
0, & \widehat a_n(X,I)=a_0(X,I),\\
1, & \widehat a_n(X,I)\neq a_0(X,I).
\end{cases}
\]

Let \(G_0\) and \(G_1\) be full \(d_0\times\rho\)-measure sets on which the supportwise gap
condition in Condition~\ref{as:LP2} holds for truths \(p^0\) and \(p^1\), respectively. Since the
fresh evaluation pair \((X,I)\) has law \(d_0\times\rho\) under both truths, intersecting events with
\(G_0\) or \(G_1\) does not change their probabilities.

We first lower-bound the regret under truth \(p^0\). By definition,
\[
\mathcal R_{0,p^0}^{\mathsf A}(n)
=
\EE_{p^0}\Big[
R_{p^0}\!\big(X,a_0(X,I),I\big)
-
R_{p^0}\!\big(X,\widehat a_n(X,I),I\big)
\Big].
\]
On the event
\[
\{(X,I)\in\mathcal D_{01}\}\cap \{\psi=1\}\cap G_0,
\]
we have \(\widehat a_n(X,I)\neq a_0(X,I)\). Since \(a_0(X,I)\) is the unique supportwise maximizer
for truth \(p^0\) on \(G_0\), the pointwise regret is at least
\(\Delta_{\min}^{\mathcal P}\). Therefore
\begin{equation}
\mathcal R_{0,p^0}^{\mathsf A}(n)
\ge
\Delta_{\min}^{\mathcal P}\,
\PP_{p^0}\big((X,I)\in\mathcal D_{01},\ \psi=1\big).
\label{eq:offline_lb_truth0}
\end{equation}

Likewise, under truth \(p^1\),
\[
\mathcal R_{0,p^1}^{\mathsf A}(n)
=
\EE_{p^1}\Big[
R_{p^1}\!\big(X,a_1(X,I),I\big)
-
R_{p^1}\!\big(X,\widehat a_n(X,I),I\big)
\Big].
\]
On the event
\[
\{(X,I)\in\mathcal D_{01}\}\cap \{\psi=0\}\cap G_1,
\]
we have \(\widehat a_n(X,I)=a_0(X,I)\neq a_1(X,I)\). Therefore, by the same supportwise gap
condition on \(G_1\), the pointwise regret is at least \(\Delta_{\min}^{\mathcal P}\), and hence
\begin{equation}
\mathcal R_{0,p^1}^{\mathsf A}(n)
\ge
\Delta_{\min}^{\mathcal P}\,
\PP_{p^1}\big((X,I)\in\mathcal D_{01},\ \psi=0\big).
\label{eq:offline_lb_truth1}
\end{equation}

If \(s_{01}=0\), then the theorem is trivial. Assume henceforth that \(s_{01}>0\), and let
\[
\mu_{01}\coloneqq \Law\big((X,I)\mid (X,I)\in\mathcal D_{01}\big).
\]
For \(m\in\{0,1\}\), let \(P_m^{(n)}\) denote the law of the logged sample \(D_n\) under truth
\(p^m\), and define
\[
\widetilde P_m
\coloneqq
P_m^{(n)}\otimes \Law(U)\otimes \mu_{01}.
\]
Because \((X,I)\) and \(U\) are independent of the logged sample and have the same law under both
truths,
\[
\PP_{p^m}\big((X,I)\in\mathcal D_{01},\ \psi=b\big)
=
s_{01}\,\widetilde P_m(\psi=b),
\qquad
m,b\in\{0,1\}.
\]
Combining \eqref{eq:offline_lb_truth0} and \eqref{eq:offline_lb_truth1} gives
\[
\mathcal R_{0,p^0}^{\mathsf A}(n)+\mathcal R_{0,p^1}^{\mathsf A}(n)
\ge
\Delta_{\min}^{\mathcal P}\,s_{01}\,
\Big(
\widetilde P_0(\psi=1)+\widetilde P_1(\psi=0)
\Big).
\]
Apply Lemma~\ref{lem:testing_lower_bound_from_kl} to the event \(\{\psi=1\}\). This yields
\[
\widetilde P_0(\psi=1)+\widetilde P_1(\psi=0)
\ge
\frac12
\exp\!\Big(
-\min\{
\KL(\widetilde P_0\|\widetilde P_1),
\KL(\widetilde P_1\|\widetilde P_0)
\}
\Big).
\]
Since the factors \(\Law(U)\) and \(\mu_{01}\) are common to both measures,
\[
\KL(\widetilde P_0\|\widetilde P_1)
=
\KL(P_0^{(n)}\|P_1^{(n)}),
\qquad
\KL(\widetilde P_1\|\widetilde P_0)
=
\KL(P_1^{(n)}\|P_0^{(n)}).
\]
The logged observations are i.i.d., so
\[
\KL(P_0^{(n)}\|P_1^{(n)})
=
n\,\KL(\mathsf P_{p^0}^{\mathrm{off}}\|\mathsf P_{p^1}^{\mathrm{off}}),
\qquad
\KL(P_1^{(n)}\|P_0^{(n)})
=
n\,\KL(\mathsf P_{p^1}^{\mathrm{off}}\|\mathsf P_{p^0}^{\mathrm{off}}).
\]
By Lemma~\ref{lem:offline_one_sample_kl_pi0},
\[
\KL(\mathsf P_{p^0}^{\mathrm{off}}\|\mathsf P_{p^1}^{\mathrm{off}})
=
\mathcal L_{p^0}(R_{p^1})-\mathcal L_{p^0}(R_{p^0}),
\]
and similarly with \(p^0,p^1\) interchanged. Therefore
\[
\widetilde P_0(\psi=1)+\widetilde P_1(\psi=0)
\ge
\frac12 e^{-n\kappa_{01}^{\mathrm{off}}}.
\]
Substituting into the previous display yields
\[
\mathcal R_{0,p^0}^{\mathsf A}(n)+\mathcal R_{0,p^1}^{\mathsf A}(n)
\ge
\frac{\Delta_{\min}^{\mathcal P}\,s_{01}}{2}\,
e^{-n\kappa_{01}^{\mathrm{off}}}.
\]
Finally,
\[
\max\Big\{
\mathcal R_{0,p^0}^{\mathsf A}(n),
\mathcal R_{0,p^1}^{\mathsf A}(n)
\Big\}
\ge
\frac12
\Big(
\mathcal R_{0,p^0}^{\mathsf A}(n)+\mathcal R_{0,p^1}^{\mathsf A}(n)
\Big),
\]
which proves
\[
\max\Big\{
\mathcal R_{0,p^0}^{\mathsf A}(n),
\mathcal R_{0,p^1}^{\mathsf A}(n)
\Big\}
\ge
\frac{\Delta_{\min}^{\mathcal P}\,s_{01}}{4}\,
e^{-n\kappa_{01}^{\mathrm{off}}}.
\]
\end{proof}

\begin{corollary}[Minimax \(\Omega(\log(1/\varepsilon))\) offline lower bound]
\label{cor:offline_logeps_lower_bound}
Assume \ref{as:LP1}--\ref{as:LP2}. The infimum below is over offline learners satisfying the measurable-selector requirement stated before Theorem~\ref{thm:offline_two_instance_exp_lower}. Suppose there exist \(p^0,p^1\in\mathcal P\) such that
\[
\mathcal G_{p^0}(R_{p^1})>0.
\]
Then the constant \(\kappa_{01}^{\mathrm{off}}\) from
Theorem~\ref{thm:offline_two_instance_exp_lower} satisfies \(0<\kappa_{01}^{\mathrm{off}}<\infty\),
and
\[
\inf_{\mathsf A}\sup_{p\in\mathcal P}\mathcal R_{0,p}^{\mathsf A}(n)
\ge
\frac{\Delta_{\min}^{\mathcal P}\,\varepsilon_{\mathrm{iso}}^{\mathcal P}}
{4\Delta_{\max}^{\mathcal P}}\,
e^{-n\kappa_{01}^{\mathrm{off}}},
\]
where \(\kappa_{01}^{\mathrm{off}}\) is the constant from
Theorem~\ref{thm:offline_two_instance_exp_lower}. Consequently, every offline learner satisfying
\[
\sup_{p\in\mathcal P}\mathcal R_{0,p}^{\mathsf A}(n)\le \varepsilon
\]
must obey
\[
n
\ge
\frac{1}{\kappa_{01}^{\mathrm{off}}}
\log\!\left(
\frac{\Delta_{\min}^{\mathcal P}\,\varepsilon_{\mathrm{iso}}^{\mathcal P}}
{4\Delta_{\max}^{\mathcal P}\,\varepsilon}
\right)
\]
for every
\(
\varepsilon\in\big(0,\Delta_{\min}^{\mathcal P}\varepsilon_{\mathrm{iso}}^{\mathcal P}
/(4\Delta_{\max}^{\mathcal P})\big].
\)
In particular, the class-minimax offline sample complexity is \(\Omega(\log(1/\varepsilon))\).
\end{corollary}

\begin{proof}[Proof of Corollary~\ref{cor:offline_logeps_lower_bound}]
By Lemma~\ref{lem:automatic_positive_regret_isolation_P}, the assumption
\[
\mathcal G_{p^0}(R_{p^1})>0
\]
implies
\[
\mathcal G_{p^0}(R_{p^1})\ge \varepsilon_{\mathrm{iso}}^{\mathcal P}.
\]
Applying Lemma~\ref{lem:regret_vs_disagreement} to the pair \((p^0,p^1)\) gives
\[
s_{01}
=
(d_0\times\rho)\{a_{p^1}\neq a_{p^0}\}
\ge
\frac{\mathcal G_{p^0}(R_{p^1})}{\Delta_{\max}^{\mathcal P}}
\ge
\frac{\varepsilon_{\mathrm{iso}}^{\mathcal P}}{\Delta_{\max}^{\mathcal P}}.
\]
In particular, \(s_{01}>0\). Applying the same lemma to the reversed pair \((p^1,p^0)\) yields
\[
\mathcal G_{p^1}(R_{p^0})\ge \Delta_{\min}^{\mathcal P}s_{01}>0.
\]

If
\[
\mathcal L_{p^0}(R_{p^1})-\mathcal L_{p^0}(R_{p^0})=0,
\]
then Lemma~\ref{lem:zero_loss_implies_zero_regret_P} would imply
\[
\mathcal G_{p^0}(R_{p^1})=0,
\]
a contradiction. Therefore
\[
\mathcal L_{p^0}(R_{p^1})-\mathcal L_{p^0}(R_{p^0})>0.
\]
Likewise, if
\[
\mathcal L_{p^1}(R_{p^0})-\mathcal L_{p^1}(R_{p^1})=0,
\]
then Lemma~\ref{lem:zero_loss_implies_zero_regret_P} would imply
\[
\mathcal G_{p^1}(R_{p^0})=0,
\]
again a contradiction. Hence
\[
\mathcal L_{p^1}(R_{p^0})-\mathcal L_{p^1}(R_{p^1})>0.
\]
Both quantities are finite by the envelope definitions above and Lemma~\ref{lem:mnl_loss_basic}, so
\[
0<\kappa_{01}^{\mathrm{off}}<\infty.
\]

Now apply Theorem~\ref{thm:offline_two_instance_exp_lower}:
\[
\inf_{\mathsf A}\sup_{p\in\mathcal P}\mathcal R_{0,p}^{\mathsf A}(n)
\ge
\frac{\Delta_{\min}^{\mathcal P}\,s_{01}}{4}
e^{-n\kappa_{01}^{\mathrm{off}}}
\ge
\frac{\Delta_{\min}^{\mathcal P}\,\varepsilon_{\mathrm{iso}}^{\mathcal P}}
{4\Delta_{\max}^{\mathcal P}}
e^{-n\kappa_{01}^{\mathrm{off}}}.
\]
Solving the last display for \(n\) gives the stated \(\Omega(\log(1/\varepsilon))\) lower bound.
\end{proof}

\begin{remark}[Interpretation]
Corollary~\ref{cor:offline_logeps_lower_bound} is the offline counterpart of the upper bound in
Corollary~\ref{cor:offline_log_eps_umr}. The upper result shows that, under
Conditions~\ref{as:LP1}--\ref{as:LP2}, exact offline ERM achieves expected temperature-zero regret
at most \(\varepsilon\) with \(O(\log(1/\varepsilon))\) \(\pi_0\)-logged samples. The lower result
shows that this logarithmic
dependence cannot be improved in general: every nontrivial class containing a pair
\((p^0,p^1)\) with positive decision disagreement has class-minimax offline sample complexity
\(\Omega(\log(1/\varepsilon))\).
\end{remark}

\end{document}